\documentclass[10pt,twocolumn,letterpaper]{article}

\usepackage[pagenumbers]{wacv} %

\usepackage{graphicx}
\usepackage{amsmath}
\usepackage{amssymb}
\usepackage{booktabs}
\usepackage{xcolor}
\usepackage{subcaption}
\usepackage{tabularx}
\usepackage{amsthm}

\newtheorem{lemma}{Lemma}

\usepackage[pagebackref,breaklinks,colorlinks]{hyperref}

\newcommand{\draft}{{DRaFT}}

\newcommand{\textbase}{{\text{base}}}
\newcommand{\textlora}{{\text{LoRA}}}
\newcommand{\td}{{\text{data}}}

\newcommand{\aat}{\sqrt{\tilde{\alpha_t}}}
\newcommand{\method}{{Annealed Importance Guidance}}
\newcommand{\methodabbv}{{AIG}}
\newcommand{\mbftheta}{\mathbf{\theta}}
\newcommand{\mbtheta}{\mbftheta}
\newcommand{\mbx}{\mathbf{x}}
\newcommand{\mbm}{\boldsymbol{\mu}}
\newcommand{\mby}{\mathbf{y}}
\newcommand{\mbc}{\mathbf{c}}
\newcommand{\mbtft}{\mbftheta^{*}_{d}}
\newcommand{\Ht}{H_{\mbtft}}
\newcommand{\dd}{\text{d}}
\newcommand{\reg}{regularization}

\newcommand{\Done}{{\textcolor{blue}{$\mathcal{D}_1$}}}
\newcommand{\Dtwo}{{\textcolor{orange}{$\mathcal{D}_2$}}}
\newcommand{\Dthree}{{\textcolor{green}{$\mathcal{D}_3$}}}
\newcommand{\done}{\mathcal{D}_1}
\newcommand{\dtwo}{\mathcal{D}_2}

\newcommand{\numusers}{36}
\newcommand{\numvotes}{1500}

\usepackage[capitalize]{cleveref}
\usepackage{wrapfig}

\crefname{section}{Sec.}{Secs.}
\Crefname{section}{Section}{Sections}
\Crefname{table}{Table}{Tables}
\crefname{table}{Tab.}{Tabs.}

\newcolumntype{Y}{>{\centering\arraybackslash}X}

\def\wacvPaperID{1844} %
\def\confName{WACV}
\def\confYear{2025}

\begin{document}
\title{Elucidating Optimal Reward-Diversity Tradeoffs in Text-to-Image Diffusion Models}

\author{
    Rohit Jena\(^{1,2}\)\thanks{Work done during an internship at NVIDIA \\Corresponding Author: \texttt{rjena@seas.upenn.edu}} \quad Ali Taghibakhshi\(^2\) \quad Sahil Jain\(^2\) \quad Gerald Shen\(^2\) \quad Nima Tajbakhsh\(^2\) \quad Arash Vahdat\(^2\)  \\
    \begin{minipage}{0.1\linewidth}
        \quad
    \end{minipage}
    \begin{minipage}{0.3\linewidth}
        \centering
        \(^1\)University of Pennsylvania 
    \end{minipage}%
    \begin{minipage}{0.5\linewidth}
        \centering
        \(^2\)NVIDIA
    \end{minipage}
}
\maketitle

\begin{abstract}
Text-to-image (T2I) diffusion models have become prominent tools for generating high-fidelity images from text prompts.
However, when trained on unfiltered internet data, these models can produce unsafe, incorrect, or stylistically undesirable images that are not aligned with human preferences.
To address this, recent approaches have incorporated human preference datasets to fine-tune T2I models or to optimize reward functions that capture these preferences.
Although effective, these methods are vulnerable to reward hacking, where the model overfits to the reward function, leading to a loss of diversity in the generated images.
In this paper, we prove the inevitability of reward hacking and study natural regularization techniques like KL divergence and LoRA scaling, and their limitations for diffusion models.
We also introduce {\method} ({\methodabbv}), an inference-time regularization inspired by Annealed Importance Sampling, which retains the diversity of the base model while achieving Pareto-Optimal reward-diversity tradeoffs.
Our experiments demonstrate the benefits of {\methodabbv} for Stable Diffusion models, striking the optimal balance between reward optimization and image diversity.
Furthermore, a user study confirms that {\methodabbv} improves diversity \textit{and} quality of generated images across different model architectures and reward functions.
\end{abstract}

\begin{figure}[ht!]
    \centering
    \includegraphics[width=\linewidth]{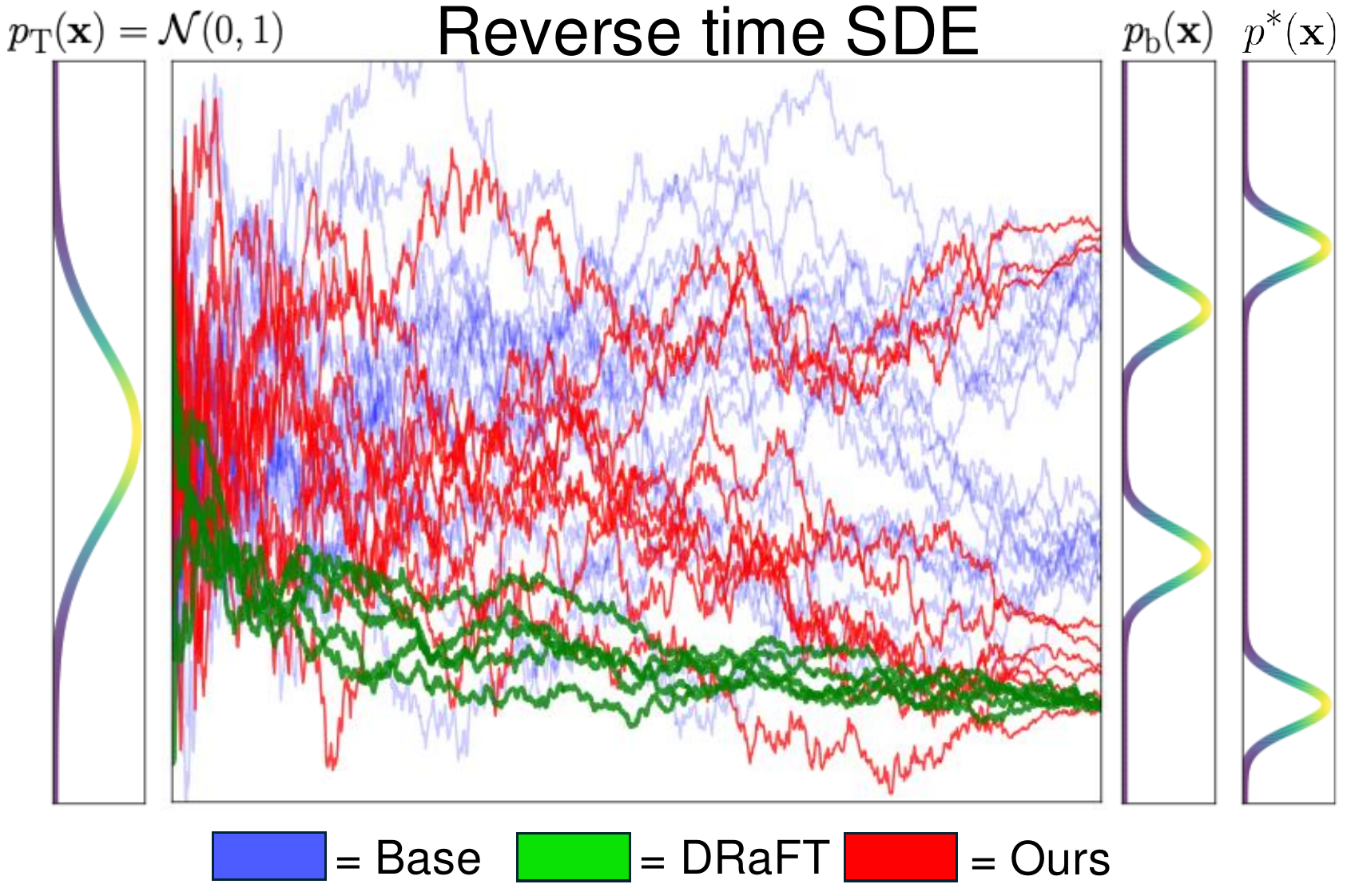}
    \caption{\textbf{Illustration of {\method}}. Blue trajectories generate samples from the base distribution $p_b(\mbx)$ but when finetuned with reward models, lead to \textit{reward hacking} and mode collapse, evidenced by the green trajectories collapsing to only one mode of desired data distribution $p^*(\mbx)$. 
    Our method anneals between the dynamics of the base and {\draft} score functions, showing base-model-like exploration followed by DRaFT-like steady convergence to modes of $p^*(\mbx)$.
    The level of annealing is controllable by the end-user at inference (see~\cref{sec:aig}).
    }
    \label{fig:teaser}
    \vspace*{-10pt}
\end{figure}

\begin{figure*}[ht!]
    \centering
    \begin{minipage}{0.32\linewidth}
        \includegraphics*[width=\linewidth]{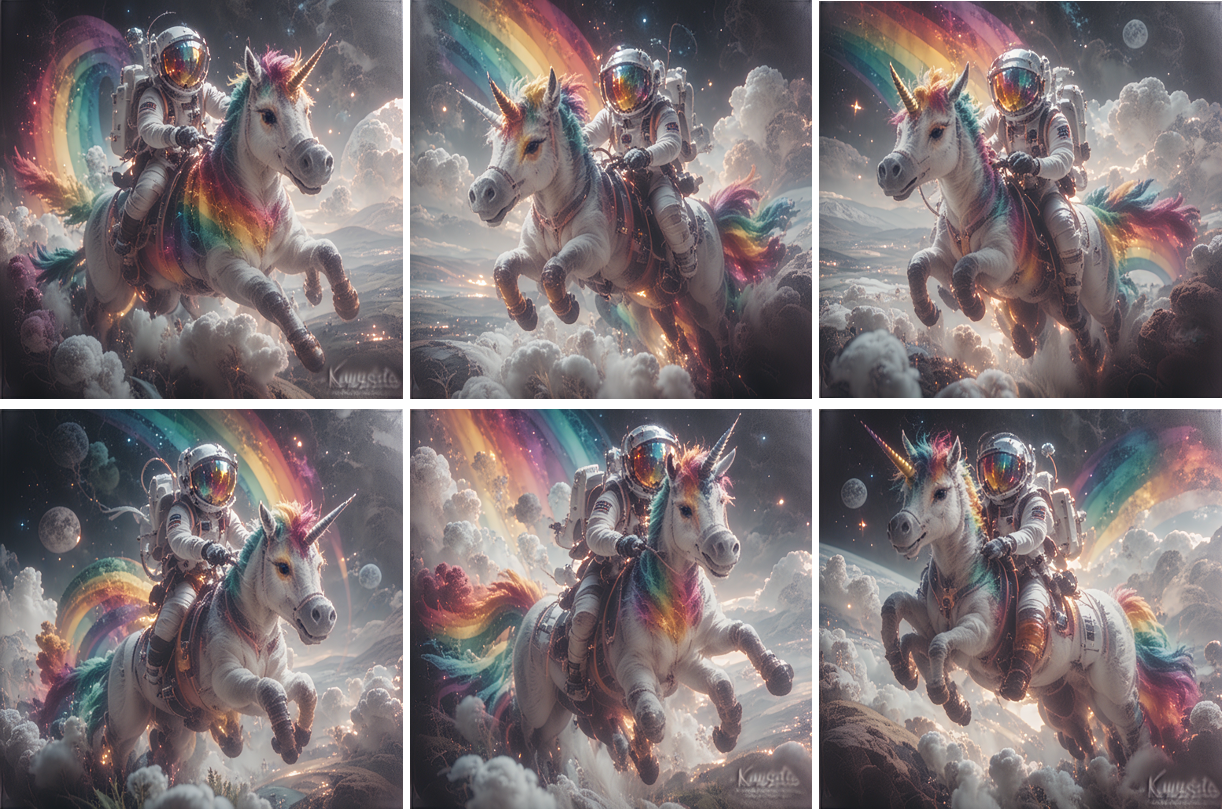}
        \subcaption{\draft}
        \label{fig:draftheader}
    \end{minipage}
    \begin{minipage}{0.32\linewidth}
        \includegraphics*[width=\linewidth]{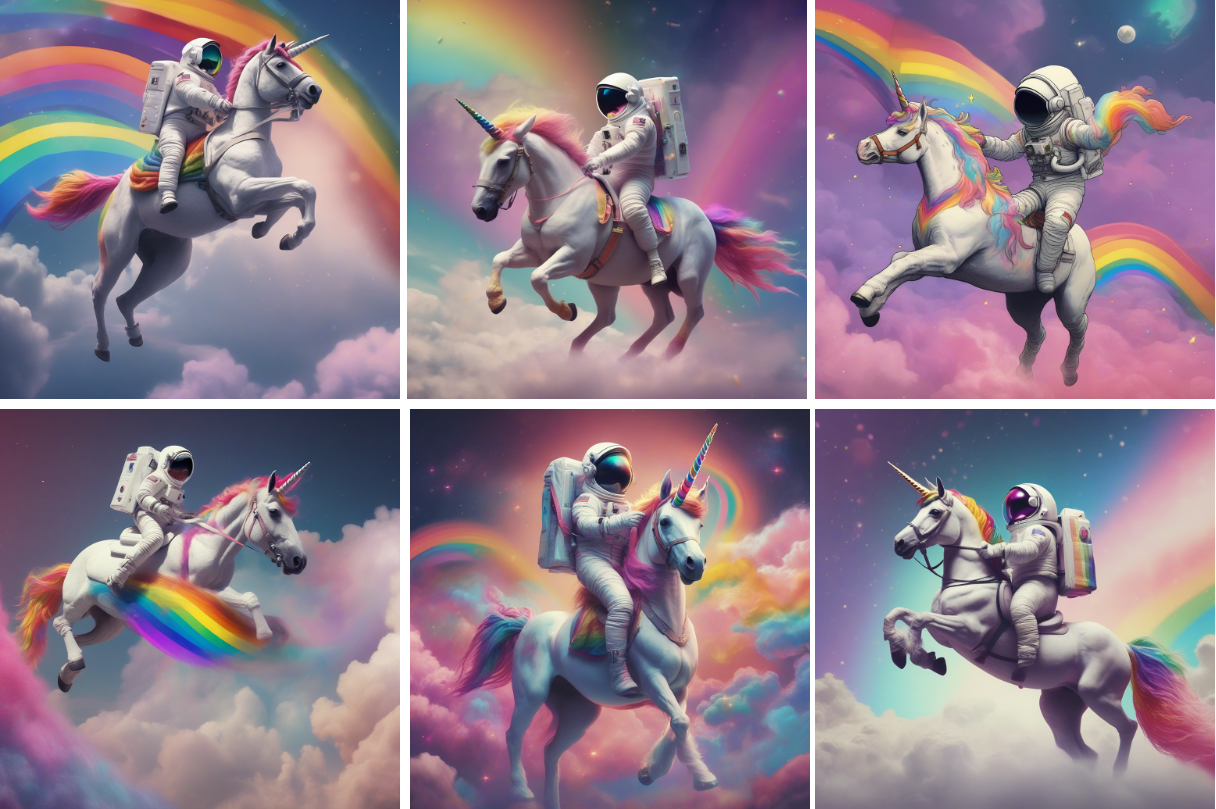}
        \subcaption{SDXL Base}
        \label{fig:basesdxlheader}
    \end{minipage}
    \begin{minipage}{0.32\linewidth}
        \includegraphics*[width=\linewidth]{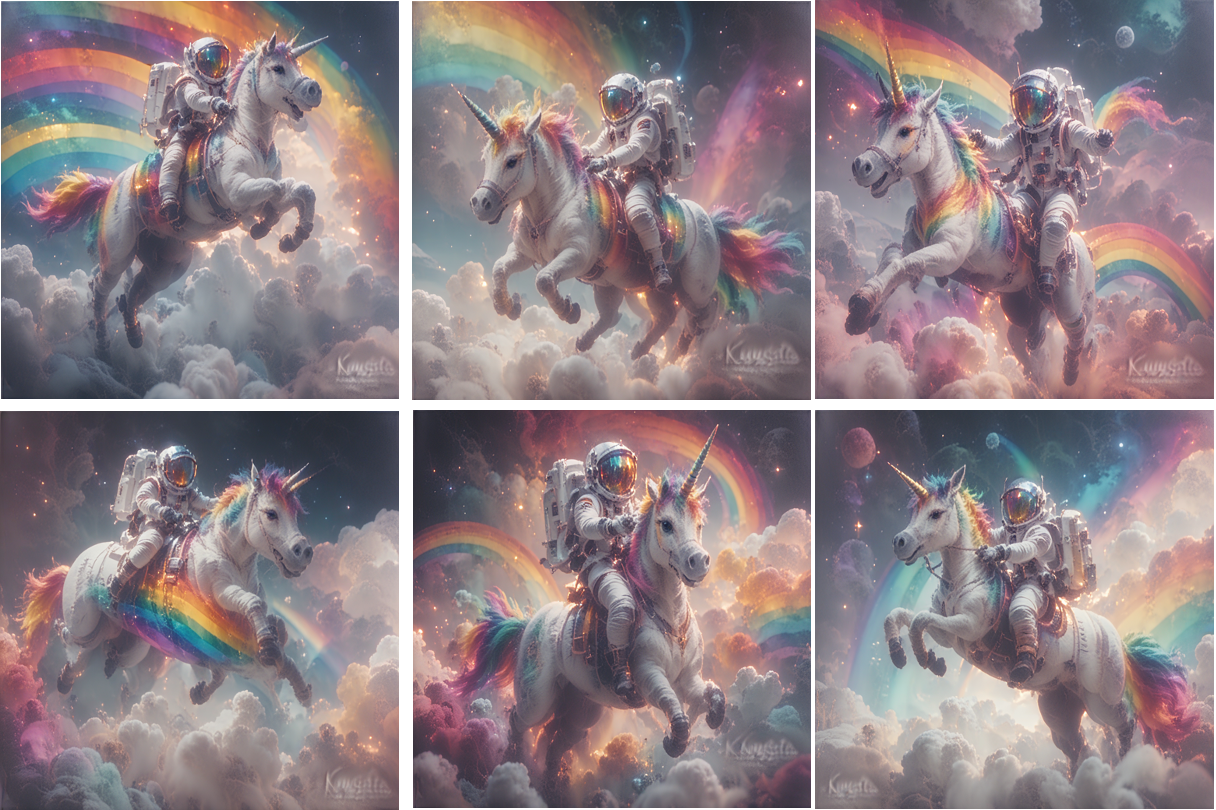}
        \subcaption{{\method} (Ours)}
        \label{fig:ourheader}
    \end{minipage}
    \caption{\textbf{Qualitative Comparison}: \cref{fig:basesdxlheader} shows six images generated from SDXL Base with the prompt ``\textit{An astronaut riding a rainbow unicorn, cinematic, dramatic}'' and \cref{fig:draftheader} shows images with the same seeds from the {\draft} model trained on the Pickscore~\cite{pickscore} reward.
    \cref{fig:draftheader} achieves high Pickscore rewards, but completely loses diversity compared to the base model in terms of pose and background variation.
    Our method (\cref{fig:ourheader}) inherits diversity from the base model while generating aesthetically pleasing images.}
    \vspace*{-10pt}
\end{figure*}

\vspace*{-8pt}
\vspace*{-10pt}
\section{Introduction}
\label{sec:intro}

Text-to-Image (T2I) Diffusion Models (DMs) ~\cite{sohldiffusion,ho2020denoising,sd,podell2023sdxl} have emerged as a powerful class of generative models for high-fidelity image generation from text prompts, trained on large-scale image-caption data.
However, training on unfiltered internet data captures the entire uncurated internet-scale data distribution that is not aligned to human preferences.  
These models may generate images that are unsafe~\cite{hong2024margin}, misaligned with respect to the prompt~\cite{huang2023t2i}, or that have undesirable stylistic effects~\cite{pickscore,hps,imagereward}. 
Following the effectiveness of learning from human preferences ~\cite{christiano2017deep} for large language models (LLMs)~\cite{dpo,yang2024using,ouyang2022training}, human preference datasets~\cite{pickscore,hps,laion,imagereward,hong2024margin} have been collected that reflect different style and content preferences by humans.
These datasets are used to either directly finetune the T2I model ~\cite{black2023training,diffusiondpo,hong2024margin} or maximize a reward function trained to capture human preferences~\cite{imagereward,draft,doodl}.
Unlike LLMs, most diffusion models can be finetuned by directly backpropagating through the (differentiable) reward function instead of relying on policy gradient algorithms, leading to high sample efficiency~\cite{draft}.
However, these methods are prone to \textit{reward hacking} where the finetuned model loses all its diversity and produce a unimodal set of images. %

In this paper, we first show the inevitability of reward hacking without regularization.
This is followed by analyzing the common regularization approaches for finetuning DMs.
However, these regularizations suffer from suboptimal reward-diversity tradeoffs due to two reasons.
The first reason is \textit{indiscriminate regularization}, i.e. the regularization is applied on \textit{all} timesteps of the sampling chain which sacrifices reward optimization for diversity.
These regularizations make sense for LLMs since the entire sequence of output tokens constitute the response to a prompt.
However, for diffusion models, different timesteps in the sampling chain have different dynamics. %
We show that earlier timesteps in the diffusion contribute to mode recovery, and later timesteps contribute to adding finegrained details -- necessitating regularization more favorably in the earlier timesteps.
The second reason is the `\textit{reference mismatch}' problem~\cite{hong2024margin}, where the enforced alignment between the finetuned and reference model fails to adequately capture the features of the preference data, especially when the reference model and preference data have distinct features.
Inspired by Annealed Importance Sampling~\cite{neal2001annealed}, we propose {\method} (\methodabbv), an inference-time regularization that alleviates both the indiscriminate {\reg} and reference mismatch problems. %
This leads to \textit{Pareto-Optimal} reward-diversity tradeoffs without training multiple models to find optimal hyperparameters.
Finally, we also notice that existing diversity~\cite{fid,precisionrecall} metrics compute differences between the generated image and the data distributions.
However, these metrics will also penalize any \textit{reference mismatch} between the base and finetuned models.
We propose a Spectral Distance metric to account for reference mismatch; this ignores any distributional shift between the base and finetuned models, and only penalizes differences in the `spread' of the distributions.

We apply {\methodabbv} to both StableDiffusion v1.4~\cite{sd} and StableDiffusion XL (Base)~\cite{podell2023sdxl} DMs, finetuned on PickScore and HPSv2 reward models, leading to four model configurations.
For each configuration, we perform an exhaustive ablation study by evaluating rewards on images generated from the PartiPrompts and HPSv2 evaluation prompts across a range of KL regularization, LoRA scaling and {\methodabbv} parameters, representing more than 149,600 generated images per configuration.
A comprehensive qualitative user study shows that our method produces highly diverse subsets of images while simultaneously increasing the quality and alignment of the generated images compared to the DRaFT model, cementing {\methodabbv} as an effective and inexpensive finetuning method for diffusion models.

\vspace*{-10pt}
\section{Related Work}
\label{sec:relatedwork}

\textbf{Text-to-image diffusion models.}
Diffusion models ~\cite{sohldiffusion,song2020score,song2019generative,ho2020denoising,kingma2021variational} are generative models that use a forward Markov process which diffuses the data distribution into a noise distribution, to convert data into noise.
This is followed by learning a corresponding reverse process which converts the noise back into data.
This line of work emerged from score matching estimation of the data distribution using samples only~\cite{hyvarinen2005estimation,song2020sliced} using neural networks, which can be used to run Langevin dynamics to sample data~\cite{bussi2007accurate}.
However, the manifold hypothesis leads to low data density regions in the ambient space, making Langevin mixing slow.
Denoising score estimation with score matching~\cite{vincent2011connection,song2019generative} was proposed to improve data density and faster mixing dynamics, setting the tone for effective score-based generative models, followed by improved training recipes~\cite{song2020improved}.
These models have been successful in a variety of domains, including high-fidelity image and video generation~\cite{ramesh2021zero,khachatryan2023text2video,ruiz2023dreambooth,saharia2022image,saharia2022palette,ho2022imagen,sd,podell2023sdxl,blattmann2023align,achiam2023gpt,videoworldsimulators2024}, audio synthesis~\cite{liu2023audioldm,luo2024diff,wang2023audit}, robotics~\cite{kapelyukh2023dall,carvalho2023motion,black2023zero}, 3D~\cite{sanghi2022clip,karnewar2023holodiffusion,poole2022dreamfusion,ntavelis2023autodecoding} and medical imaging~\cite{pinaya2022brain,wu2024medsegdiff,yoon2023sadm,ozbey2023unsupervised,ali2022spot}.
However, training on noisy image-caption pairs from the internet lead to models that show lack of alignment in terms of attribute binding~\cite{huang2023t2i,liu2022compositional,feng2022training}, counting, spatial relationships and occlusions~\cite{xu2024amodal}, stylistic choices~\cite{hps,pickscore}, and text rendering~\cite{liu2022character,TextDiffuser,chen2023textdiffuser}. 

\textbf{Text-to-image model alignment.}
Learning from human preferences is a popular strategy to align text-to-image models from curated human preference datasets, following its efficacy in alignment of large language models (LLMs).
Alignment is performed either by learning a reward model from the preference dataset followed by a finetuning the model on this reward model using RL~\cite{christiano2017deep,ouyang2022training,dubois2024alpacafarm}, or finetuning the model on the preference data directly~\cite{dpo,yang2024using}.
Popular multimodal reward models for human preferences include PickScore~\cite{pickscore}, Human Preference Score (HPS)~\cite{hps}, ImageReward~\cite{imagereward} and LAION-Aesthetic~\cite{laion} models which finetune a CLIP~\cite{clip} model on different sets of curated preference data.
Taking inspiration from policy gradient algorithms and RLHF for LLMs, several methods utilize RL-based finetuning on the DDPM reverse sampling chain ~\cite{fan2023optimizing,black2023training,fan2023optimizing,hao2024optimizing}.
Analogous to DPO~\cite{dpo}, Diffusion-DPO~\cite{diffusiondpo} express the likelihood of human preference of an offline in terms of the optimal and reference policy, allowing direct optimization of the policy without a reward model.
MaPO~\cite{hong2024margin} uses a margin-aware preference optimization without a reference model.

However, most classes of reward functions are differentiable, allowing backpropagation through the reward function directly.
This makes the sample efficiency of these algorithms much better than using RL.  
DOODL~\cite{doodl} propose an inference-time optimization of the noise latent directly to maximize a classifier effectively acting as a reward model.
Similar to DOODL, ReNO~\cite{reno} propose noise optimization for single-step diffusion models like SD(XL)-Turbo, and is shown to be much faster since it avoiding the reverse SDE/ODE altogether.
However, inference-time optimization is expensive for generation from large prompt datasets.
ReFL~\cite{imagereward} and AlignProp~\cite{prabhudesai2023aligning} compute gradients through random timesteps in the reverse sampling chain to optimize T2I models.
\draft~\cite{draft} also propose a truncated backpropagation through the reverse sampling chain, but use LoRA and activation checkpointing to avoid memory issues with large T2I models like StableDiffusion~\cite{sd}.

\textbf{Reward hacking and diversity of generative models.}
Most finetuning algorithms suffer from a phenomenon called \textit{reward hacking} or \textit{overoptimization}.
It refers to the model producing images that achieve very high rewards from the reward model, but produce low-diversity images with overly saturated textures. %
In the non-parametric case, we show a simple proof 
that for any given reward function, the optimal likelihood is a Delta distribution, making reward hacking inevitable without regularization.

\section{Method}
\label{sec:method}
\vspace*{-5pt}

\subsection{Problem Setup}
We briefly review the DM finetuning approach with differentiable reward functions.
\draft~\cite{draft} proposes a simple yet effective approach for fine-tuning diffusion models for differentiable reward functions.
Although RL-based methods that have been shown to be successful for finetuning LLMs, they are shown to be sample inefficient and harder to tune for diffusion models.
For a noise latent $\mathbf{x_T} \in \mathbb{R}^n$, (prompt) condition $\mathbf{c} \in \mathbb{R}^m$ and timestep $t \in \mathbb{R}^+$, {\draft} considers a differentiable reward function $r: \mathbb{R}^n\times\mathbb{R}^m \rightarrow \mathbb{R}$ that is used to optimize a diffusion model $f_\mbftheta: \mathbb{R}^n\times\mathbb{R}^m\times\mathbb{R}^+ \rightarrow \mathbb{R}^n$ parameterized by $\mbftheta$.
The diffusion model is finetuned to maximize the reward function on its generated samples: 
\begin{equation}
    J(\mbftheta) = \mathbb{E}_{\mathbf{c} \sim p(c), \mathbf{x_T} \sim \mathcal{N}(0, 1)} \left[r(\mathbf{x_0}(\mbftheta, \mathbf{c}, \mathbf{x_T}), \mathbf{c})\right]
    \label{eq:draftreward}
\end{equation}
where $\mathbf{x_0}$ is the generated image which is the solution of the reverse-time probability flow ODE~\cite{song2020score} (or an equivalent reverse-time SDE):
\begin{equation}
    d\mathbf{x} = \left[f(t)\mbx - \frac{g(t)^2}{2} \nabla_\mbx \log(p_t(\mbx|\mbc))\right] dt , \mbx_T \sim \mathcal{N}(0, 1)
    \label{eq:revode}
\end{equation}
where $\nabla_\mbx \log(p_t(\mbx|\mbc)) \approx f_\mbftheta(\mbx, \mbc, t)$ is the estimated score function learned using a diffusion model.
\cref{eq:draftreward} can be interpreted as finding the distribution $p_\mbftheta(\mbx_0|\mbc)$ that maximizes the expected reward.
We denote $p(\mbx_0 | \mbc)$ to refer to the probability distribution induced by ~\cref{eq:revode}.

\begin{figure*}
    \centering
    \begin{minipage}{0.98\textwidth}
        \centering
        \includegraphics[width=0.22\linewidth]{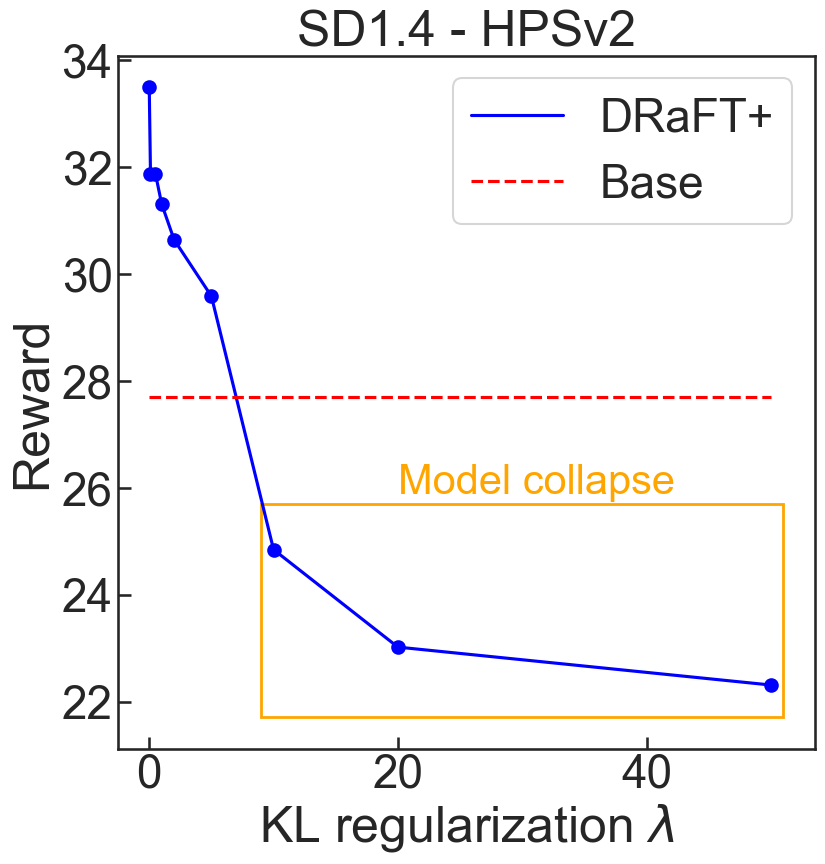}
        \includegraphics[width=0.22\linewidth]{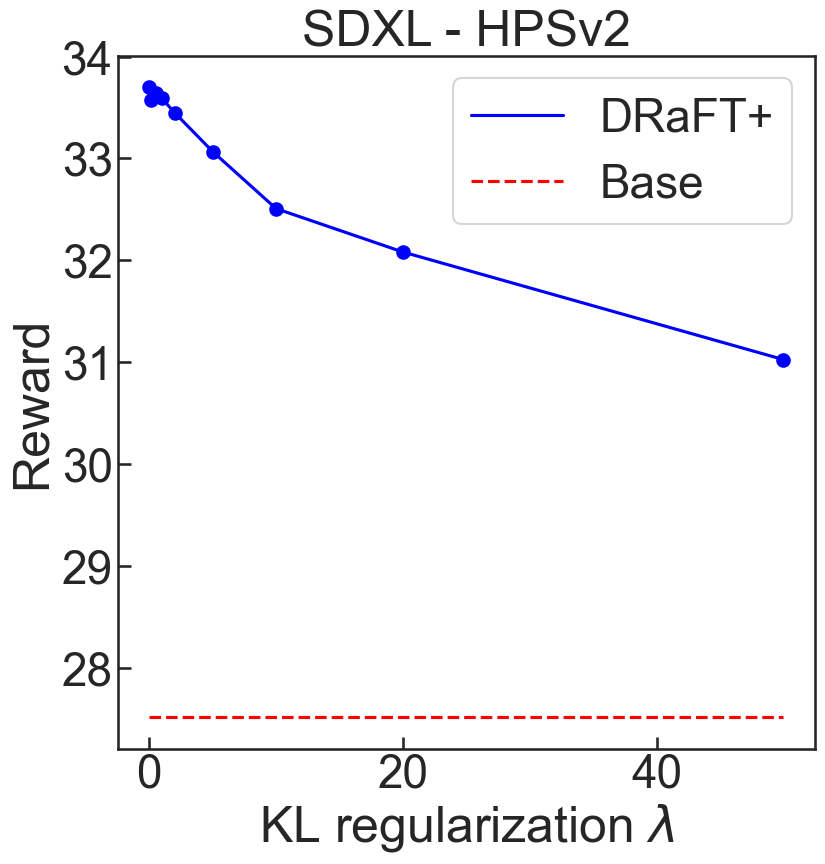}
        \includegraphics[width=0.22\linewidth]{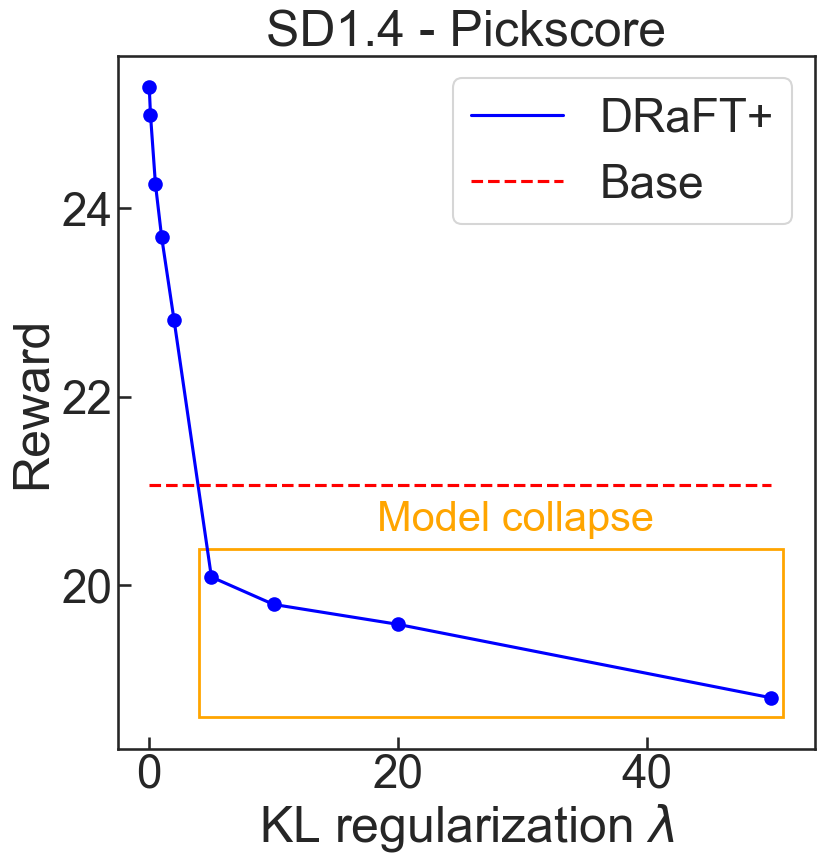}
        \includegraphics[width=0.22\linewidth]{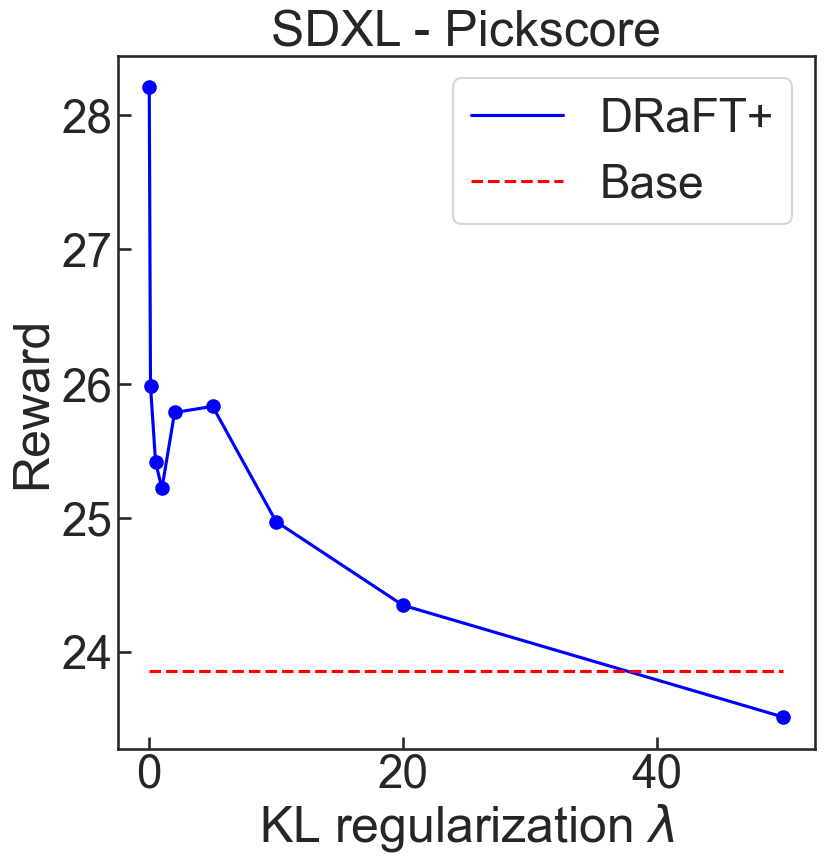}
        \subcaption{\textbf{Tradeoff between reward and KL divergence}. For SD1.4, the model breaks down at higher KL values, indicating that the same $\lambda$ hyperparameter cannot be used universally across architectures or reward models.
        The SD1.4 model suffers from breakdown even with LoRA finetuning.}
        \label{fig:klvsreward}
    \end{minipage}
    \begin{minipage}{0.5\linewidth}
        \centering
        \includegraphics[width=\textwidth]{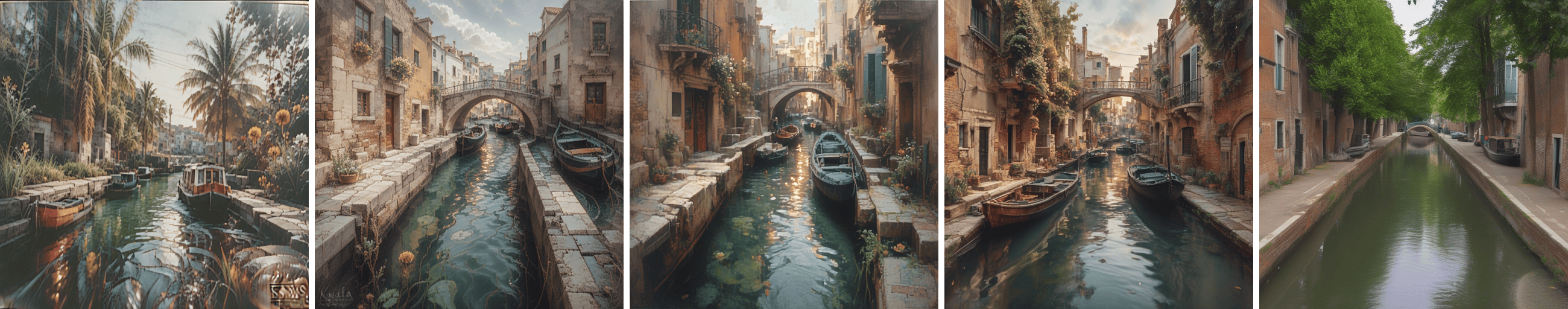}    
        \includegraphics[width=\textwidth]{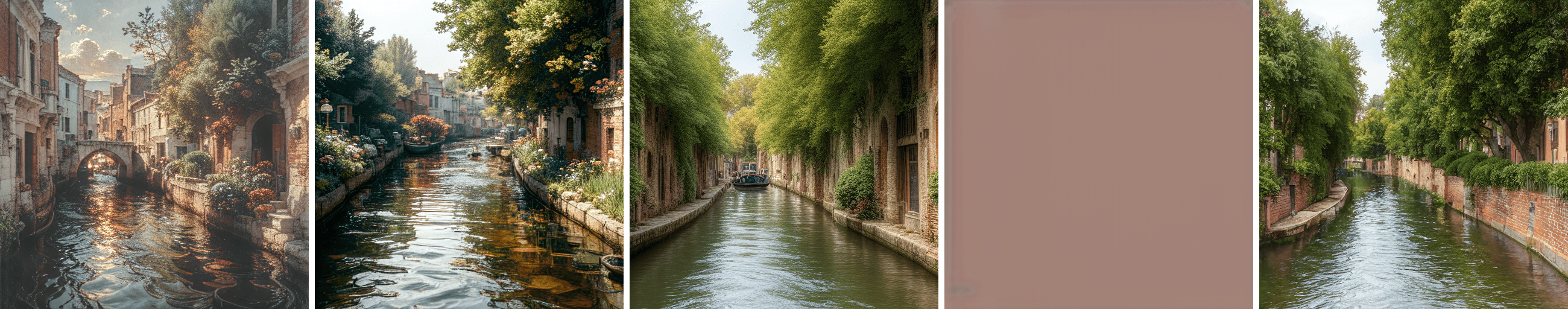}    
        \begin{tabularx}{\textwidth}{YYYYY}
           $\lambda = 0$ &  $\lambda = 0.1$ &  $\lambda = 2$ & $\lambda = 10$ & Base \\
        \end{tabularx}
        \subcaption{Qualitative comparison of SDXL (\textit{top row}) and SD1.4 (\textit{bottom row}) trained on the Pickscore model. SDXL exhibits high quality images at $\lambda=10$ while SD1.4 breaks down at those values, and performs best at $\lambda = 0.1$.}
        \label{fig:klvsreward-pickscore}
    \end{minipage}
    \begin{minipage}{0.495\linewidth}
        \centering
        \includegraphics[width=\textwidth]{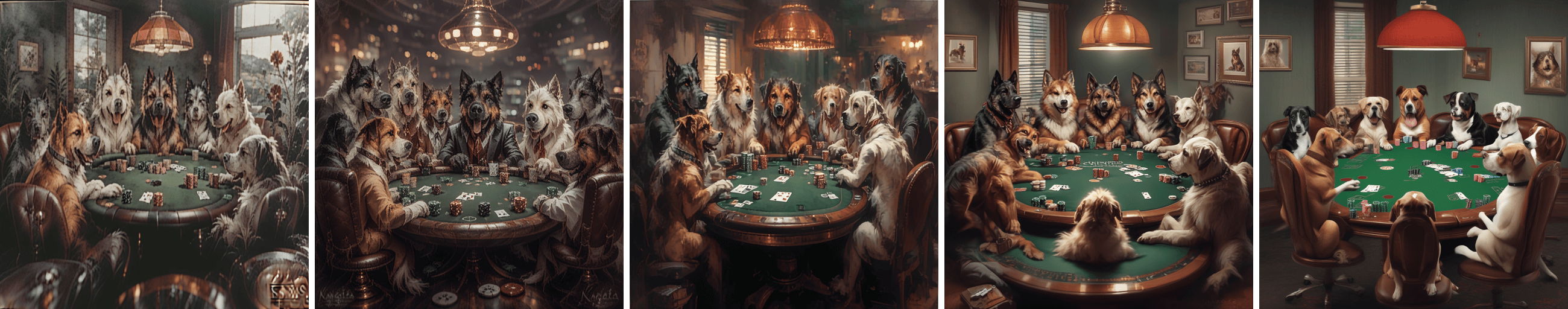}    
        \includegraphics[width=\textwidth]{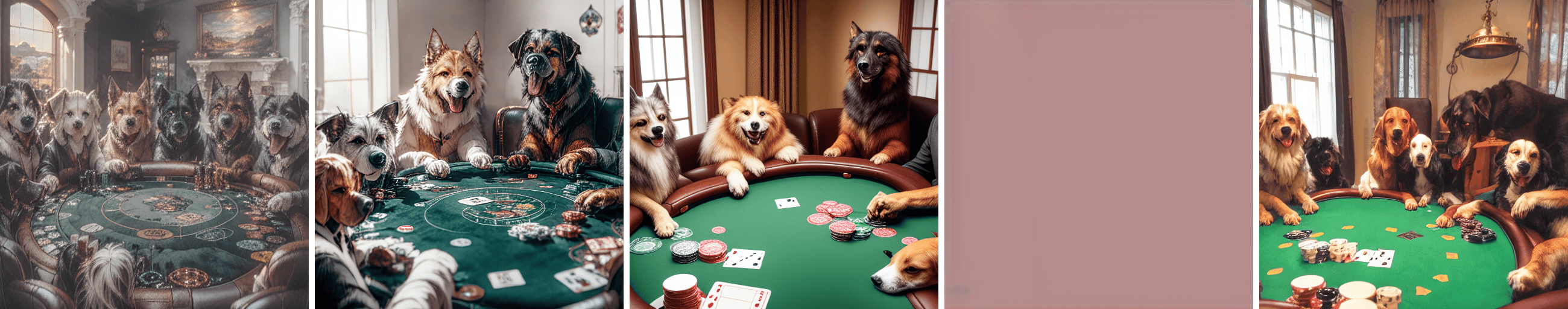}    
        \begin{tabularx}{\textwidth}{YYYYY}
           $\lambda = 0$ &  $\lambda = 0.1$ &  $\lambda = 2$ & $\lambda = 10$ & Base \\
        \end{tabularx}
        \subcaption{Qualitative comparison of SDXL (\textit{top row}) and SD1.4 (\textit{bottom row}) trained on the HPSv2 model. SDXL exhibits high quality images at $\lambda=10$ while SD1.4 breaks down at those values, and performs best at $\lambda = 0.1$.}
        \label{fig:klvsreward-hps}
    \end{minipage}
    \caption{\textbf{Effect of KL divergence on {\draft}} The KL hyperparameter $\lambda$ affects different architectures' training dynamics differently, leading to model collapse for the smaller SDv1.4 model at higher values.
    Moreover, each point on the plot requires a separate run, consuming 32 H100 GPU hours each for hyperparameter tuning.}
    \label{fig:klanalysis}
\end{figure*}

\textbf{Reward Hacking}
RL-based and direct finetuning methods including {\draft} are known to be vulnerable to a phenomenon called \textit{reward hacking}~\cite{dpo,draft,skalse2022defining,fan2024reinforcement,diffusiondpo}, which refers to the model only producing a very small subset of samples.
\cite{draft} posits ``Reward hacking points to deficiencies in existing rewards; efficiently finding gaps between the desired behavior and the behavior implicitly captured by the reward is a crucial step in fixing the reward.''
Consequently, one of the interesting experiments they tried is performing aggressive reward dropout.
We prove that reward hacking is not due to deficiencies in existing rewards, but is an unavoidable artifact of the \textit{Expected Reward Maximization} problem formulation.
We first show the inevitability of reward hacking in an unrestricted setting.
\begin{lemma}
    \textbf{Inevitability of Reward Hacking} In a non-parameteric setting under the expected reward maximization optimization, the optimal probability distribution $p(\mbx_0|\mbc)$ collapses to $p(\mbx_0|\mbc) = \delta(\mbx_0 - \mbx_0^{*rc}), \mbx_0^{*rc} = \arg\max_{\mbx_0} r(\mbx_0, \mbc)$, where $\delta$ is the Dirac-delta function.
\end{lemma}
\begin{proof}
Proof is in \cref{sec:rewardproof}. The proof also shows that reward dropout does not work because it is simply reward-maximization under a new reward function.
\end{proof}

Even with low-dimensional parameteric updates like LoRA finetuning that restricts the set of admissible distributions of $p_\mbftheta(\mbx_0|\mbc)$, the generated images exhibit severe mode collapse (see \cref{sec:app-qualitative}).
We discuss existing regularization schemes to preserve the diversity of the model.

\subsection{Existing Regularizations}
We discuss existing proposed attempts to prevent reward hacking for diffusion models, identify a unifying interpretation for these regularizations, and propose our method that tackle limitations of existing regularizations.
\subsubsection{KL divergence}
A straightforward strategy to mitigate the reward hacking behavior is to introduce a KL divergence term between the base and finetuned distribution, denoted as $p_{\textbase}$ and $p_\mbftheta$ respectively.
This new objective is written as
\begin{align}
    \resizebox{0.45\textwidth}{!}{%
    $J(\mbftheta) = \mathbb{E}_{\mbc, \mbx_0\sim p_\mbftheta} \left[-r(\mbx_0, \mbc) + \lambda \text{KL}\left(p_\theta(\mbx_0|\mbc) || p_{\textbase}(\mbx_0|\mbc)\right)\right]$
    }
    \label{eq:map}
\end{align}
For diffusion models, the second term does not have a closed form.
However, the second term is upper bounded by an ELBO term of the form $\sum_{t=1}^{T}\text{KL}(p_\theta(\mbx_{t-1}|\mbx_t, \mbc) || p_\textbase(\mbx_{t-1}|\mbx_t, \mbc))$ (see Appendix in ~\cite{fan2024reinforcement} for proof).
Substituting the term,
\begin{equation}
    \resizebox{0.45\textwidth}{!} {%
    $J(\mbftheta) = \mathbb{E}_{\mbc, \mbx_0\sim p_\mbftheta} \left[-r(\mbx_0, \mbc) + \sum_{t=1}^{T} \lambda \| \epsilon_\theta(\mbx_t, \mbc) - \epsilon_{\textbase}(\mbx_t, \mbc) \|_2^2\right]$
    }
    \label{eq:opt}
\end{equation}

Similar to \draft-K, we only calculate gradients through the last K steps of this loss function.
In this paper, we extend some of our previous qualitative analysis~\cite{blog} using KL divergence for \draft-based reward finetuning.
Although KL regularization is a straightforward addition to prevent reward hacking, this regularization suffers from a few practical considerations.
First, the hyperparameter $\lambda$ has to be chosen before finetuning, leading to substantial compute cost to find the optimal tradeoff between reward and diversity.
Second, the interpretation of $\lambda$ depends on the parameterization $\mbftheta$, i.e. the same $\lambda$ gives widely different solutions for different architectures. 
This is evident in \cref{fig:klanalysis} where SDv1.4 and SDXL admit `optimal' qualitative performance for very different ranges of $\lambda$ (SDv1.4 breaks down at higher values of $\lambda$ where SDXL works well).
Third, there is a problem of indiscriminate regularization (see \cref{sec:referencemismatch}) and direct regularization with respect to the base model does not account for `reference mismatch', leading to a suboptimal reward-diversity tradeoff.

\subsubsection{LoRA scaling}
{\draft} also proposes LoRA scaling to `interpolate' between the base and finetuned model.
Specifically, for a finetuned model with weights $\mbtheta^* = \mbtheta_\textbase + \mbtheta_\textlora$, LoRA scaling uses the weights $\mbtheta^*_{\text{new}} = \mbtheta_\textbase + \alpha'\mbtheta_\textlora, \alpha' \in (0, 1)$ to produce samples.
Although this trick was proposed only empirically, we can theoretically interpret LoRA scaling as an implicit weight regularization over the network parameters w.r.t. the base model, with mild assumptions.

\begin{lemma}
Consider a function $f(\mbftheta)$ with minimizer $\mbtft = \mbtheta_\textbase + \mbtheta_\textlora$, and $f$ is locally convex around $\mbtft$, 
Then, $\mbtheta^*_{\text{new}} = \mbtheta_\textbase + \alpha'\mbtheta_\textlora$ is the minimizer of $f$ with an additional Tikhonov (L$_2$) regularization on the parameter $\mbtheta$.
\end{lemma}

\begin{proof}
Since $f$ is assumed to be locally convex at $\mbtft$, we can write $f(\mbtheta) \approx f(\mbtft) + (\mbtheta - \mbtft)^T \Ht (\mbtheta - \mbtft)$, where $\Ht$ is the Hessian at $\mbtft$.
The first term in the Taylor series is 0 because $\nabla f(\mbtft) = 0$ since $\mbtft$ is a minimizer, and $\Ht$ is positive semi-definite because of the locally convex assumption of $f$.
Since $\Ht$ represents the local convexity around $\mbtft$, we consider an augmented minimization objective with a Mahalanobis distance regularization from $\mbtheta_\textbase$ as follows:
\begin{align}
    \arg\min_\mbtheta \left[f(\mbtheta) + \alpha (\mbtheta - \mbtheta_\textbase)^T \Ht (\mbtheta - \mbtheta_\textbase) \right]
\end{align}
\begin{equation}
    \resizebox{0.45\textwidth}{!}{%
    $\approx f(\mbtft) + (\mbtheta - \mbtft)^T \Ht (\mbtheta - \mbtft) + \alpha (\mbtheta - \mbtheta_\textbase)^T \Ht (\mbtheta - \mbtheta_\textbase)$}
\end{equation}
Taking gradients w.r.t $\mbtheta$ and setting it to 0, we get:
\begin{align}
    &\Ht(\mbtheta - \mbtft) + \alpha \Ht(\mbtheta - \mbtheta_\textbase) = 0 \\
    &\implies \mbtheta = \frac{\mbtft}{1 + \alpha} + \frac{\alpha \mbtheta_\textbase}{1 + \alpha}  
\end{align}
Since $\mbtft$ is obtained using LoRA finetuning from the base model, i.e. $\mbtft = \mbtheta_\textbase + \mbtheta_\textlora$, we substitute this and get:
\begin{equation}
    \mbtheta = \mbtheta_\textbase + \left(\frac{1}{1 + \alpha}\right)\mbtheta_\textlora = \mbtheta_\textbase + {\alpha'}\mbtheta_\textlora 
\end{equation}
Since we have $\alpha \in [0, \infty)$, we have $\alpha' = \frac{1}{1 + \alpha} \in (0, 1]$.
\end{proof}

Therefore, if the minimizer of $f(\mbtheta)$ is $\mbtheta_\textbase + \mbtheta_\textlora$, then the minimizer of $f(\mbtheta) + \alpha(\mbtheta - \mbtheta_\textbase)^T \Ht (\mbtheta - \mbtheta_\textbase)$ is $\mbtheta_\textbase + \alpha'\mbtheta_\textlora$.
LoRA scaling can be interpreted as the solution to the minimization objective with a Mahalanobis regularization in the parameter space, with some convexity assumptions about the function $f$ (in this case, $f(\mbtheta) = \mathbb{E}_{\mbc, \mbx_0\sim p_\mbftheta} \left[-r(\mbx_0, \mbc)\right]$).

Unlike KL regularization, LoRA scaling is an inference time regularization and can be used without additional compute. 
We note that LoRA scaling being a regularization on the parameter space essentially suffers from the same issues as KL divergence (which is a regularization on the output space), namely, reference mismatch and indiscriminate regularization on the weight space irrespective of the timestep in the sampling chain.  \\

\textbf{Indiscriminate regularization and reference mismatch}
\label{sec:referencemismatch}
Both KL regularization and LoRA scaling suffer from the indiscriminate regularization, i.e. the regularization affects the entire sampling chain.
These regularizations are fine for Large Language Models (LLMs) because the entire output sequence of tokens constitutes the response to a prompt (all output tokens contribute equally to the log-likelihood), so regularization must be applied on the entire sequence.
In contrast, only the end of the sampling chain constitutes a sample from the image distribution, reflecting the asymmetric nature of the outputs of the sampling chain.
Moreover, we show that early steps in the sampling chain are responsible for recovering modes, and later steps are responsible for recovering fine-grained details (see \cref{sec:mixingdynamics}).
Since the goal of regularization is to recover multiple modes while obtaining high rewards, the regularization must be applied asymmetrically -- higher during the earlier timesteps, and lower during the later timesteps.
Furthermore, both KL divergence and LoRA scaling are regularizations with respect to the base model (in output and weight spaces respectively), discounting the reference mismatch between the base and finetuned models due to regularizing the later steps of the sampling chain.
We introduce an alternate regularization that addresses all these issues, and is shown to obtain better reward-diversity tradeoffs.

\subsection{\method}
\label{sec:aig}

Our proposed algorithm - \method, builds on the Annealed Importance Sampling (AIS)~\cite{neal2001annealed} interpretation of NCSN ~\cite{song2019generative}.
Simulated Annealing (SA) samples from a distribution $p_0(\mbx)$ by using a set of proposal distributions $p_t(\mbx) \propto p_0(\mbx)^{\beta_t} , 1 = \beta_0 > \beta_1 > \ldots \beta_T$.
A sample is generated by first sampling from $p_T(\mbx)$ and repeatedly transitioning from $p_{t}$ to $p_{t-1}$ until we get a sample from $p_0(\mbx)$.
Here, $p_T$ is the uniform distribution, which is easier to sample from, and Langevin dynamics can be run to converge to a sample from $p_0(\mbx)$.
Similarly, the diffusion models define $p_t(\mbx) \propto \int_y p_0(\mby)G(\tilde{\alpha}_t \mby, \mbx, t) d\mby$ as `smoothened' versions of $p_0(\mbx)$, and anneal from $p_T(\mbx) \approx \mathcal{N}(0, 1)$ to $p_0(\mbx)$. %
Specifically, the reverse-time SDE or probability ODE (as in ~\cref{eq:revode}) formulation can be thought of as performing SA with the proposal distributions $p_t(\mbx_t)$ using the Langevin dynamics defined by its score function $\nabla_{\mbx_t} \log(p_t(\mbx_t)) \approx f_\mbtheta(\mbx_t, t)$. %

In contrast, AIS~\cite{neal2001annealed} defines the intermediate proposal distributions as $p_t(\mbx) \propto p_0(\mbx)^{\beta_t}p_n(\mbx)^{1-\beta_t}$, where $p_0$ is the distribution of interest, and $p_n$ is a distribution we can sample from. 
In our case, $p_n(\mbx)$ is the base distribution, and $p_0(\mbx)$ is the {\draft} distribution.
Therefore, we propose the following transition probability distributions:
\begin{equation}
    p_{\text{\methodabbv}}(\mbx_t, t) = p_\textbase(\mbx_t, t)^{\gamma(t)} p_\mbtheta(\mbx_t, t)^{(1 - \gamma(t))}
\end{equation}
for an monotonic function $\gamma(t)$ with $\gamma(0) = 0, \gamma(T) = 1$.
Intuitively, this family of proposal distributions ensures that the Langevin mixing dynamics are governed by $p_\textbase$ in the earlier timesteps, and $p_\mbtheta$ in the later timesteps.
Early mixing from $p_\textbase$ ensures that modes from the data distribution are recovered (see \cref{sec:mixingdynamics}), later mixing from $p_\mbtheta$ then pushes the noisy data to the nearest high-reward mode.
The modified annealed Langevin dynamics are therefore governed by the score function:
\begin{equation}
    \resizebox{0.45\textwidth}{!}{%
    $\nabla_{\mbx_t}\log(p_{\text{\methodabbv}}(\mbx_t, t)) \approx \gamma(t)f_\textbase(\mbx_t, t) + (1-\gamma(t))f_\mbtheta(\mbx_t, t)$ %
    }
    \label{eq:aigscore}
\end{equation}
\cref{eq:aigscore} is now substituted to the probability ODE~\cref{eq:revode} to sample from the modified distribution.

{\methodabbv} \textit{implicitly} regularizes the earlier timesteps of the sampling chain since the influence of the {\draft} model is negligible in the beginning and is dominant only in the later timesteps due to the scheduling of $\gamma$. 
Moreover, since the influence of the base model becomes negligible in the later sampling steps, the sample $\mbx_t$ is not bound by any regularization and is allowed to freely converge according to the dynamics of the score function of the {\draft} model.
This helps alleviate the reference mismatch problem.
{\method} is also inference-time, allowing the end-user to experiment and find the best tradeoff without a lot of compute.
The proposed method is similar to classifier-free guidance -- using a weighted average of two score functions to achieve better samples.
However unlike classifier-free guidance, the score functions are \textit{interpolated}, and the mixing weight $\gamma(t)$ is changed over the course of sampling.

\section{Experiments}
\label{sec:experiments}

In this section, we compare the effects of various forms of regularization, and showcase finetuning applications. \\

\textbf{Implementation Details}
We consider checkpoints from two models - Stable Diffusion v1.4 (SD1.4) and Stable Diffusion XL Base (SDXL) for our experiments. 
We ablate on two popular reward models - the Human Preference Score (HPSv2)~\cite{hps} and Pickscore~\cite{pickscore}.
All experiments use DRaFT-1 style training, wherein gradients are computed only through the last sampling step, since that exhibits the most stable training dynamics and high rewards~\cite{draft}.
We use the DDIM sampler for SD1.4 and EDM sampler for SDXL, with 25 sampling steps.
We use a classifier-free guidance weight of 7.5 for SD1.4 and 5.0 for SDXL, following best practices.
All finetuning is implemented and run using the NVIDIA NeMo framework ~\cite{nemo}.
We use a subset of 50k prompts from the Pick-a-Pic dataset for finetuning with the reward model.
All {\draft} and \draft+KL models are trained on a single node with 8 H100 GPUs for 4 hours, with DDP-level parallelism.
More implementation details are present in Appendix ~\cref{sec:impldetails}.
\\

\begin{figure*}[ht!]
    \centering
    \begin{minipage}{0.85\textwidth}
        \includegraphics[width=\linewidth]{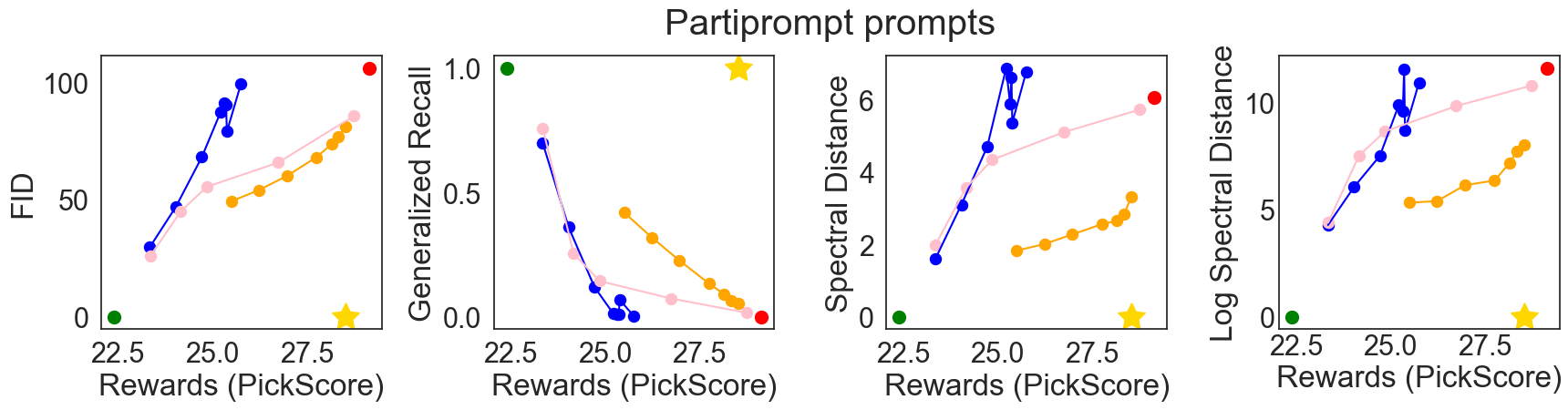}
        \includegraphics[width=\linewidth]{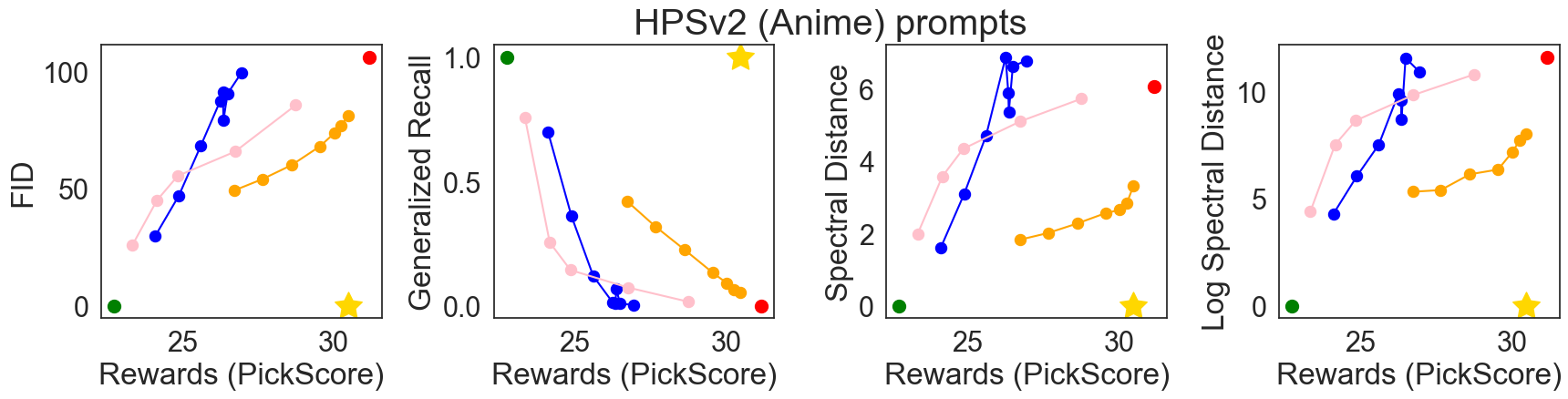}
        \subcaption{\textbf{Reward-diversity tradeoff for SDXL model trained on PickScore}}
        \label{fig:sdxl_pickscore_partiprompt}
    \end{minipage}
    \begin{minipage}{0.12\textwidth}
        \includegraphics[width=\linewidth]{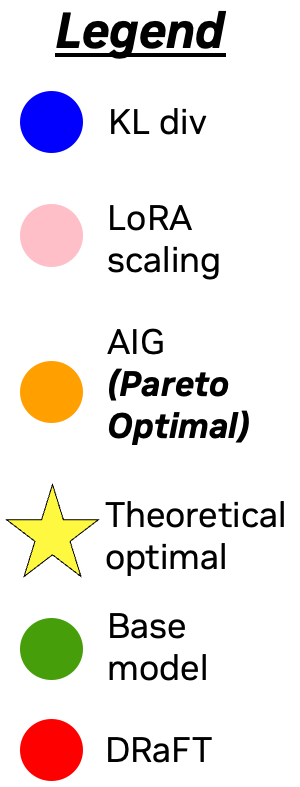}
    \end{minipage}

    \begin{minipage}{0.85\textwidth}
        \includegraphics[width=\linewidth]{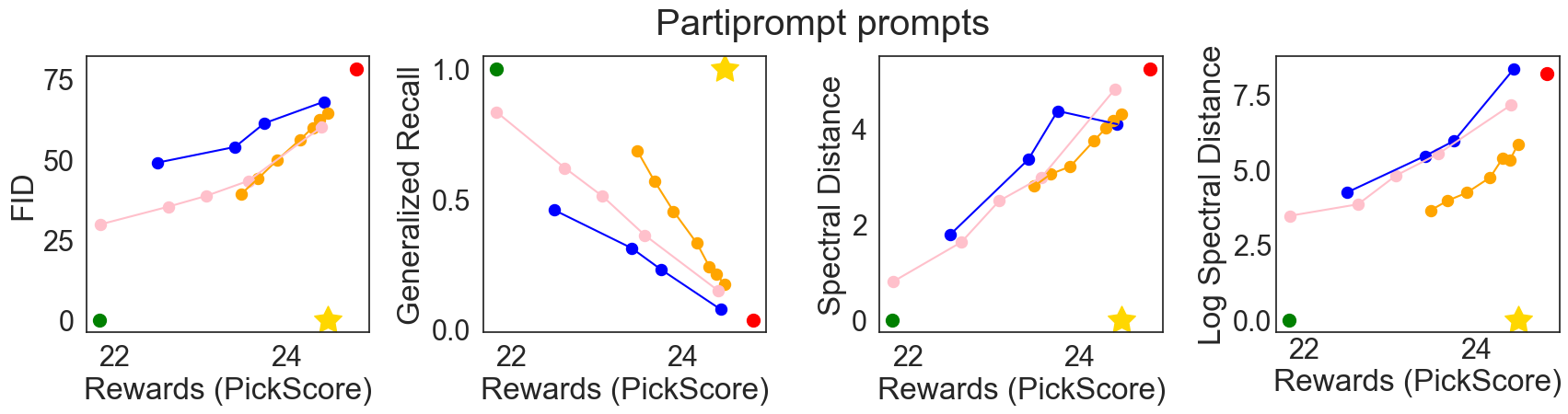}
        \includegraphics[width=\linewidth]{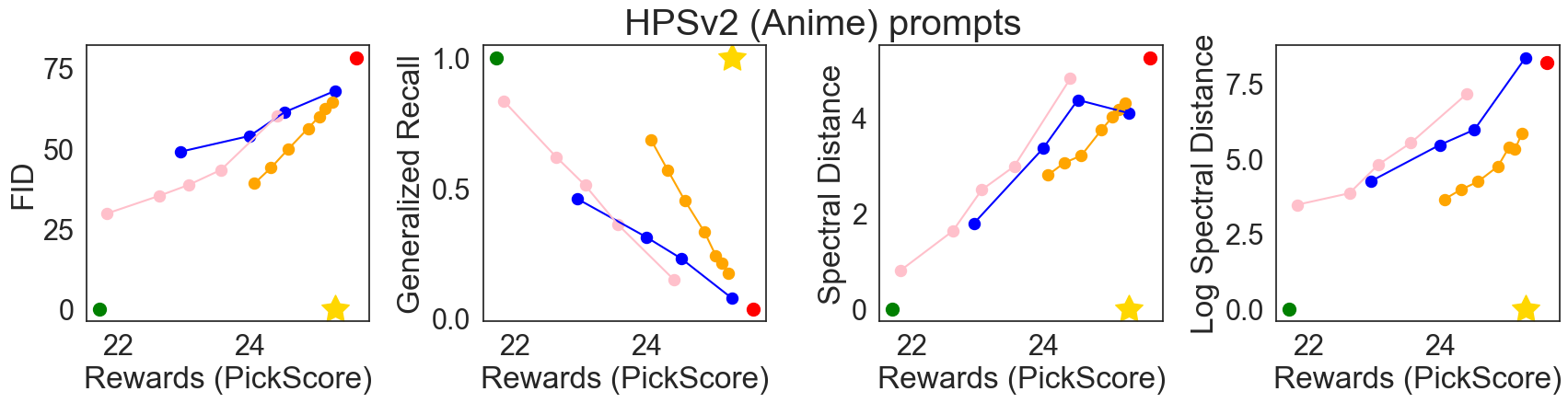}
        \subcaption{\textbf{Reward-diversity tradeoff for SDv1.4 model trained on PickScore}}
        \label{fig:sd_pickscore_partiprompt}
    \end{minipage}
    \begin{minipage}{0.12\textwidth}
        \includegraphics[width=\linewidth]{images/graph-legend.png}
    \end{minipage}

    \caption{\scriptsize \textbf{Reward-diversity tradeoff for models trained on PickScore}: \textcolor{green}{\textbf{Green}} represents the base model, \textcolor{red}{\textbf{Red}} represents {\draft} with no regularization, \textcolor{yellow}{\textbf{Gold}} star represents the ideal score. \textcolor{blue}{\textbf{Blue}} represents different models with different KL regularization coefficients $\lambda$, \textcolor{pink}{\textbf{Pink}} represents different amounts of LoRA scaling, and \textcolor{orange}{\textbf{Orange}} represents different $\gamma(t)$ for \methodabbv.
    An ideal baseline would achieve the highest reward (represented by \draft) as well as a complete match with the base distribution.
    All diversity metrics are computed with respect to the base model, therefore the base model achieves the ideal diversity score (1 for Recall, 0 for everything else).
    For all measures for both PartiPrompt and HPSv2 prompts, the {\methodabbv} regularization achieves Pareto-optimality for SD and SDXL.
    LoRA scaling is competitive in FID, but suffers from low recall and higher spectral distance, especially for SDXL.
    More coverage metrics are in \cref{sec:app-reward-diversity}.
    }
    \label{fig:quantitative-rewarddiversity-tradeoff}
\end{figure*}

\vspace*{-10pt}
\subsection{Evaluation Measures}
\label{sec:eval}
First, we generate images from the Partiprompts and HPDv2 prompt datasets. 
The Partiprompt prompt dataset encompasses prompts from a variety of concept classes, while the HPDv2 focuses on stylistic features - evident by its 4 subclasses (photos, anime, concept-art, paintings).
Then, we evaluate the Pickscore reward, HPSv2 reward, and CLIP alignment on these generated images.
Furthermore, we explicitly compute the `diversity/coverage' with respect to the base model.
To this end, we manually choose a set of 40 complex prompts from the PartiPrompt dataset (referred to as the `coverage prompt dataset'), and generate 50 images for each prompt with the same random seeds.
The distribution of these images is then compared with that of the base model.
We evaluate the reward scores in conjunction with four coverage measures, to analyse the Pareto efficiency of each finetuning method.
We choose the following coverage evaluation criteria:

\textbf{Frechet Inception Distance}~\cite{fid}
This metric captures the discrepancy in feature distribution considering a Gaussian distribution for the features of each distribution. 
Given samples from two distributions $D_1 = \{ f_1^{(1)}, f_1^{(2)} \ldots f_1^{(n)} \}$ and $D_2 = \{ f_2^{(1)}, f_2^{(2)} \ldots f_2^{(m)} \}$, the FID is defined as 
\begin{equation}
    \text{FID}(D_1, D_2) = \| \mu_1 - \mu_2 \|_2^2 + Tr(\Sigma_1 + \Sigma_2 - 2(\Sigma_1\Sigma_2)^{\frac{1}{2}})
\end{equation}
where $\mu$ and $\Sigma$ are the mean and standard deviation of the datasets, and $Tr$ is the trace of the matrix. 

\begin{figure}[ht!]
    \centering
    \includegraphics[width=0.99\linewidth]{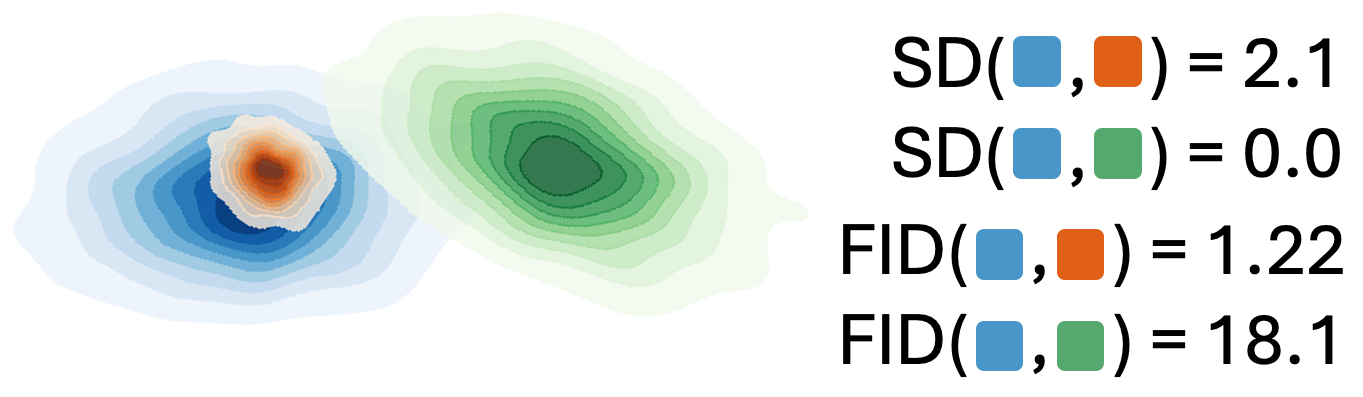}
    \caption{\footnotesize{\textbf{Spectral (Covariance) Distance}. Given a base distribution \Done and two distributions (\Dtwo, \Dthree), SCD aims to measure the similarity of distributions taking `reference mismatch'~\cite{hong2024margin} into account. In this case, {\Dthree} has the same distribution spread as \Done, but reference mismatch in the form of shifted mean and rotated covariance. {\Dtwo} has a mean close to that of {\Done} but has a tighter distribution. In this case, FID characterizes {\Dtwo} to be the preferred distribution, but Spectral Distance characterizes \Dthree.}
    }
    \label{fig:spectral}
\end{figure}

\textbf{Generalized Recall}
\cite{precisionrecall} introduce the classic concepts of precision and recall to the study of generative models, motivated by the observation that FID cannot be used for making conclusions about precision and recall separately.
While high precision implies more realism of the images compared to the base distribution, high recall implies higher coverage of the data distribution by the generator.
~\cite{improvedprc} propose to form non-parametric representations of the data manifolds using overlapping hyperspheres using kNN with hyperparameter $k$ (nearest neighbors), followed by binary assignments of each data point to the manifold to compute recall.
We compute the recall of the distribution with respect to the base model, which covers multiple modes of the dataset.
We use $k=10$ in our experiments.

\textbf{Spectral Covariance Distance}
One limitation of FID and Generalized Recall is that these metrics were created to ensure \textit{exact} overlap between the data and generator distribution.
Consequently, they penalize any `reference mismatch' w.r.t. the base distribution as well.
A good coverage metric should only penalize differences in the `spread' of the distribution, ignoring any reference mismatch due to translational and rotational differences.
To measure the `spread', consider the eigendecomposition of covariance of the dataset features as $\Sigma_1 = \mathbb{E}[(x - \mu)(x - \mu)^T] = U_1 \Lambda_1 U_1^T$ and $\Sigma_2 = U_2 \Lambda_2 U_2^T$. 
Rotating the data from $\done$ by rotation $R = U_2 U_1^T$, the new covariance matrix of the data becomes $\Sigma_1' = U_2 \Lambda_1 U_2^T$, whose principal eigenvectors now align with that of $\dtwo$.
We propose the Spectral distance as the differences in corresponding eigenvalues
\begin{equation}
    \resizebox{0.4\textwidth}{!}{%
    $\text{SCD}(\done, \dtwo) = \| \Sigma_1' - \Sigma_2 \|_2^2 = \sum_{i=1}^{N} (\lambda_1^{(i)} - \lambda_2^{(i)})^2$
    }
    \label{eq:scd}
\end{equation}
Since the eigenvalues are non-negative, and denote scaling along the principal axes, one can consider a metric using relative scales:
\begin{equation}
    \resizebox{0.5\textwidth}{!}{%
    $\text{LSCD}(D_1, D_2) =  \sum_{i=1}^{N} \left( \log\left(\frac{\lambda_1^{(i)}}{\lambda_2^{(i)}}\right) \right)^2 = \sum_{i=1}^{N} \left( \log(\lambda_1^{(i)}) - \log(\lambda_2^{(i)}) \right)^2$
    }
    \label{eq:logscd}
\end{equation}
\cref{eq:scd,eq:logscd} penalize absolute and relative deviations in coverage of the space, without penalizing any translational and rotational reference mismatch (see \cref{fig:spectral}). 
We use the default 2048 dimensional \texttt{pool3} layer from the Inception network~\cite{szegedy2015going} for computing all coverage metrics.

\subsection{Reward-KL tradeoff}
First, we train SD1.4 and SDXL on both the HPSv2 and Pickscore reward functions.
We choose a variety of hyperparameter values for KL regularization, i.e. $\lambda = 0.1, 0.5, 1, 2, 5, 10, 20, 50$.
Each hyperparameter run requires 4 hours of training on 8 NVIDIA H100 GPUs.
Note that SD1.4 and SDXL share the same underlying UNet architecture that differ only in the number of parameters and conditioning (the latter also accepts top-left location and original image sizes as additional inputs).
Fig.~\cref{fig:klvsreward} shows the quantitative reward-KL tradeoff, and ~\cref{fig:klvsreward-pickscore,fig:klvsreward-hps} shows qualitative effects of training with different $\lambda$ and datasets.
We note that for the larger $\lambda$, the SDv1.4 model breaks down during training, and for smaller $\lambda$ SDXL does not produce qualitatively better images.
Therefore, the choice of optimal $\lambda$ is dependent on the model architecture and requires expensive finetuning.

\subsection{Reward-diversity tradeoff}
\label{sec:rewarddiversity}
In this section, we show a more comprehensive comparison of models with diversity metrics relative to the base model.
\cref{fig:quantitative-rewarddiversity-tradeoff} shows the Pareto fronts for models regularized with KL divergence, LoRA scaling, and \methodabbv, for SD and SDXL trained on Pickscore, and evaluated on PartiPrompt and HPSv2 datasets.
The Pareto front is computed by choosing different hyperparameters for each regularization ($\lambda$ for KL, $\alpha'$ for LoRA, $\gamma$ for \methodabbv), and plotting a curve.
In all cases, {\methodabbv} achieves a better reward-diversity tradeoff.
Pareto fronts on other splits of the HPSv2 dataset, DMs trained on HPSv2 rewards and qualitative comparisons are in \cref{sec:app-reward-diversity}.  \\

\vspace*{-8pt}
\textbf{Qualitative Results} Due to space constraints, we show qualitative results and comparisons in \cref{sec:app-qualitative}.

\subsection{User Preference Study}
\label{sec:userstudy}

\begin{figure}[ht!]
    \centering
    \begin{minipage}{\linewidth}
        \includegraphics[width=\linewidth]{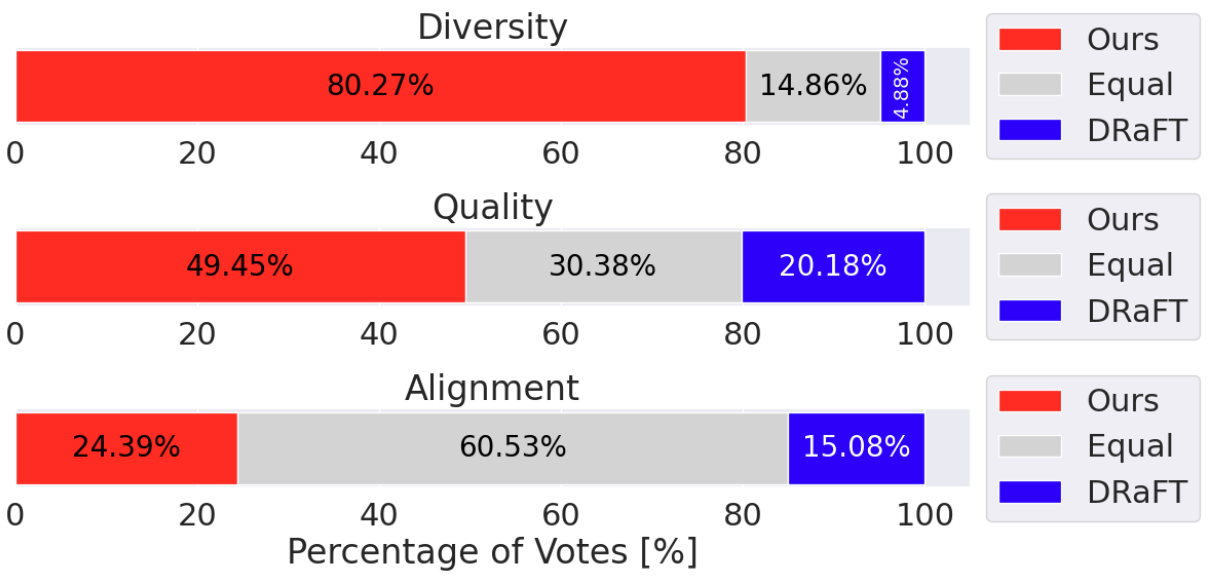}
    \end{minipage}
    \caption{\textbf{User study across all DM and reward configurations} Our method (\methodabbv) demonstrates exceptional diversity and quality while preserving alignment.}
    \label{fig:userstudy-overall}
    \vspace*{-8pt}
\end{figure}

We perform a rigorous user study to verify the efficacy of {\methodabbv} compared to {\draft} in terms of diversity, image quality, and T2I alignment.
Our user study had {\numusers} participants resulting in more than {\numvotes} votes.
More ablations and details about the user study setup are outlined in \cref{sec:app-user-study}.
\cref{fig:userstudy-overall} shows the overall win rates of our method compared to \draft.
{\methodabbv} demonstrates exceptionally high diversity \textit{and} quality, without sacrificing T2I alignment. %
We also show results aggregated by DM and reward model configurations (\cref{fig:userstudy-indconfig}) to verify that user preferences are not skewed towards any particular configuration.
User preferences are consistent across DM and reward configurations, showing that {\methodabbv} is the primary contributor to high diversity and quality according to users. 

\section{Conclusion}
\label{sec:conclusion}

In this paper, we studied the problem of \textit{reward hacking} for T2I diffusion models and presented a simple proof of model collapse, necessitating regularization.
We analyse and interpret both KL divergence and LoRA scaling as regularization with respect to the base model in the output and parameter spaces respectively.
This is followed by identifying pitfalls of existing regularizations, namely the training requirements of the KL hyperparameter $\lambda$, a suboptimal reward-diversity tradeoff and reference-mismatch due to indiscriminate regularization over the entire sampling chain.
Inspired by Annealed Importance Sampling, we propose {\methodabbv}, an inference-time regularization that gradually anneals the score function from the base model to the {\draft} model in the reverse-time sampling chain.
The mixing dynamics of the earlier steps use the base model to separate multiple modes from noise, and the later steps use the {\draft} model to finetune the latent towards generating high-reward images.
The gradual shift of score function leads to little to no regularization during the later timesteps from the base model, mitigating reference-mismatch during the finetuning steps and reducing bias. %
In addition to existing metrics, we propose two coverage metrics that take reference mismatch into account.
Extensive quantitative results show that {\methodabbv} achieves Pareto-optimal reward-diversity tradeoffs, and a user study shows {\methodabbv} improves the diversity \textit{and} quality of samples without sacrificing alignment across different DM and reward models. 
We believe this work provides motivation and groundwork for a deeper understanding of the nature of reward optimization while preserving characteristics of the generated data distribution. 
We discuss limitations in \cref{sec:limitations}.

\clearpage
{\small
\bibliographystyle{ieee_fullname}
\bibliography{egbib}

\begin{thebibliography}{10}\itemsep=-1pt

\bibitem{achiam2023gpt}
Josh Achiam, Steven Adler, Sandhini Agarwal, Lama Ahmad, Ilge Akkaya, Florencia~Leoni Aleman, Diogo Almeida, Janko Altenschmidt, Sam Altman, Shyamal Anadkat, et~al.
\newblock Gpt-4 technical report.
\newblock {\em arXiv preprint arXiv:2303.08774}, 2023.

\bibitem{ali2022spot}
Hazrat Ali, Shafaq Murad, and Zubair Shah.
\newblock Spot the fake lungs: Generating synthetic medical images using neural diffusion models.
\newblock In {\em Irish Conference on Artificial Intelligence and Cognitive Science}, pages 32--39. Springer, 2022.

\bibitem{black2023training}
Kevin Black, Michael Janner, Yilun Du, Ilya Kostrikov, and Sergey Levine.
\newblock Training diffusion models with reinforcement learning.
\newblock {\em arXiv preprint arXiv:2305.13301}, 2023.

\bibitem{black2023zero}
Kevin Black, Mitsuhiko Nakamoto, Pranav Atreya, Homer Walke, Chelsea Finn, Aviral Kumar, and Sergey Levine.
\newblock Zero-shot robotic manipulation with pretrained image-editing diffusion models.
\newblock {\em arXiv preprint arXiv:2310.10639}, 2023.

\bibitem{blattmann2023align}
Andreas Blattmann, Robin Rombach, Huan Ling, Tim Dockhorn, Seung~Wook Kim, Sanja Fidler, and Karsten Kreis.
\newblock Align your latents: High-resolution video synthesis with latent diffusion models.
\newblock In {\em Proceedings of the IEEE/CVF Conference on Computer Vision and Pattern Recognition}, pages 22563--22575, 2023.

\bibitem{videoworldsimulators2024}
Tim Brooks, Bill Peebles, Connor Holmes, Will DePue, Yufei Guo, Li Jing, David Schnurr, Joe Taylor, Troy Luhman, Eric Luhman, Clarence Ng, Ricky Wang, and Aditya Ramesh.
\newblock Video generation models as world simulators.
\newblock 2024.

\bibitem{bussi2007accurate}
Giovanni Bussi and Michele Parrinello.
\newblock Accurate sampling using langevin dynamics.
\newblock {\em Physical Review E—Statistical, Nonlinear, and Soft Matter Physics}, 75(5):056707, 2007.

\bibitem{carvalho2023motion}
Joao Carvalho, An~T Le, Mark Baierl, Dorothea Koert, and Jan Peters.
\newblock Motion planning diffusion: Learning and planning of robot motions with diffusion models.
\newblock In {\em 2023 IEEE/RSJ International Conference on Intelligent Robots and Systems (IROS)}, pages 1916--1923. IEEE, 2023.

\bibitem{chen2023textdiffuser}
Jingye Chen, Yupan Huang, Tengchao Lv, Lei Cui, Qifeng Chen, and Furu Wei.
\newblock Textdiffuser-2: Unleashing the power of language models for text rendering.
\newblock {\em arXiv preprint arXiv:2311.16465}, 2023.

\bibitem{TextDiffuser}
Jingye Chen, Yupan Huang, Tengchao Lv, Lei Cui, Qifeng Chen, and Furu Wei.
\newblock Textdiffuser: Diffusion models as text painters.
\newblock In A. Oh, T. Naumann, A. Globerson, K. Saenko, M. Hardt, and S. Levine, editors, {\em Advances in Neural Information Processing Systems}, volume~36, pages 9353--9387. Curran Associates, Inc., 2023.

\bibitem{christiano2017deep}
Paul~F Christiano, Jan Leike, Tom Brown, Miljan Martic, Shane Legg, and Dario Amodei.
\newblock Deep reinforcement learning from human preferences.
\newblock {\em Advances in neural information processing systems}, 30, 2017.

\bibitem{draft}
Kevin Clark, Paul Vicol, Kevin Swersky, and David~J Fleet.
\newblock Directly fine-tuning diffusion models on differentiable rewards.
\newblock {\em ICLR}, 2023.

\bibitem{dubois2024alpacafarm}
Yann Dubois, Chen~Xuechen Li, Rohan Taori, Tianyi Zhang, Ishaan Gulrajani, Jimmy Ba, Carlos Guestrin, Percy~S Liang, and Tatsunori~B Hashimoto.
\newblock Alpacafarm: A simulation framework for methods that learn from human feedback.
\newblock {\em Advances in Neural Information Processing Systems}, 36, 2024.

\bibitem{reno}
Luca Eyring, Shyamgopal Karthik, Karsten Roth, Alexey Dosovitskiy, and Zeynep Akata.
\newblock Reno: Enhancing one-step text-to-image models through reward-based noise optimization.
\newblock {\em arXiv preprint arXiv:2406.04312}, 2024.

\bibitem{fan2023optimizing}
Ying Fan and Kangwook Lee.
\newblock Optimizing ddpm sampling with shortcut fine-tuning.
\newblock {\em arXiv preprint arXiv:2301.13362}, 2023.

\bibitem{fan2024reinforcement}
Ying Fan, Olivia Watkins, Yuqing Du, Hao Liu, Moonkyung Ryu, Craig Boutilier, Pieter Abbeel, Mohammad Ghavamzadeh, Kangwook Lee, and Kimin Lee.
\newblock Reinforcement learning for fine-tuning text-to-image diffusion models.
\newblock {\em Advances in Neural Information Processing Systems}, 36, 2024.

\bibitem{feng2022training}
Weixi Feng, Xuehai He, Tsu-Jui Fu, Varun Jampani, Arjun Akula, Pradyumna Narayana, Sugato Basu, Xin~Eric Wang, and William~Yang Wang.
\newblock Training-free structured diffusion guidance for compositional text-to-image synthesis.
\newblock {\em arXiv preprint arXiv:2212.05032}, 2022.

\bibitem{hao2024optimizing}
Yaru Hao, Zewen Chi, Li Dong, and Furu Wei.
\newblock Optimizing prompts for text-to-image generation.
\newblock {\em Advances in Neural Information Processing Systems}, 36, 2024.

\bibitem{fid}
Martin Heusel, Hubert Ramsauer, Thomas Unterthiner, Bernhard Nessler, and Sepp Hochreiter.
\newblock Gans trained by a two time-scale update rule converge to a local nash equilibrium.
\newblock In I. Guyon, U.~Von Luxburg, S. Bengio, H. Wallach, R. Fergus, S. Vishwanathan, and R. Garnett, editors, {\em Advances in Neural Information Processing Systems}, volume~30. Curran Associates, Inc., 2017.

\bibitem{ho2022imagen}
Jonathan Ho, William Chan, Chitwan Saharia, Jay Whang, Ruiqi Gao, Alexey Gritsenko, Diederik~P Kingma, Ben Poole, Mohammad Norouzi, David~J Fleet, et~al.
\newblock Imagen video: High definition video generation with diffusion models.
\newblock {\em arXiv preprint arXiv:2210.02303}, 2022.

\bibitem{ho2020denoising}
Jonathan Ho, Ajay Jain, and Pieter Abbeel.
\newblock Denoising diffusion probabilistic models.
\newblock {\em Advances in neural information processing systems}, 33:6840--6851, 2020.

\bibitem{hong2024margin}
Jiwoo Hong, Sayak Paul, Noah Lee, Kashif Rasul, James Thorne, and Jongheon Jeong.
\newblock Margin-aware preference optimization for aligning diffusion models without reference.
\newblock {\em arXiv preprint arXiv:2406.06424}, 2024.

\bibitem{huang2023t2i}
Kaiyi Huang, Kaiyue Sun, Enze Xie, Zhenguo Li, and Xihui Liu.
\newblock T2i-compbench: A comprehensive benchmark for open-world compositional text-to-image generation.
\newblock {\em Advances in Neural Information Processing Systems}, 36:78723--78747, 2023.

\bibitem{hyvarinen2005estimation}
Aapo Hyv{\"a}rinen and Peter Dayan.
\newblock Estimation of non-normalized statistical models by score matching.
\newblock {\em Journal of Machine Learning Research}, 6(4), 2005.

\bibitem{kapelyukh2023dall}
Ivan Kapelyukh, Vitalis Vosylius, and Edward Johns.
\newblock Dall-e-bot: Introducing web-scale diffusion models to robotics.
\newblock {\em IEEE Robotics and Automation Letters}, 8(7):3956--3963, 2023.

\bibitem{karnewar2023holodiffusion}
Animesh Karnewar, Andrea Vedaldi, David Novotny, and Niloy~J Mitra.
\newblock Holodiffusion: Training a 3d diffusion model using 2d images.
\newblock In {\em Proceedings of the IEEE/CVF conference on computer vision and pattern recognition}, pages 18423--18433, 2023.

\bibitem{khachatryan2023text2video}
Levon Khachatryan, Andranik Movsisyan, Vahram Tadevosyan, Roberto Henschel, Zhangyang Wang, Shant Navasardyan, and Humphrey Shi.
\newblock Text2video-zero: Text-to-image diffusion models are zero-shot video generators.
\newblock In {\em Proceedings of the IEEE/CVF International Conference on Computer Vision}, pages 15954--15964, 2023.

\bibitem{kingma2021variational}
Diederik Kingma, Tim Salimans, Ben Poole, and Jonathan Ho.
\newblock Variational diffusion models.
\newblock {\em Advances in neural information processing systems}, 34:21696--21707, 2021.

\bibitem{pickscore}
Yuval Kirstain, Adam Polyak, Uriel Singer, Shahbuland Matiana, Joe Penna, and Omer Levy.
\newblock Pick-a-pic: An open dataset of user preferences for text-to-image generation.
\newblock {\em Advances in Neural Information Processing Systems}, 36:36652--36663, 2023.

\bibitem{nemo}
Oleksii Kuchaiev, Jason Li, Huyen Nguyen, Oleksii Hrinchuk, Ryan Leary, Boris Ginsburg, Samuel Kriman, Stanislav Beliaev, Vitaly Lavrukhin, Jack Cook, et~al.
\newblock Nemo: a toolkit for building ai applications using neural modules.
\newblock {\em arXiv preprint arXiv:1909.09577}, 2019.

\bibitem{improvedprc}
Tuomas Kynk{\"a}{\"a}nniemi, Tero Karras, Samuli Laine, Jaakko Lehtinen, and Timo Aila.
\newblock Improved precision and recall metric for assessing generative models.
\newblock {\em Advances in neural information processing systems}, 32, 2019.

\bibitem{liu2023audioldm}
Haohe Liu, Zehua Chen, Yi Yuan, Xinhao Mei, Xubo Liu, Danilo Mandic, Wenwu Wang, and Mark~D Plumbley.
\newblock Audioldm: Text-to-audio generation with latent diffusion models.
\newblock {\em arXiv preprint arXiv:2301.12503}, 2023.

\bibitem{liu2022compositional}
Nan Liu, Shuang Li, Yilun Du, Antonio Torralba, and Joshua~B Tenenbaum.
\newblock Compositional visual generation with composable diffusion models.
\newblock In {\em European Conference on Computer Vision}, pages 423--439. Springer, 2022.

\bibitem{liu2022character}
Rosanne Liu, Dan Garrette, Chitwan Saharia, William Chan, Adam Roberts, Sharan Narang, Irina Blok, RJ Mical, Mohammad Norouzi, and Noah Constant.
\newblock Character-aware models improve visual text rendering.
\newblock {\em arXiv preprint arXiv:2212.10562}, 2022.

\bibitem{luo2024diff}
Simian Luo, Chuanhao Yan, Chenxu Hu, and Hang Zhao.
\newblock Diff-foley: Synchronized video-to-audio synthesis with latent diffusion models.
\newblock {\em Advances in Neural Information Processing Systems}, 36, 2024.

\bibitem{sdedit}
Chenlin Meng, Yutong He, Yang Song, Jiaming Song, Jiajun Wu, Jun-Yan Zhu, and Stefano Ermon.
\newblock Sdedit: Guided image synthesis and editing with stochastic differential equations.
\newblock {\em arXiv preprint arXiv:2108.01073}, 2021.

\bibitem{neal2001annealed}
Radford~M Neal.
\newblock Annealed importance sampling.
\newblock {\em Statistics and computing}, 11:125--139, 2001.

\bibitem{ntavelis2023autodecoding}
Evangelos Ntavelis, Aliaksandr Siarohin, Kyle Olszewski, Chaoyang Wang, Luc~V Gool, and Sergey Tulyakov.
\newblock Autodecoding latent 3d diffusion models.
\newblock {\em Advances in Neural Information Processing Systems}, 36:67021--67047, 2023.

\bibitem{ouyang2022training}
Long Ouyang, Jeffrey Wu, Xu Jiang, Diogo Almeida, Carroll Wainwright, Pamela Mishkin, Chong Zhang, Sandhini Agarwal, Katarina Slama, Alex Ray, et~al.
\newblock Training language models to follow instructions with human feedback.
\newblock {\em Advances in neural information processing systems}, 35:27730--27744, 2022.

\bibitem{ozbey2023unsupervised}
Muzaffer {\"O}zbey, Onat Dalmaz, Salman~UH Dar, Hasan~A Bedel, {\c{S}}aban {\"O}zturk, Alper G{\"u}ng{\"o}r, and Tolga {\c{C}}ukur.
\newblock Unsupervised medical image translation with adversarial diffusion models.
\newblock {\em IEEE Transactions on Medical Imaging}, 2023.

\bibitem{pinaya2022brain}
Walter~HL Pinaya, Petru-Daniel Tudosiu, Jessica Dafflon, Pedro~F Da~Costa, Virginia Fernandez, Parashkev Nachev, Sebastien Ourselin, and M~Jorge Cardoso.
\newblock Brain imaging generation with latent diffusion models.
\newblock In {\em MICCAI Workshop on Deep Generative Models}, pages 117--126. Springer, 2022.

\bibitem{podell2023sdxl}
Dustin Podell, Zion English, Kyle Lacey, Andreas Blattmann, Tim Dockhorn, Jonas M{\"u}ller, Joe Penna, and Robin Rombach.
\newblock Sdxl: Improving latent diffusion models for high-resolution image synthesis.
\newblock {\em arXiv preprint arXiv:2307.01952}, 2023.

\bibitem{poole2022dreamfusion}
Ben Poole, Ajay Jain, Jonathan~T Barron, and Ben Mildenhall.
\newblock Dreamfusion: Text-to-3d using 2d diffusion.
\newblock {\em arXiv preprint arXiv:2209.14988}, 2022.

\bibitem{prabhudesai2023aligning}
Mihir Prabhudesai, Anirudh Goyal, Deepak Pathak, and Katerina Fragkiadaki.
\newblock Aligning text-to-image diffusion models with reward backpropagation.
\newblock {\em arXiv preprint arXiv:2310.03739}, 2023.

\bibitem{clip}
Alec Radford, Jong~Wook Kim, Chris Hallacy, Aditya Ramesh, Gabriel Goh, Sandhini Agarwal, Girish Sastry, Amanda Askell, Pamela Mishkin, Jack Clark, et~al.
\newblock Learning transferable visual models from natural language supervision.
\newblock In {\em International conference on machine learning}, pages 8748--8763. PMLR, 2021.

\bibitem{dpo}
Rafael Rafailov, Archit Sharma, Eric Mitchell, Christopher~D Manning, Stefano Ermon, and Chelsea Finn.
\newblock Direct preference optimization: Your language model is secretly a reward model.
\newblock {\em Advances in Neural Information Processing Systems}, 36, 2024.

\bibitem{ramesh2021zero}
Aditya Ramesh, Mikhail Pavlov, Gabriel Goh, Scott Gray, Chelsea Voss, Alec Radford, Mark Chen, and Ilya Sutskever.
\newblock Zero-shot text-to-image generation.
\newblock In {\em International conference on machine learning}, pages 8821--8831. Pmlr, 2021.

\bibitem{sd}
Robin Rombach, Andreas Blattmann, Dominik Lorenz, Patrick Esser, and Bj{\"o}rn Ommer.
\newblock High-resolution image synthesis with latent diffusion models.
\newblock In {\em Proceedings of the IEEE/CVF conference on computer vision and pattern recognition}, pages 10684--10695, 2022.

\bibitem{ruiz2023dreambooth}
Nataniel Ruiz, Yuanzhen Li, Varun Jampani, Yael Pritch, Michael Rubinstein, and Kfir Aberman.
\newblock Dreambooth: Fine tuning text-to-image diffusion models for subject-driven generation.
\newblock In {\em Proceedings of the IEEE/CVF conference on computer vision and pattern recognition}, pages 22500--22510, 2023.

\bibitem{saharia2022palette}
Chitwan Saharia, William Chan, Huiwen Chang, Chris Lee, Jonathan Ho, Tim Salimans, David Fleet, and Mohammad Norouzi.
\newblock Palette: Image-to-image diffusion models.
\newblock In {\em ACM SIGGRAPH 2022 conference proceedings}, pages 1--10, 2022.

\bibitem{saharia2022image}
Chitwan Saharia, Jonathan Ho, William Chan, Tim Salimans, David~J Fleet, and Mohammad Norouzi.
\newblock Image super-resolution via iterative refinement.
\newblock {\em IEEE transactions on pattern analysis and machine intelligence}, 45(4):4713--4726, 2022.

\bibitem{precisionrecall}
Mehdi~SM Sajjadi, Olivier Bachem, Mario Lucic, Olivier Bousquet, and Sylvain Gelly.
\newblock Assessing generative models via precision and recall.
\newblock {\em Advances in neural information processing systems}, 31, 2018.

\bibitem{sanghi2022clip}
Aditya Sanghi, Hang Chu, Joseph~G Lambourne, Ye Wang, Chin-Yi Cheng, Marco Fumero, and Kamal~Rahimi Malekshan.
\newblock Clip-forge: Towards zero-shot text-to-shape generation.
\newblock In {\em Proceedings of the IEEE/CVF Conference on Computer Vision and Pattern Recognition}, pages 18603--18613, 2022.

\bibitem{laion}
Christoph Schuhmann, Romain Beaumont, Richard Vencu, Cade Gordon, Ross Wightman, Mehdi Cherti, Theo Coombes, Aarush Katta, Clayton Mullis, Mitchell Wortsman, et~al.
\newblock Laion-5b: An open large-scale dataset for training next generation image-text models.
\newblock {\em Advances in Neural Information Processing Systems}, 35:25278--25294, 2022.

\bibitem{skalse2022defining}
Joar Skalse, Nikolaus Howe, Dmitrii Krasheninnikov, and David Krueger.
\newblock Defining and characterizing reward gaming.
\newblock {\em Advances in Neural Information Processing Systems}, 35:9460--9471, 2022.

\bibitem{sohldiffusion}
Jascha Sohl-Dickstein, Eric Weiss, Niru Maheswaranathan, and Surya Ganguli.
\newblock Deep unsupervised learning using nonequilibrium thermodynamics.
\newblock In {\em International conference on machine learning}, pages 2256--2265. PMLR, 2015.

\bibitem{song2019generative}
Yang Song and Stefano Ermon.
\newblock Generative modeling by estimating gradients of the data distribution.
\newblock {\em Advances in neural information processing systems}, 32, 2019.

\bibitem{song2020improved}
Yang Song and Stefano Ermon.
\newblock Improved techniques for training score-based generative models.
\newblock {\em Advances in neural information processing systems}, 33:12438--12448, 2020.

\bibitem{song2020sliced}
Yang Song, Sahaj Garg, Jiaxin Shi, and Stefano Ermon.
\newblock Sliced score matching: A scalable approach to density and score estimation.
\newblock In {\em Uncertainty in Artificial Intelligence}, pages 574--584. PMLR, 2020.

\bibitem{song2020score}
Yang Song, Jascha Sohl-Dickstein, Diederik~P Kingma, Abhishek Kumar, Stefano Ermon, and Ben Poole.
\newblock Score-based generative modeling through stochastic differential equations.
\newblock {\em arXiv preprint arXiv:2011.13456}, 2020.

\bibitem{szegedy2015going}
Christian Szegedy, Wei Liu, Yangqing Jia, Pierre Sermanet, Scott Reed, Dragomir Anguelov, Dumitru Erhan, Vincent Vanhoucke, and Andrew Rabinovich.
\newblock Going deeper with convolutions.
\newblock In {\em Proceedings of the IEEE conference on computer vision and pattern recognition}, pages 1--9, 2015.

\bibitem{blog}
Ali Taghibakhshi, Sahil Jain, Gerald Shen, Nima Tajbakhsh, and Arash Vahdat.
\newblock Enhance text-to-image fine-tuning with draft+, now part of nvidia nemo, 2024.

\bibitem{vincent2011connection}
Pascal Vincent.
\newblock A connection between score matching and denoising autoencoders.
\newblock {\em Neural computation}, 23(7):1661--1674, 2011.

\bibitem{diffusiondpo}
Bram Wallace, Meihua Dang, Rafael Rafailov, Linqi Zhou, Aaron Lou, Senthil Purushwalkam, Stefano Ermon, Caiming Xiong, Shafiq Joty, and Nikhil Naik.
\newblock Diffusion model alignment using direct preference optimization.
\newblock In {\em Proceedings of the IEEE/CVF Conference on Computer Vision and Pattern Recognition}, pages 8228--8238, 2024.

\bibitem{doodl}
Bram Wallace, Akash Gokul, Stefano Ermon, and Nikhil Naik.
\newblock End-to-end diffusion latent optimization improves classifier guidance.
\newblock In {\em Proceedings of the IEEE/CVF International Conference on Computer Vision}, pages 7280--7290, 2023.

\bibitem{wang2023audit}
Yuancheng Wang, Zeqian Ju, Xu Tan, Lei He, Zhizheng Wu, Jiang Bian, et~al.
\newblock Audit: Audio editing by following instructions with latent diffusion models.
\newblock {\em Advances in Neural Information Processing Systems}, 36:71340--71357, 2023.

\bibitem{wu2024medsegdiff}
Junde Wu, Rao Fu, Huihui Fang, Yu Zhang, Yehui Yang, Haoyi Xiong, Huiying Liu, and Yanwu Xu.
\newblock Medsegdiff: Medical image segmentation with diffusion probabilistic model.
\newblock In {\em Medical Imaging with Deep Learning}, pages 1623--1639. PMLR, 2024.

\bibitem{hps}
Xiaoshi Wu, Yiming Hao, Keqiang Sun, Yixiong Chen, Feng Zhu, Rui Zhao, and Hongsheng Li.
\newblock Human preference score v2: A solid benchmark for evaluating human preferences of text-to-image synthesis.
\newblock {\em arXiv preprint arXiv:2306.09341}, 2023.

\bibitem{imagereward}
Jiazheng Xu, Xiao Liu, Yuchen Wu, Yuxuan Tong, Qinkai Li, Ming Ding, Jie Tang, and Yuxiao Dong.
\newblock Imagereward: Learning and evaluating human preferences for text-to-image generation.
\newblock {\em Advances in Neural Information Processing Systems}, 36, 2024.

\bibitem{xu2024amodal}
Katherine Xu, Lingzhi Zhang, and Jianbo Shi.
\newblock Amodal completion via progressive mixed context diffusion.
\newblock In {\em Proceedings of the IEEE/CVF Conference on Computer Vision and Pattern Recognition}, pages 9099--9109, 2024.

\bibitem{yang2024using}
Kai Yang, Jian Tao, Jiafei Lyu, Chunjiang Ge, Jiaxin Chen, Weihan Shen, Xiaolong Zhu, and Xiu Li.
\newblock Using human feedback to fine-tune diffusion models without any reward model.
\newblock In {\em Proceedings of the IEEE/CVF Conference on Computer Vision and Pattern Recognition}, pages 8941--8951, 2024.

\bibitem{yoon2023sadm}
Jee~Seok Yoon, Chenghao Zhang, Heung-Il Suk, Jia Guo, and Xiaoxiao Li.
\newblock Sadm: Sequence-aware diffusion model for longitudinal medical image generation.
\newblock In {\em International Conference on Information Processing in Medical Imaging}, pages 388--400. Springer, 2023.

\end{thebibliography}
}

\clearpage
\appendix
\onecolumn
\section{Appendix}
\label{sec:appendix}

\subsection{Proof of inevitability of reward hacking}
\label{sec:rewardproof}
Consider an arbitrary reward function $r(x)$ that is sufficiently smooth. Consider a non-parametric probability distribution $p(x)$ which maximizes the expected reward:
\begin{equation}
    p^* = \arg\max \mathbb{E}_{x \sim p(x)}\left[r(x)\right]
    \label{eq:rewardmaximization}
\end{equation}
with the constraint $\int_x p(x) \dd x = 1$. This is written as a maximization problem with Langrange multiplier $\beta$
\begin{align}
    &\max_p \int_x p(x) r(x) \dd x - \beta \left(\left(\int_x p(x) \dd x\right) - 1\right) \\
    &= \int_x p(x)\left[r(x) - \beta\right] \dd x + \beta \\
    &= \int_x \mathcal{L}(x, p, \dot{p}) \dd x + \beta
\end{align}
This is an Euler-Langrage equation where $\mathcal{L}(x, p, \dot{p}) = p(x)(r(x) - \beta)$. The maximizer of this equation is given by:
\begin{equation}
    \frac{\partial L}{\partial p} - \frac{\dd}{\dd t}\left[\frac{\partial L}{\partial \dot{p}}\right] = 0
    \label{eq:el}
\end{equation}
Since the second term in~\cref{eq:el} is zero, we get
\begin{equation}
    \frac{\partial L}{\partial p} = r(x) - \beta = 0
    \label{eq:grad0}
\end{equation}
Note that $\beta$ should be a constant, for a general $r(x)$ \cref{eq:grad0} will not hold true. However, note that $p(x) \ge 0 \forall x$. Therefore, if $r(x) < \beta$, then $p(x) = 0$ and if $r(x) > \beta$, then $p(x)$ will grow indefinitely, violating the pdf constraint $\int_x p(x) \dd x = 1$.
However, $\beta$ should be chosen such that $r(x) \le \beta  \forall x$. However, if $r(x) < \beta \forall x$, then $p(x) = \forall x$.
Therefore, $\beta$ is chosen to be $\beta = \sup r(x)$, leading to the optimal distribution 
\begin{equation}
    p^*(x) = \delta(x - x^*)
\end{equation}
where $x^* = \arg\max r(x)$. 
If $r$ has multiple maxima with the same maximum value, say $\{x_1^*, x_2^* \ldots x_n^*\}, r(x_i^*) = \sup r(x)$, then there exists a family of optimal distributions:
\begin{equation}
    p^*(x) = \sum_i w_i \delta(x - x_i^*) 
\end{equation}
such that $\sum_i w_i = 1$. We assume that $r$ is not `flat' at this maximum value, therefore $p^*(x)$ lacks diversity. \\

For a conditional distribution $p(x|c)$ maximizing the reward $r(x, c)$, a similar derivation yields $p(x|c) = \delta(x - x^{*rc})$, where $x^{*rc} = \arg\max r(x, c)$.
This completes the proof in the non-parameteric case.
Note that no assumption is made about the finetuning algorithm (DPO, DRaFT, ReFL, etc.) or nature of the reward function (CLIP, JPEG compression, Aesthetics, etc.).
This proves that reward hacking is an artifact of the expected reward maximization problem formulation itself. \\

In the parameteric case, the optimal $p^*(x)$ may not be achievable due to the parameterization.
However, even with low-dimensional parameter updates like LoRA, we notice a substantial loss of image diversity when training {\draft}.
Qualitative comparisons between the base, {\draft} and our regularization are shown in ~\cref{fig:qualdiv1,fig:qualdiv2,fig:qualdiv3,fig:qualdiv4}.

\subsubsection{Adding dropout to reward functions does not work}
Moreover, this explains why even reward functions in ~\cite{draft} with aggressive dropout rates ($> 0.95$) still led to reward collapse.
Let the reward model be parameterized by $\varphi$ and let $N = |\varphi|$ be the dimension of the reward model parameters.
Under the dropout case with dropout parameter $\xi$, the expected reward maximization formulation becomes:
\begin{equation}
    p^* = \arg\max \mathbb{E}_{x \sim p(x), u \in \mathcal{U}[0, 1]^{N}}\left[r_{\varphi[u,\xi]}(x)\right]
    \label{eq:dropoutreward}
\end{equation}
where $u$ is sampled from an i.i.d. multidimensional uniform distribution over 0 to 1, i.e. $u \sim \mathcal{U}[0, 1]^N$, and $\varphi[u,\xi]$ are the parameters after applying dropout with random variable $u$ and dropout threshold $\xi$.
Since $x$ and $u$ are independent, we can simplify the expression \cref{eq:dropoutreward} by expanding the expectation over $u$ to obtain:
\begin{equation}
    p^* = \arg\max \mathbb{E}_{x \sim p(x)}\left[\mathbb{E}_u \left[r_{\varphi[u,\xi]}(x)\right]\right] = \arg\max \mathbb{E}_{x \sim p(x)} \left[ \tilde{r}_\xi(x) \right]
\end{equation}
where 
\begin{equation}
    \tilde{r}_\xi(x) = \mathbb{E}_{u \in \mathcal{U}[0, 1]^N}\left[r_{\varphi[u,\xi]}(x)\right] 
\end{equation}
is independent of random variable $u$.
This is the same optimization problem as \cref{eq:rewardmaximization} with a new reward function $\tilde{r}_\xi$, therefore having the same reward hacking problem.

\subsection{More reward-diversity tradeoff analysis}
\label{sec:app-reward-diversity}
We perform more ablations on the analysis shown in \cref{sec:rewarddiversity}.
Specifically, we train SDv1.4/SDXL on Pickscore/HPSv2 reward models and compare the reward-diversity tradeoffs for various regularizations (i.e. KL, LoRA scaling, {\methodabbv}).
Then, we generate images from the PartiPrompt prompt dataset containing over 1600 prompts, and all four subsets of the HPSv2 prompt dataset (containing 800 prompts each from 4 categories).
Additionally, images are generated from the coverage prompts as mentioned in ~\cref{sec:eval}.
Next, the training reward score is computed on the generated images on these prompt datasets, and plotted against the diversity score from the coverage dataset.
These quantitative reward-diversity tradeoffs are shown in \cref{fig:sdxl_pickscore_all,fig:sd_pickscore_all,fig:sdxl_hps_all,fig:sd_hps_all}.
{\methodabbv} consistently attains better reward-diversity tradeoff than LoRA scaling and KL divergence on both SDv1.4 and SDXL architectures trained on the Pickscore dataset.
On the HPSv2 trained models, we observe a smaller Pareto gap between the baseline and our method.
Moreover, there is a trend reversal among the baselines. In the HPSv2 trained models, KL regularization seems to outperform LoRA scaling, but in the Pickscore trained models, LoRA scaling tends to outperform KL regularization. 
This highlights the versatility and reliability of {\methodabbv} as an effective regularization, being less volatile to trend reversals. \\

\textbf{CLIP-based image-text alignment} We also compare CLIP~\cite{clip} scores on the HPSv2 prompt data splits for all four configurations.
These results are shown in \cref{fig:hps_anime_clip,fig:hps_photo_clip,fig:hps_painting_clip,fig:hps_conceptart_clip}.
Note that unlike the reward-diversity analysis, maximizing Pickscore/HPSv2 rewards does not imply higher text-to-image alignment. 
This is evident from the {\draft} model consistently underperforming the base model in terms of CLIP score.
Consequently, the theoretical optimal is sometimes very close to the base models and there is no CLIP-diversity tradeoff anymore.
Therefore, we compare models only on the CLIP score for a particular value of diversity metric (FID, Recall, Spectral Distance).
However, we notice that different regularizations do not deviate in CLIP score for different values of regularization.
We make two interesting observations comparing {\methodabbv} and LoRA scaling.
First, {\methodabbv} outperforms LoRA scaling overall (on PartiPrompts and 3/4 HPSv2 subsets) on all four configurations (SDv1.4/SDXL trained on Pickscore/HPSv2), achieving a notably higher CLIP score.
The difference is more notable for the SDv1.4 variants, owing to the already high baseline image-text alignment of SDXL compared to SDv1.4.
Second, we qualitatively observe that reward model training produces images that are stylistically more cartoony.
Consequently, this leads to lower text-image alignment for HPSv2 photo prompts where the prompts mention the words "photo" but the generated images look more cartoony and dreamish.  
{\methodabbv} preserves more characteristics of the {\draft} alignment, that leads to a slightly lower CLIP score than LoRA scaling on this particular subset. 

\subsection{Qualitative Results}
\label{sec:app-qualitative}

We present two qualitative comparisons. 

\textbf{Comparison with \draft.}
In this section, we show a few uncurated subsets of images generated by {\draft} and {\methodabbv}, similar to those shown in user studies in ~\cref{fig:qualdiv1,fig:qualdiv2,fig:qualdiv3,fig:qualdiv4}.
Note that across network architectures and reward functions, {\methodabbv} consistently demonstrates high quality images compared to the base model, while much higher diversity than the {\draft} model.

\textbf{Comparison with other baselines.}
We consider three other baselines in this work:

\textbf{DOODL}~\cite{doodl}: This method aims to use Exact Diffusion Inversion to optimize the noise latent that produces an image to maximize classifier guidance, by directly backpropagating through the pre-trained classier's score on the generated image.
The classifier can be interpreted as a reward model, and the optimization is done at inference-time. However, this method takes very long to produce results.
For example, with its default configuration (50 optimization steps for 50 DDIM steps), producing the images from the PartiPrompt prompt set will take $\sim$775 GPU hours, as opposed to $<30$ minutes for our method.

\textbf{ReNO}~\cite{reno}: This method is functionally similar to DOODL, except it works on one-step diffusion models. Given a one-step diffusion model (distilled from a DDPM/DDIM model) $G_\mbtheta(\epsilon, \mbc)$ that generates an image based on noise $\epsilon$ and (prompt) conditioning $\mbc$, and a reward function $R$, the ReNO objective is defined as
\begin{equation}
    \epsilon^* = \arg\max_\epsilon R(G_\mbtheta(\epsilon, \mbc), \mbc)
    \label{eq:reno}
\end{equation}
Since both the one-step diffusion model $G$ and reward function $R$ are differentiable, \cref{eq:reno} is solved using direct optimization using gradient ascent techniques.
Moreover, to prevent divergence from the initial data distribution, a regularization based on the proximity of the noise $\epsilon$ to the normal distribution $\mathcal{N}(0, 1)$ is measured.
In the paper, a $\chi^d$ regularization on the norm of the noise is used.

\textbf{ReFL}~\cite{imagereward}: Directly optimizing LDMs with a reward models is expensive due to its many sampling steps.
However, ~\cite{imagereward} observe that the the rewards of the images in the middle of the sampling chain are indicative of the final scores.
Therefore, the LDM chooses a model and a randomly chosen timestep in the middle of the sample chain, and computes gradient \textit{only} with respect to that step.
This prevents an expensive gradient computation step for finetuning the LDM.

\cite{draft} already show that {\draft} performs quantitatively better than these baselines; we focus on qualitative differences.
For all methods, we use their default recommended configurations.
We qualitatively evaluate all models by generating images from the PartiPrompt dataset.
Qualitative comparison is shown in \cref{fig:qualall-1,fig:qualall-2}.

\subsection{User Study}
\label{sec:app-user-study}

\begin{figure}[ht!]
    \centering
    \begin{minipage}{\linewidth}
        \centering
        \includegraphics[width=0.32\linewidth]{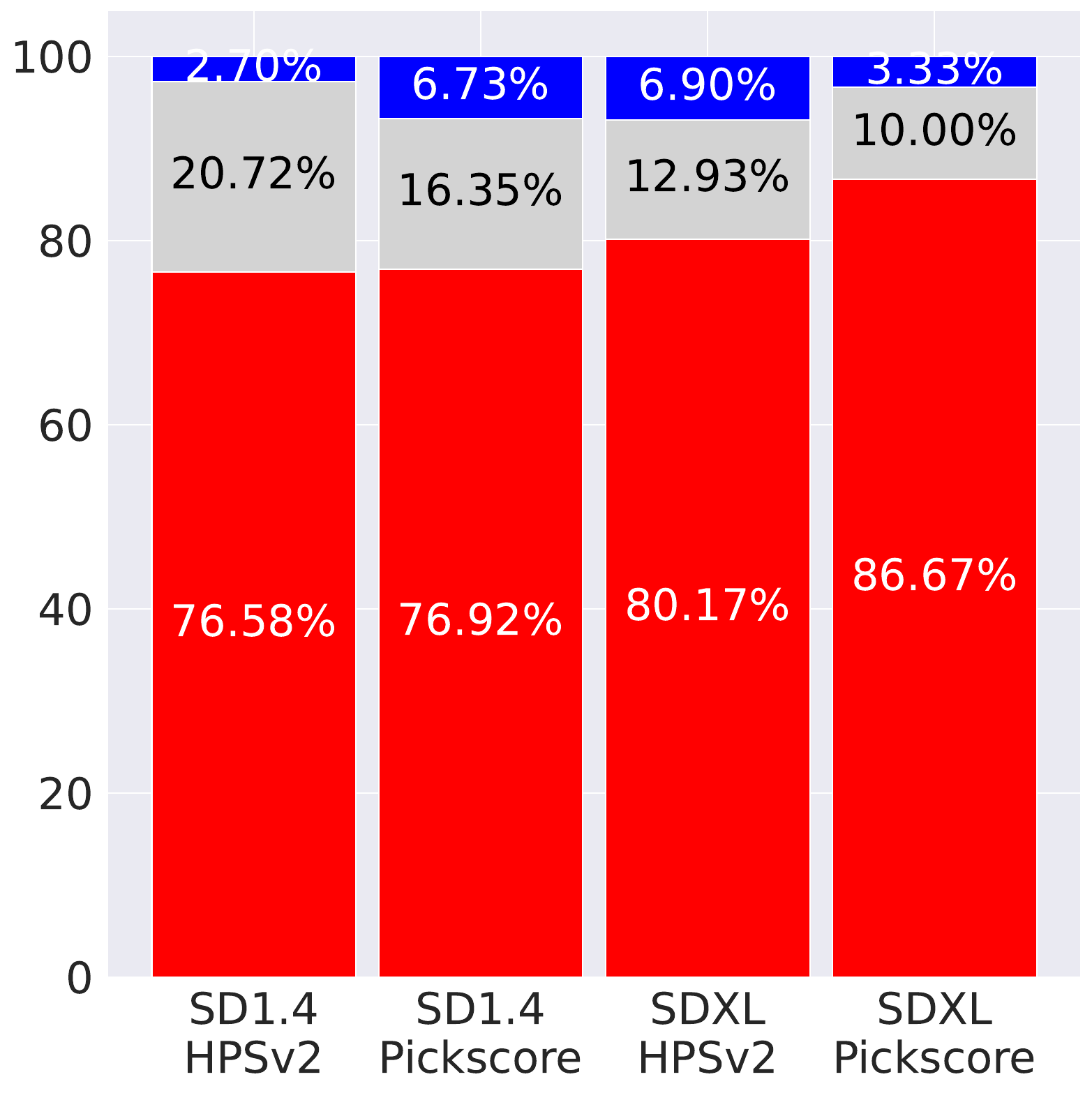}
        \includegraphics[width=0.32\linewidth]{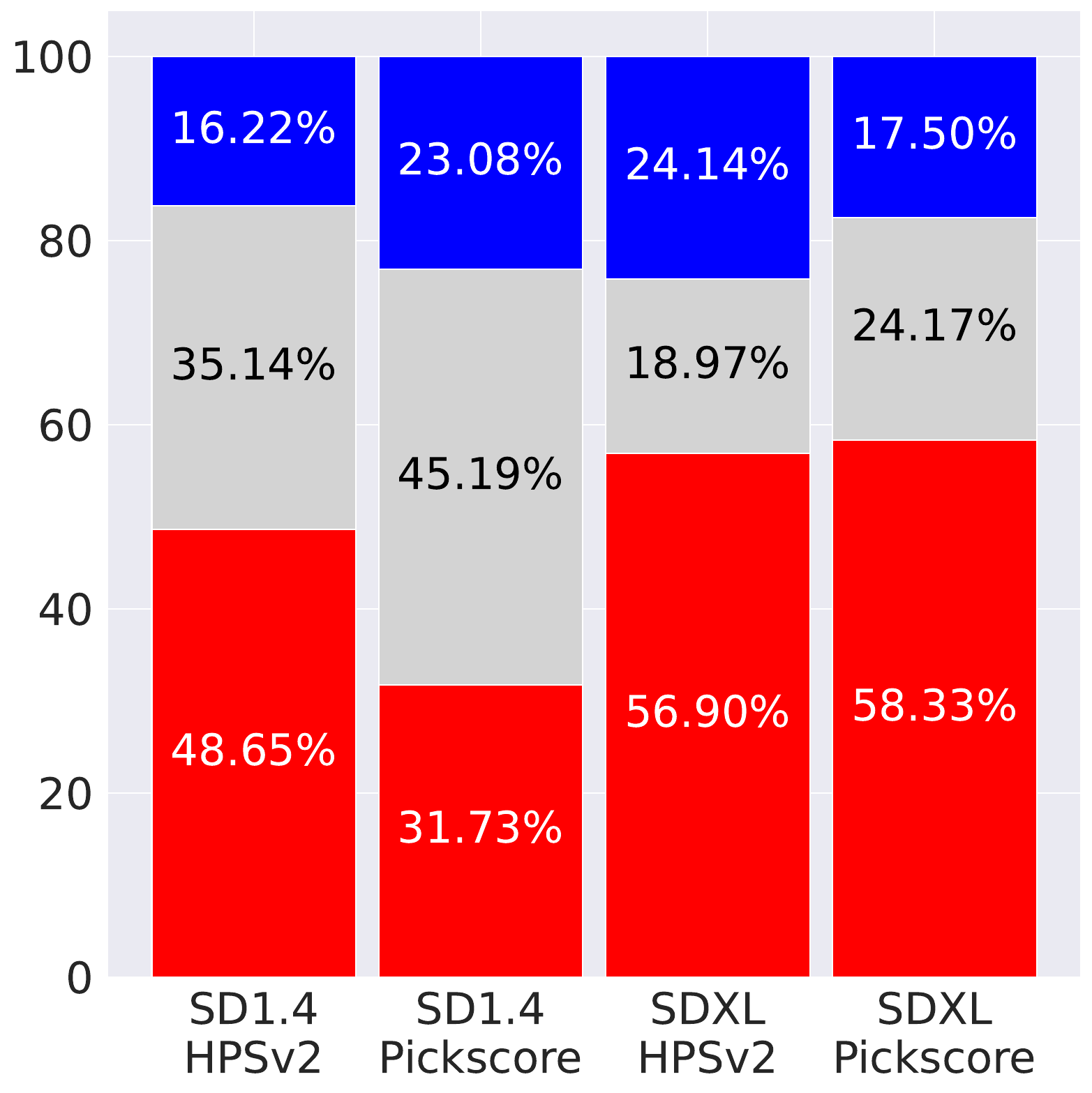}
        \includegraphics[width=0.32\linewidth]{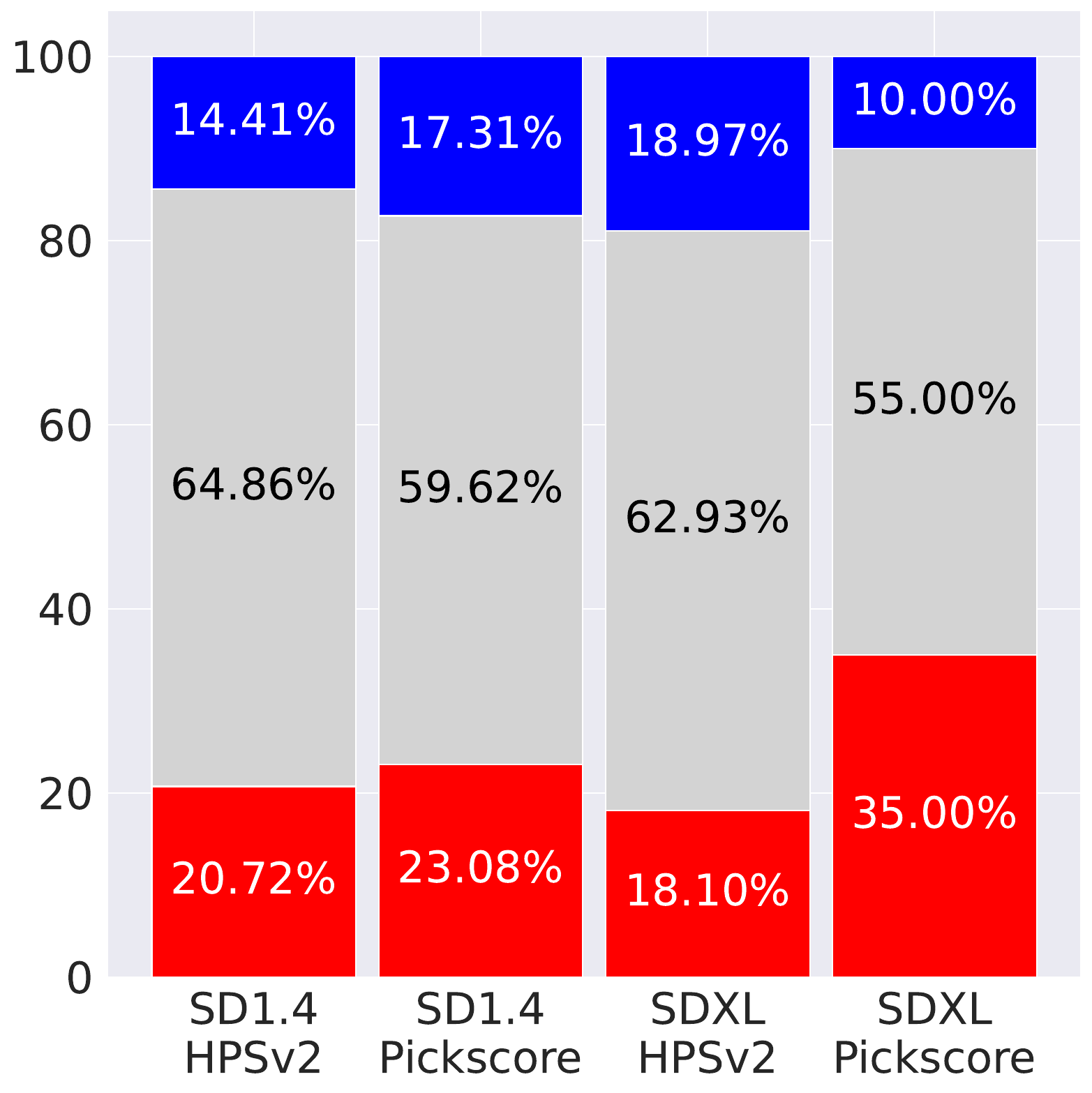}
        \begin{tabularx}{\linewidth}{YYY}
           Diversity & Quality & Alignment \\ 
        \end{tabularx}
    \end{minipage}
    \caption{\textbf{User study results aggregated by DM+reward configuration} 
    User preferences are consistent across models and reward functions,
     confirming that the user preference towards our method is due to the {\methodabbv} itself and is not due to any specific architecture or reward model.}
    \label{fig:userstudy-indconfig}
\end{figure}

In this section, we provide more details on the user preference study.
The objective of the user study is to quantify if the proposed method: {\methodabbv} leads to increased diversity at the cost of any loss of quality or alignment.
To this end, we use the `coverage prompt' dataset, which is a subset of 40 prompts from the PartiPrompt prompt dataset.
For each prompt, 50 images are generated for all methods with the same noise latents for consistency.
These images are generated for all four configurations - SDv1.4 and SDXL models that are trained on Pickscore and HPSv2 rewards.
The web UI assigns a unique user ID to a browser session, and randomly chooses a prompt, DM+reward configuration, and selects 9 random indices from the 50 generated images without replacement, randomly shuffles the order, and displays the images side by side (\cref{fig:userstudyui}).   
We recruit 36 participants and provide them basic app usage instructions prior to conducting the study, and we collect more than 1500 total votes from all users.
Users are referenced by browser cookie information, allowing us to preserve user anonymity while collecting user-specific voting statistics.
Both overall voting results, and results aggregated by configurations are summarized and discussed in ~\cref{sec:userstudy}.

\textbf{User-normalized preference}
However, our user study allows the users to cast a different number of votes as per their convenience. 
This leads to a slight non-uniformity in the distribution of votes (\cref{fig:userstudy-votingstats}), which can skew the preference scores towards users who cast more votes.
To highlight this potential discrepancy between overall vote distribution and user-normalized vote distribution, we compute re-normalized preference scores as follows.
Instead of counting votes towards a particular baseline ({\draft}, {\methodabbv} or Equal) followed by normalization, we normalize the number of votes of each individual user to sum to a 100.
Next, we aggregate these normalized votes for each user.
Essentially, we calculate the average preference of users irrespective of the total number of votes cast by each user.
These results are shown in \cref{fig:userstudynorm}.
Interestingly, the trends do not drastically shift from that in ~\cref{fig:userstudy-overall,fig:userstudy-indconfig}, showing that user agreeability on preferences is high.
If users disagreed on preferences, then the unweighted and reweighted preference statistics may have been different.

\begin{figure}[ht!]
    \centering
    \begin{minipage}{0.6\linewidth}
        \centering
        \includegraphics[width=\linewidth]{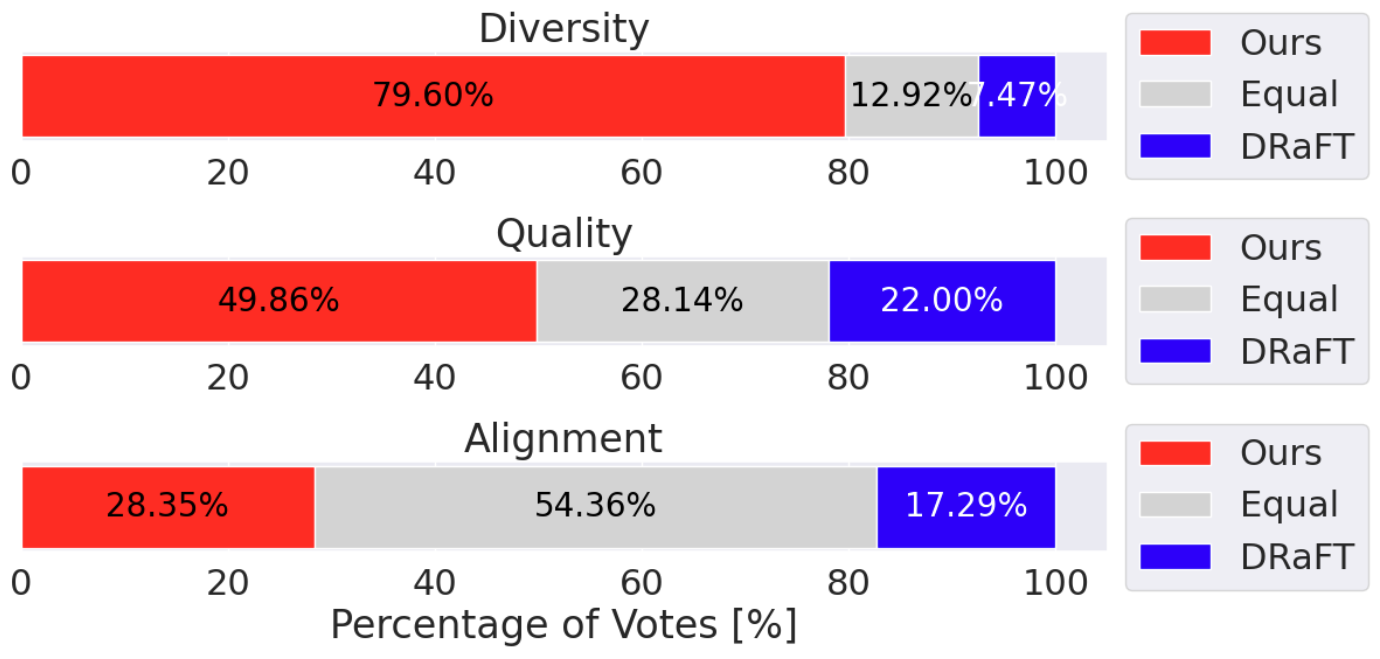}
        \subcaption{\textbf{User study with normalized vote counts for all users.} Our method (\methodabbv) demonstrates exceptional diversity and quality while preserving alignment even when all users are re-weighted to have equal contribution in votes.}
        \label{fig:userstudy-overall-norm}
    \end{minipage}
    \begin{minipage}{0.39\linewidth}
        \centering
        \includegraphics[width=\linewidth]{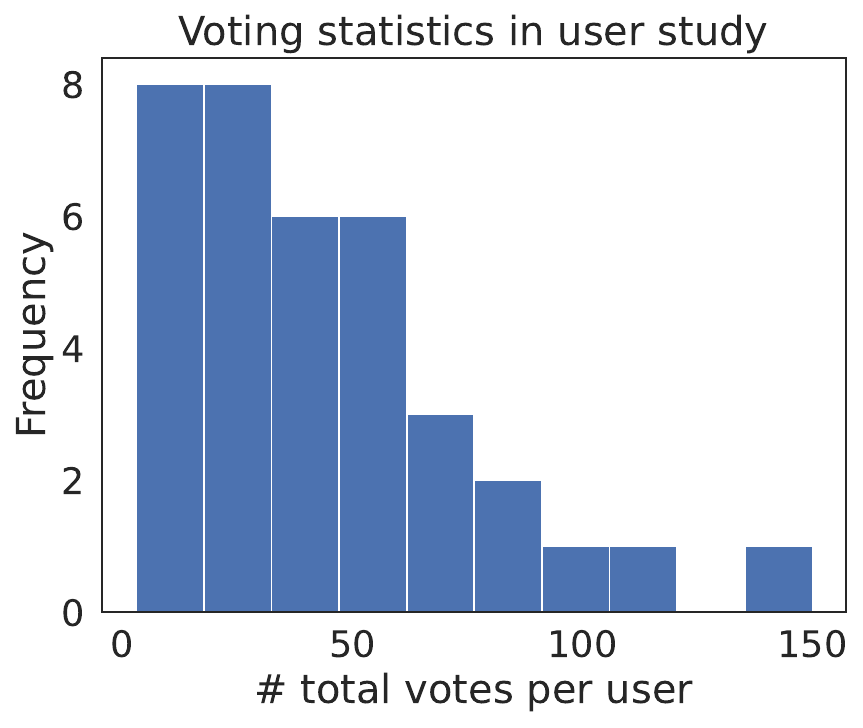}
        \subcaption{\textbf{Voting statistics.} The distribution of votes of all users who participated in the study. Few users voted disproportionately more than others, potentially skewing the user study results (in \cref{fig:userstudy-overall,fig:userstudy-indconfig}).}
        \label{fig:userstudy-votingstats}
    \end{minipage} 
    \begin{minipage}{\linewidth}
        \centering
        \includegraphics[width=0.32\linewidth]{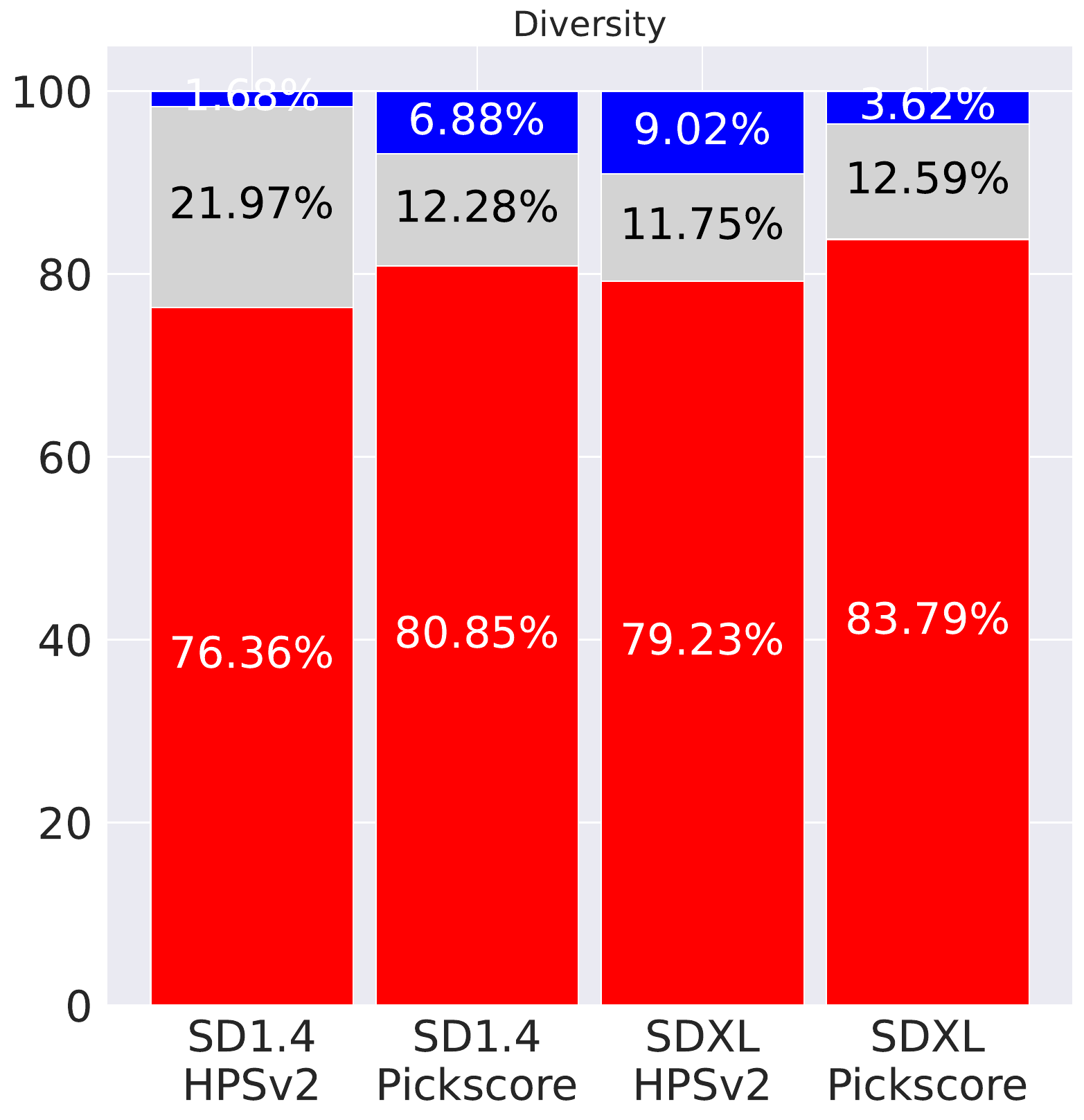}
        \includegraphics[width=0.32\linewidth]{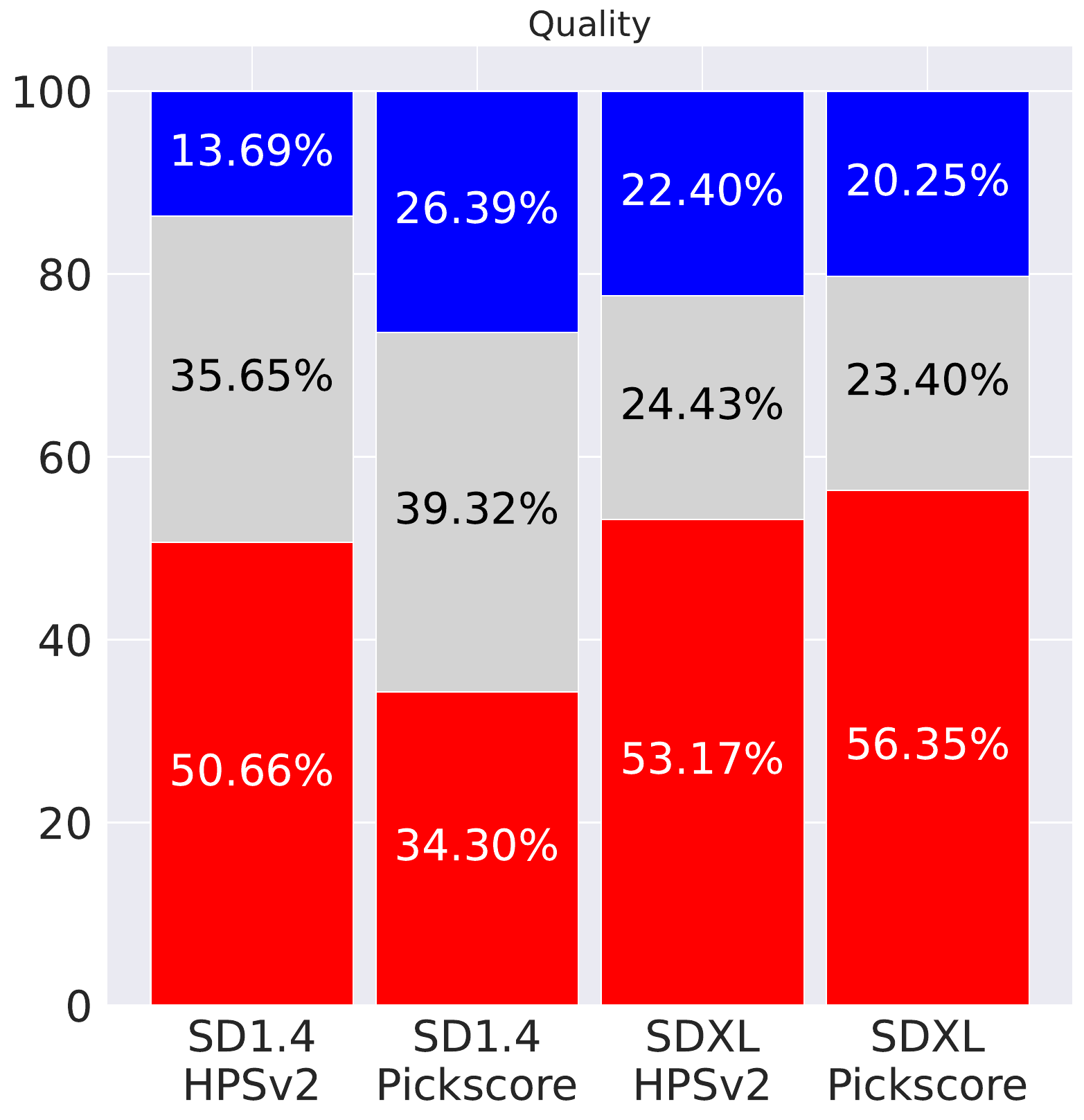}
        \includegraphics[width=0.32\linewidth]{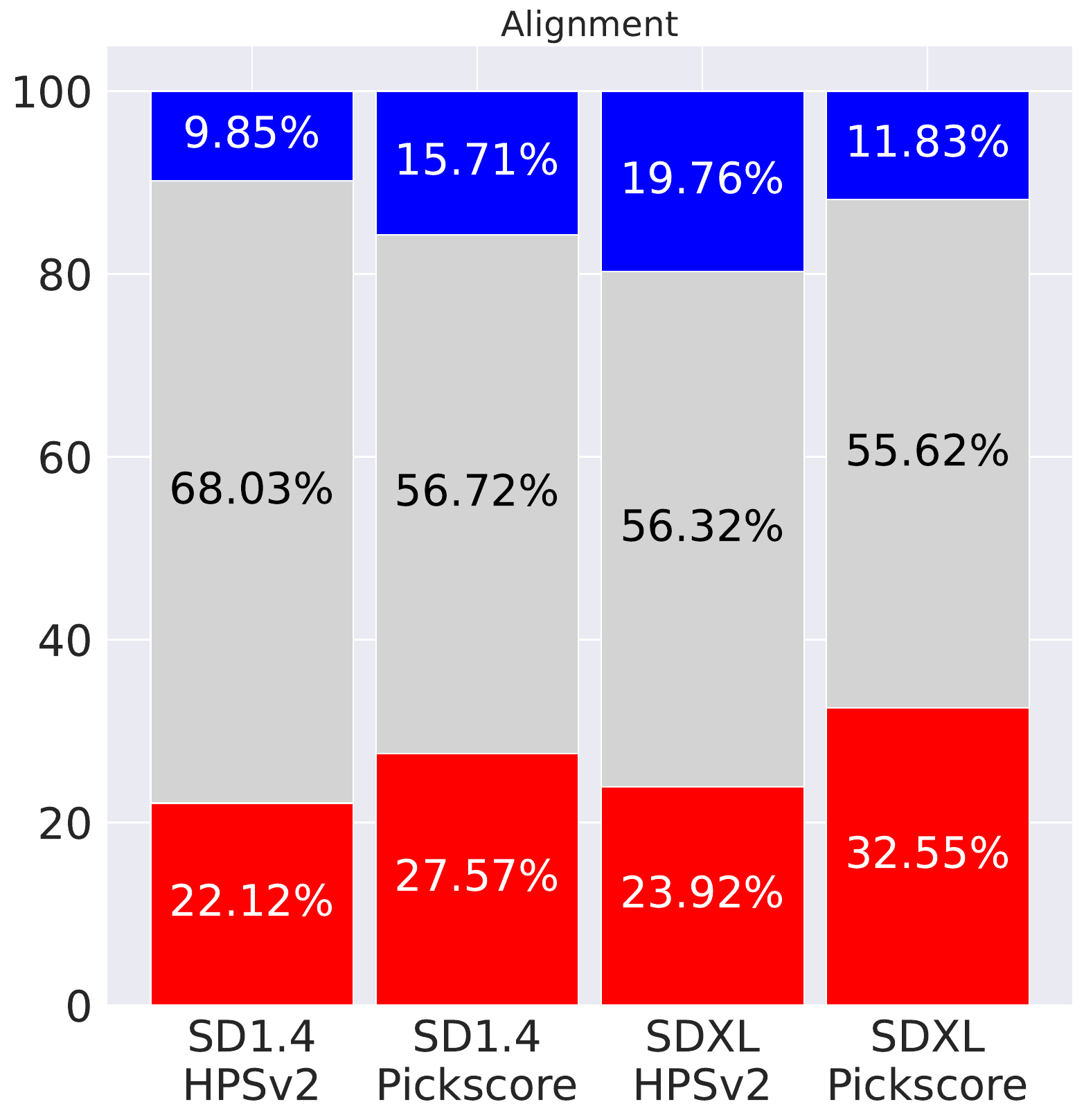}
        \begin{tabularx}{\linewidth}{YYY}
           Diversity & Quality & Alignment \\ 
        \end{tabularx}
        \subcaption{\textbf{User study results aggregated by DM+reward configuration with normalized votes} User preferences are consistent across models and reward functions, indicating the improvement is due to the {\methodabbv} itself.  Images best viewed zoomed in.}
        \label{fig:userstudy-indconfig-norm}
    \end{minipage}
    \caption{\textbf{User Study comparing {\methodabbv} with {\draft} with user-based reweighing of votes}. We recruited {\numusers} participants who compared the quality, diversity and alignment of {\methodabbv} and {\draft}, resulting in more than {\numvotes} votes.
    In contrast to \cref{fig:userstudy-overall}, this study reweighs each user contribution to have the same number of votes, to avoid skewing the user study in favor of users who voted more than others (\cref{fig:userstudy-votingstats}).
    }
    \label{fig:userstudynorm}
\end{figure}

\subsection{Limitations and Future Work}
\label{sec:limitations}
Although our work aims to study regularization techniques that are inference-time, and mitigate the `reference mismatch' problem, it does not completely eliminate it.
For example, in {\methodabbv}, the earlier sampling steps are dominated by the score function of the base model,
 with the underlying assumption that each mode of the original data distribution has a `high-reward region' close to it.
This assumption is usually true for stylistic changes (e.g. Pickscore or HPSv2 rewards that primarily alter the style of the images) but may not necessarily hold for text-to-image alignment (e.g. changing spatial relationships, counts or attributes of objects).
This assumption is also similar to the motivation used in works like SDEdit~\cite{sdedit}, which recover realistic images from a `partially noisy label'.
Consequently, we observe a huge improvement in quality, but do not improve text-to-image alignment significantly.
The lack of improvement of text-to-image alignment also raises questions about the efficacy of reward models trained on human preference data themselves.
Most reward models trained on large human preference datasets achieve maximum validation accuracies of 65-70$\%$~\cite{pickscore,hps,imagereward}, questioning the presence of any discernable, objective learnable signal present in these images w.r.t. alignment.
This forms the basis for future work by using finegrained alignment using Large Multimodal Models (LMMs) to identify and correct mistakes in the image that do not align with the prompt.
Another line of future work pertains to the choice of $\gamma$.
Since our choice of $\gamma$ allows immense flexibility of the interplay of the base and {\draft} sampling dynamics, a user interface can be built where users can choose from a predetermined set of $\gamma$ curves, followed by finetuning these curves using spline interpolation from points clicked on by the user.

\subsection{More intuition for {\method}}
\subsubsection{Asymmetric Mixing Dynamics}
\label{sec:mixingdynamics}
Consider a data distribution composed of only a finite set of images $\mbm_i \in \mathbb{R}^d$, i.e. $p_\td(\mbx) = \sum_{i=1}^{N} w_i \delta(\mbx - \mbm_i)$, $\sum_i w_i = 1$.
The average distance between any two `modes' of the distribution is $m_\td = \mathbb{E}_{i{\ne}j}\left[\| \mbm_i - \mbm_j \|_2\right]$.
Now, consider the forward diffusion $\mbx_t = \aat \mbx + \sigma_t \epsilon$, $\epsilon \in \mathcal{N}(0, 1)$.
The distribution is given by
\begin{align}
    q_t(\mbx_t) &= \int_{\mbx_0} q(\mbx_t | \mbx_0) p_\td(\mbx_0) \dd \mbx_0 \\
    &= \int_{\mbx_0} \frac{1}{\sqrt{2\pi \sigma_t^d}} \exp\left({-\frac{\| \mbx_t - \aat \mbx_0 \|_2^2}{2\sigma_t^2}}\right) \left(\sum_i w_i \delta(\mbx_0 - \mbm_i)\right) \dd \mbx_0 \\
    q_t(\mbx_t) &= \frac{1}{\sqrt{2\pi \sigma_t^d}} \sum_i w_i  \exp\left({-\frac{\| \mbx_t - \aat \mbm_i \|_2^2}{2\sigma_t^2}}\right) 
\end{align}

Therefore, $q_t(\mbx_t)$ is a Gaussian Mixture Model (GMM) with means $\mbm^{(t)}_i = \aat \mbm_i$. 
The average distance between the means is now $m_\td^{(t)} = \mathbb{E}_{i{\ne}j}\left[\| \mbm_i^{(t)} - \mbm_j^{(t)}\|_2\right] = \aat \mathbb{E}_{i{\ne}j}\left[\| \mbm_i - \mbm_j \|_2\right] = \aat m_\td$.
The average distance between any two modes decreases, as they collapse onto each other to form a unimodal Gaussian distribution. 
Two data modes are therefore easy to tell apart from each other for small $t$ instead of larger $t$.
Therefore, the score matching objective must be most discriminative in terms of mode recovery from the later stages, where samples from $q_t(\mbx_t)$ cannot be reliably distinguished from each other in terms of the mode of the data distribution that they originated from.
We demonstrate this using a simple example. \\

\begin{figure}
    \centering
    \includegraphics[width=\linewidth]{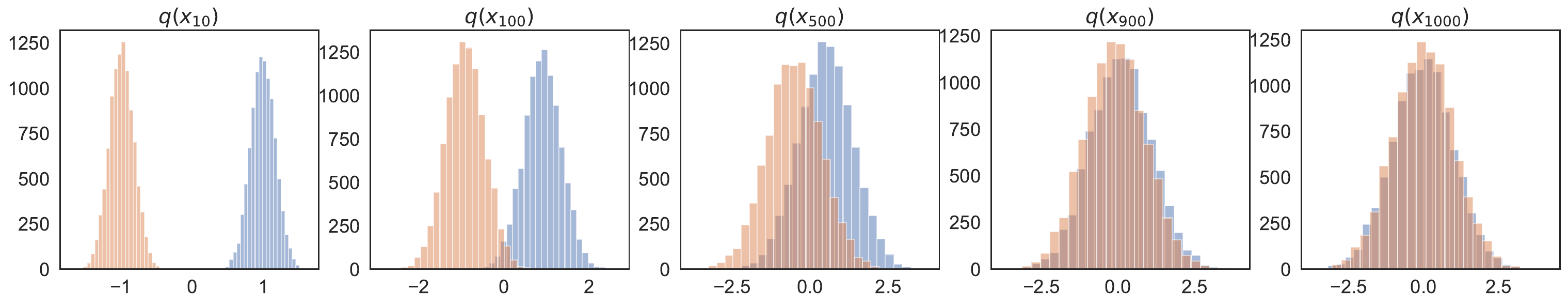}
    \includegraphics[width=\linewidth]{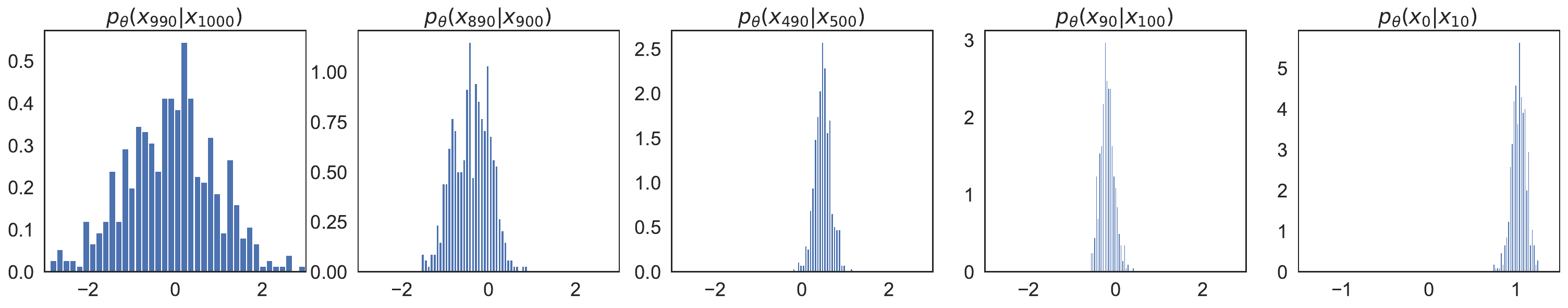}
    \includegraphics[width=\linewidth]{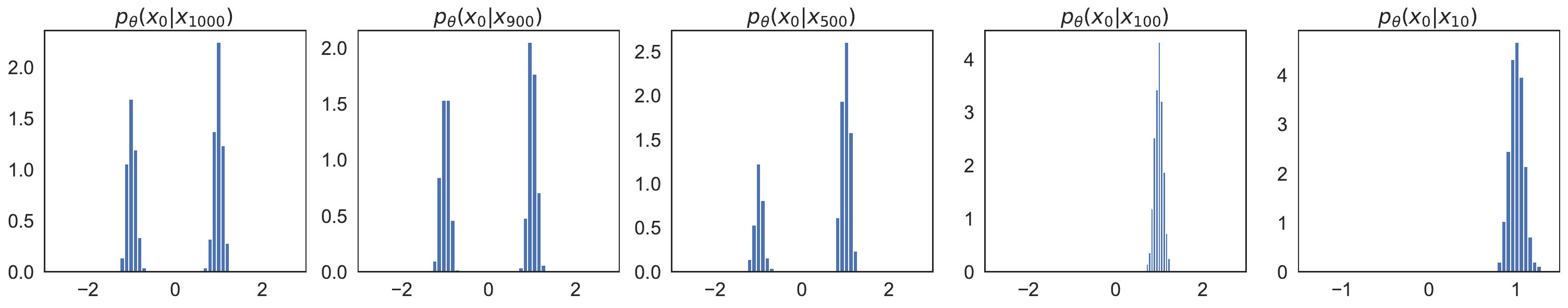}
    \caption{\textbf{Toy example showing the sample distribution during forward and reverse time sampling for a 1D problem}.
    Top row shows that data modes tend to mix more with larger $t$, indicating that mode-recovery behavior must emerge during later timesteps.
    Middle and bottom rows show that if a data is sampled from $q_t$ for higher $t$, then multiple modes of the data are covered from Langevin dynamics, motivating the use of the base model to 
    guide the initial phase of sampling from the reverse-time SDE, and using {\draft} for local finetuning.
    }
    \label{fig:toyreversesampling}
\end{figure}

\textbf{Toy problem illustrating mode-recovering nature of sampling dynamics.}
Consider a toy example of a 1D Gaussian Mixture distribution with two modes as data distribution, i.e. $p_\td(\mbx) = 0.5(\mathcal{N}(1, 0.05) + \mathcal{N}(-1, 0.05))$.
We consider a 1000 forward diffusion steps for this problem. 
Top row in ~\cref{fig:toyreversesampling} shows samples from $q_t(\mbx_t)$ with the colors representing the mode of the original $p_\td$ distribution from which the sample of $q_t$ is generated.
As $t$ increases, the samples mix with each other and become less distinguishable.
To show the mode recovering behavior, we sample one data point from $q_t(\mbx_t)$ and plot the distribution of $p(\mbx_{t-10}|\mbx_t)$ using Langevin dynamics with the ground-truth score function.
The intuition is that if $p(\mbx_{t-10}|\mbx_t)$ is high variance, then the subsequent samples from $p(\mbx_{t-20}|\mbx_{t-10})$ will discover different modes, and eventually the true data distribution.
Middle row in ~\cref{fig:toyreversesampling} shows the samples from $p(\mbx_{t-10}|\mbx_t)$ for different $t$, using the ground-truth score function, and bottom row shows the samples from $p(\mbx_{0}|\mbx_t)$.
The most high variance behavior is shown for larger values of $t$, and only a local finetuning to a particular mode of the data for smaller values of $t$.
This motivates our regularization for {\method}.
During the earlier stages of reverse stage sampling (high values of $t$) when the mode-recovering behavior is the highest, we let the sampling dynamics be governed by the base model.
This helps in early recovery of multiple modes of the data.
During the later stages (smaller $t$), the score function from {\draft} dominates the convergence of these samples to the nearest high-reward samples.

\begin{figure*}[ht!]
    \centering
    \begin{minipage}{\linewidth}
        \fbox{\includegraphics[width=\linewidth]{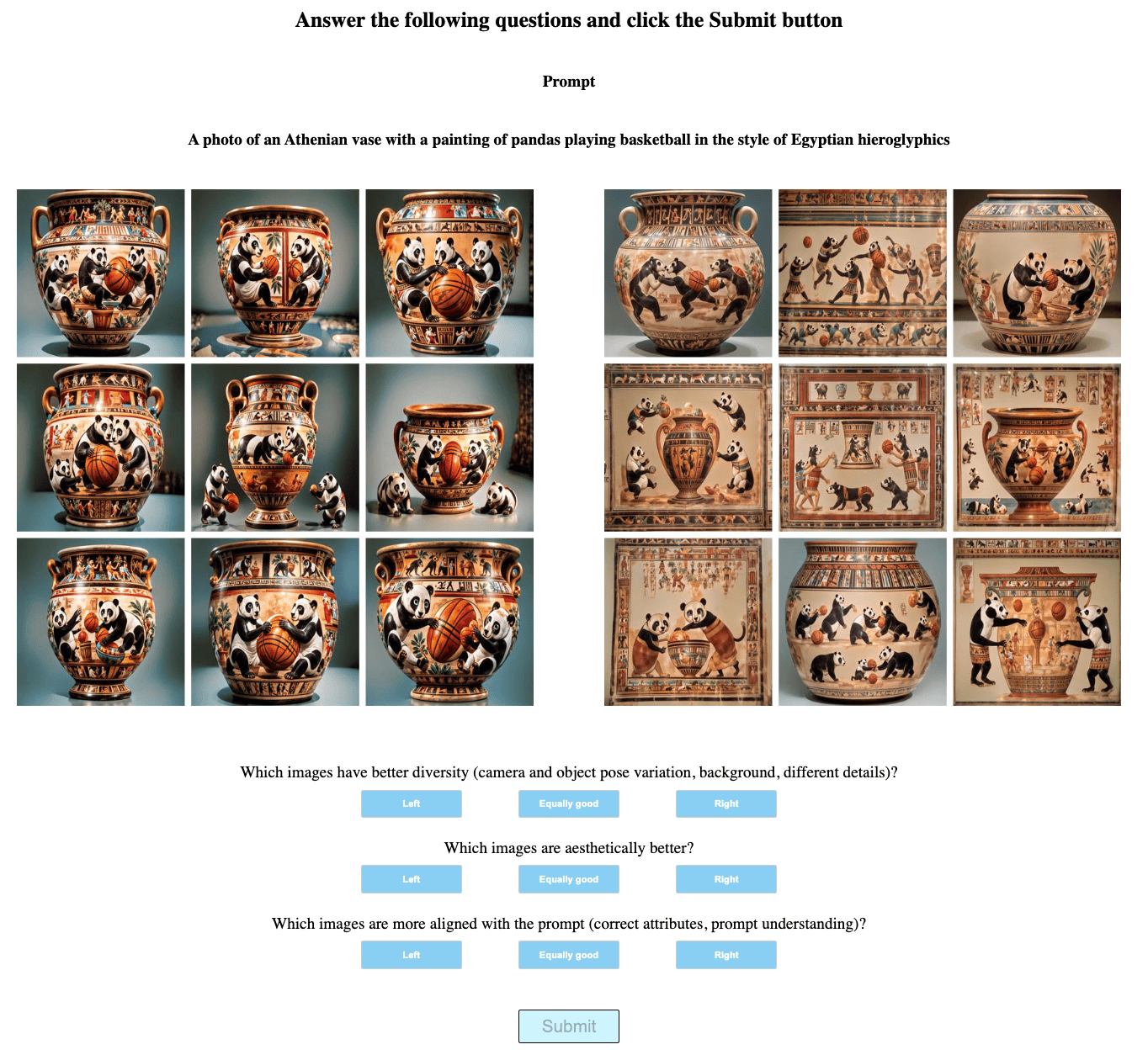}}
    \end{minipage}
    \caption{\textbf{Minimalistic web app designed for user study.} A web app built with Flask dynamically selects a prompt from the coverage dataset, selects nine random indices from the 50 generated images without replacement, randomly shuffles the order, and displays the images side by side.
    This is followed by three questions. Users click on either option and hit Submit. Upon hitting submit, the vote is recorded and a new set of images are shown.
    }
    \label{fig:userstudyui}
\end{figure*}

\subsection{Implementation Details}
\label{sec:impldetails}

\subsubsection{Training details}
We initialize the model with learnable LoRA parameters in the UNet of the latent diffusion model, and keep all other components (text encoders, VAE decoder, reward model) frozen.
All models are trained on 8 NVIDIA H100 GPUs.
We use a micro batch size of 1 with 4 gradient accumulation steps. 
For each model, we generate images of the recommended resolution, i.e. images with resolution 512$\times$512 for SDv1.4 and images with resolution 1024$\times$1024 for SDXL.
For all models, a constant learning rate of 2.5e-4 is used, without any warmup, annealing, decay or warm restarts.
We use the AdamW optimizer for all experiments with $\beta_1 = 0.9, \beta_2 = 0.999$, and a gradient clipping parameter of $0.1$.
To save memory, all models are trained with BF16 mixed precision training, with DDP level parallelism.

\begin{figure}[ht!]
    \begin{tabularx}{\linewidth}{YYYYYY}
        ReNO & DOODL & ReFL & {\draft} & Ours & SDXL Base \\
    \end{tabularx}
    \begin{minipage}{\linewidth}
        \centering
        \includegraphics[width=\linewidth]{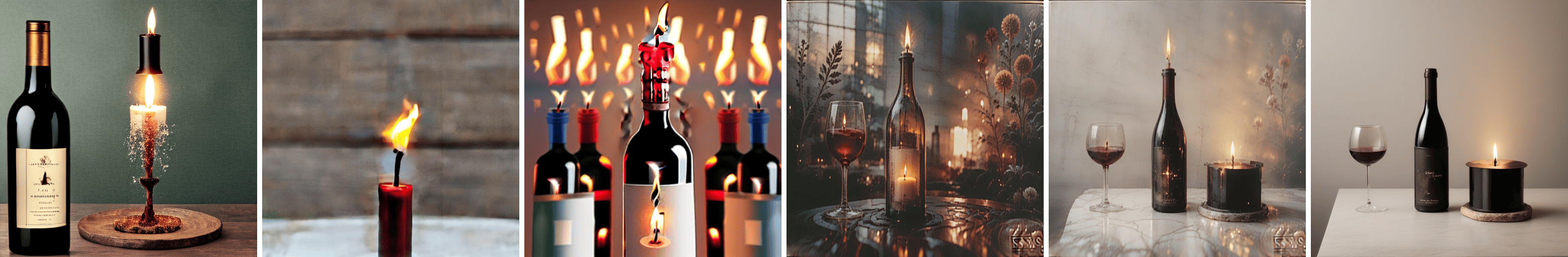}
        \textbf{Prompt}: \input{images/qualitative-allmethods/prompt00513.txt}
    \end{minipage}
    \vspace*{3pt}

    \begin{minipage}{\linewidth}
        \centering
        \includegraphics[width=\linewidth]{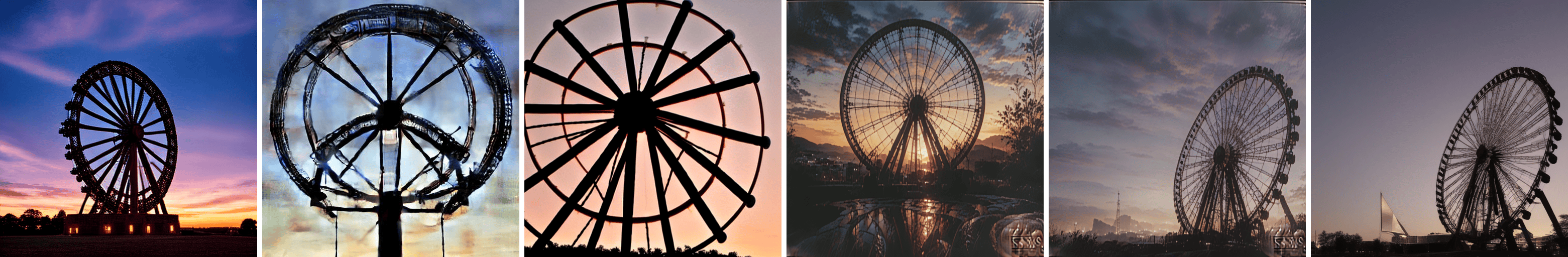}
        \textbf{Prompt}: \input{images/qualitative-allmethods/prompt01522.txt}
    \end{minipage}
    \vspace*{3pt}

    \begin{minipage}{\linewidth}
        \centering
        \includegraphics[width=\linewidth]{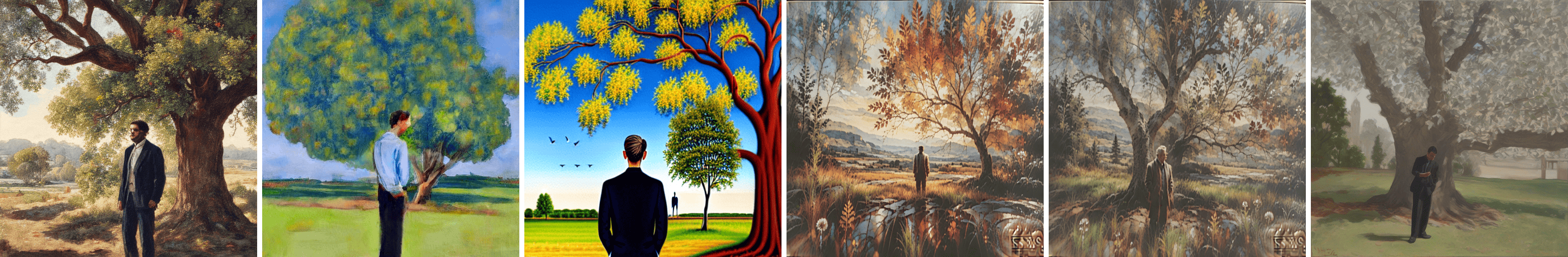}
        \textbf{Prompt}: \input{images/qualitative-allmethods/prompt01525.txt}
    \end{minipage}
    \vspace*{3pt}

    \begin{minipage}{\linewidth}
        \centering
        \includegraphics[width=\linewidth]{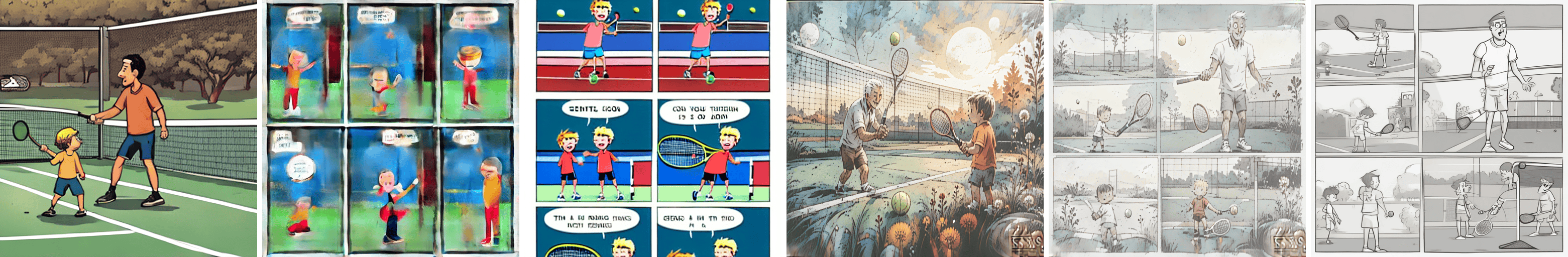}
        \textbf{Prompt}: \input{images/qualitative-allmethods/prompt01532.txt}
    \end{minipage}
    \vspace*{3pt}

    \begin{minipage}{\linewidth}
        \centering
        \includegraphics[width=\linewidth]{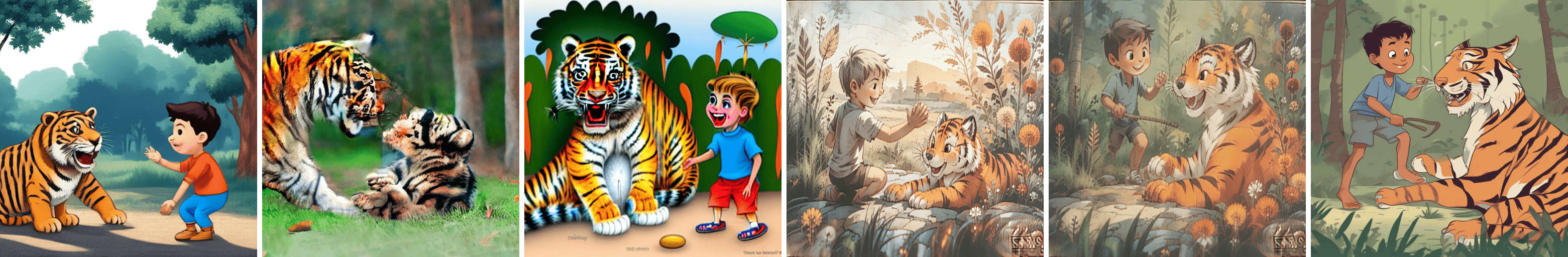}
        \textbf{Prompt}: \input{images/qualitative-allmethods/prompt01533.txt}
    \end{minipage}
    \vspace*{3pt}

    \begin{minipage}{\linewidth}
        \centering
        \includegraphics[width=\linewidth]{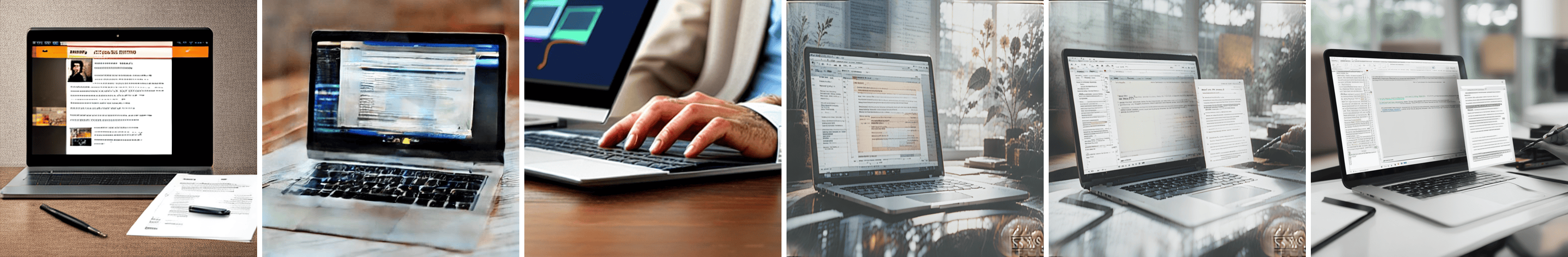}
        \textbf{Prompt}: \input{images/qualitative-allmethods/prompt01616.txt}
    \end{minipage}
    \vspace*{3pt}

    \caption{\textbf{Qualitative comparison of reward finetuning methods on PartiPrompt prompt dataset}. 
    Qualitatively, our model inherits the large-scale details from the base model, inheriting its diversity, but generated images follow stylistic aspects of the {\draft} model.
    }
    \label{fig:qualall-1}
\end{figure}

\begin{figure}[ht!]
    \begin{tabularx}{\linewidth}{YYYYYY}
        ReNO & DOODL & ReFL & {\draft} & Ours & SDXL Base \\
    \end{tabularx}
    \begin{minipage}{\linewidth}
        \centering
        \includegraphics[width=\linewidth]{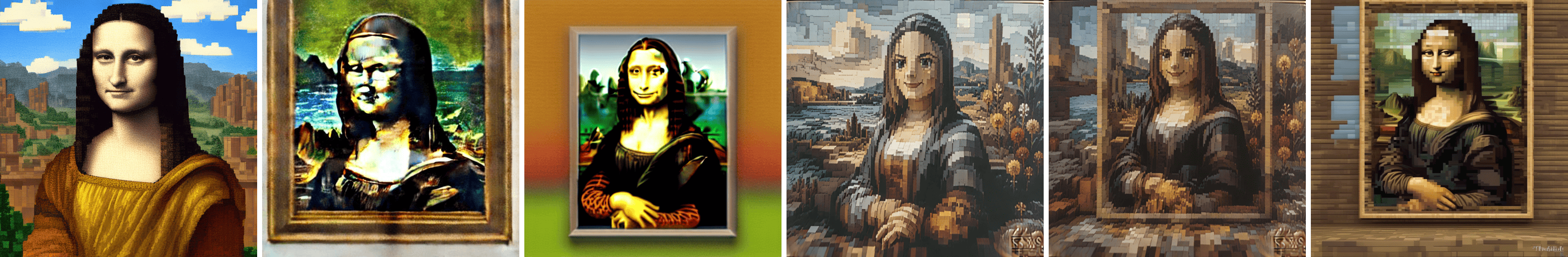}
        \textbf{Prompt}: \input{images/qualitative-allmethods/prompt01537.txt}
    \end{minipage}
    \vspace*{3pt}

    \begin{minipage}{\linewidth}
        \centering
        \includegraphics[width=\linewidth]{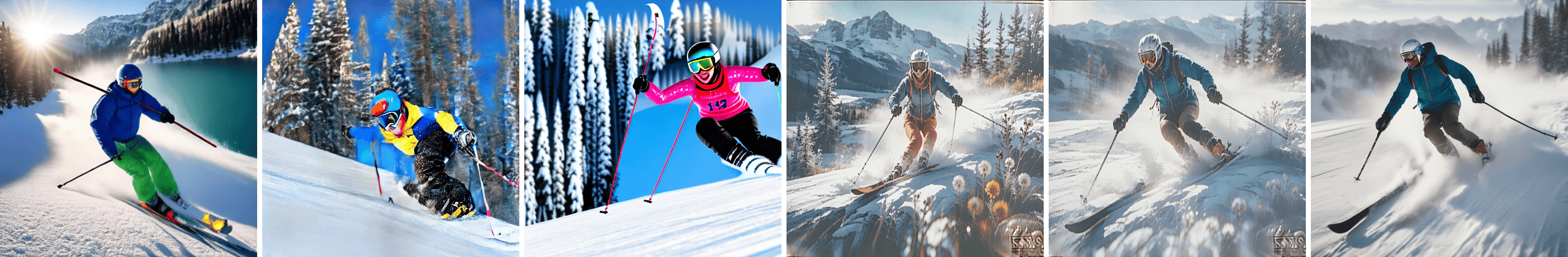}
        \textbf{Prompt}: \input{images/qualitative-allmethods/prompt01538.txt}
    \end{minipage}
    \vspace*{3pt}

    \begin{minipage}{\linewidth}
        \centering
        \includegraphics[width=\linewidth]{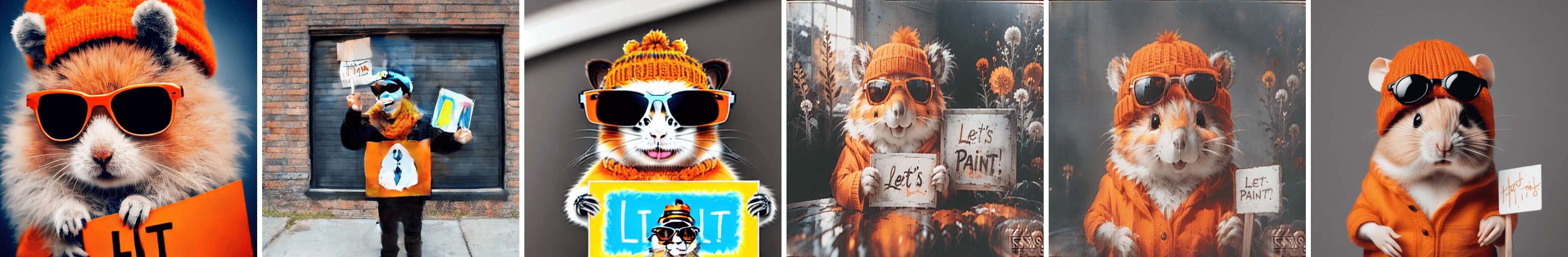}
        \textbf{Prompt}: \input{images/qualitative-allmethods/prompt01574.txt}
    \end{minipage}
    \vspace*{3pt}

    \begin{minipage}{\linewidth}
        \centering
        \includegraphics[width=\linewidth]{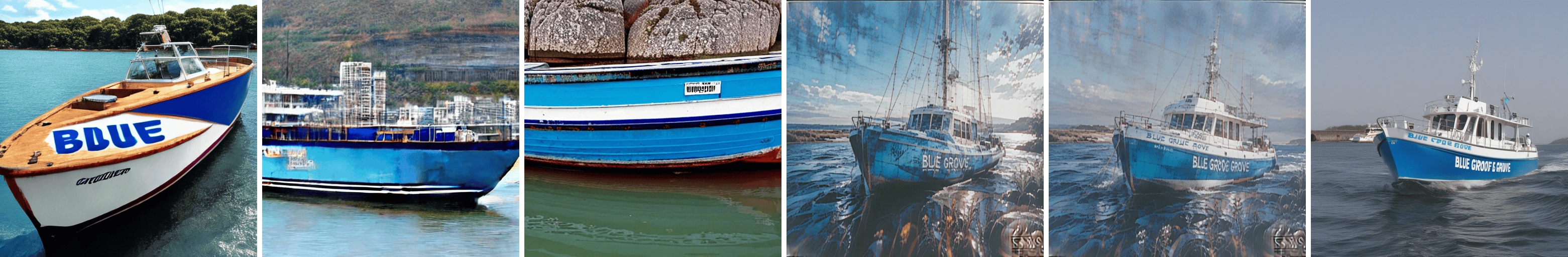}
        \textbf{Prompt}: \input{images/qualitative-allmethods/prompt01580.txt}
    \end{minipage}
    \vspace*{3pt}

    \begin{minipage}{\linewidth}
        \centering
        \includegraphics[width=\linewidth]{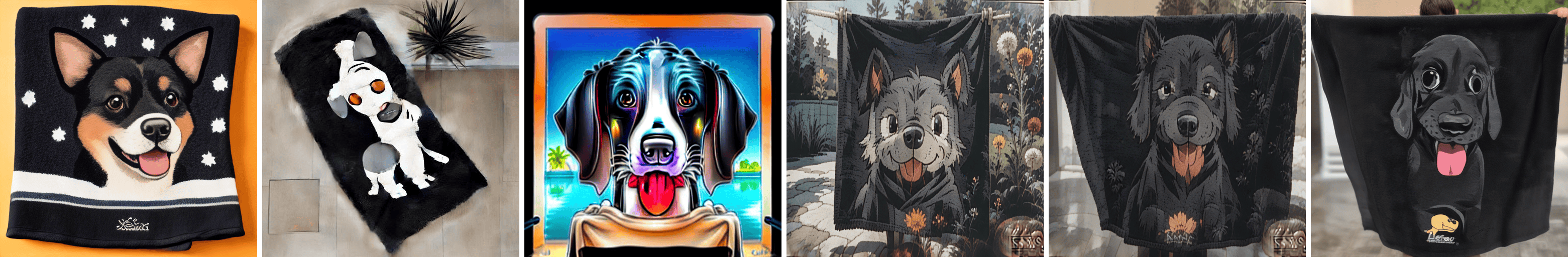}
        \textbf{Prompt}: \input{images/qualitative-allmethods/prompt01592.txt}
    \end{minipage}
    \vspace*{3pt}

    \begin{minipage}{\linewidth}
        \centering
        \includegraphics[width=\linewidth]{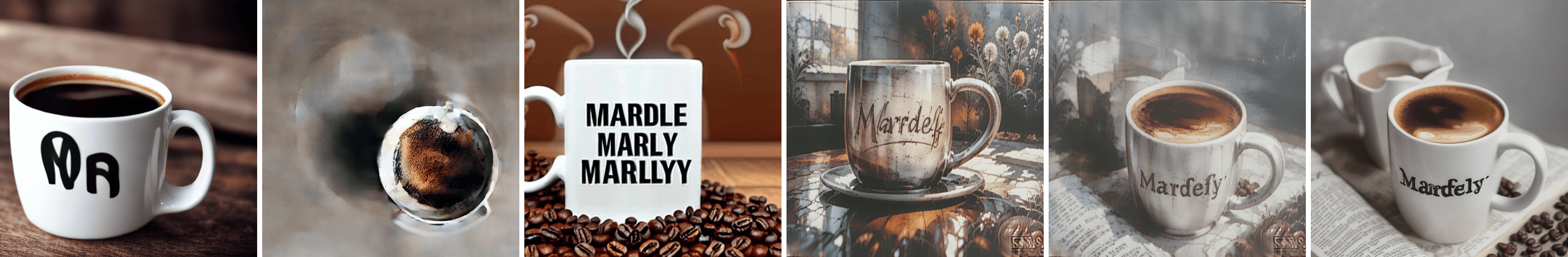}
        \textbf{Prompt}: \input{images/qualitative-allmethods/prompt01595.txt}
    \end{minipage}
    \vspace*{3pt}

    \caption{\textbf{Qualitative comparison of reward finetuning methods on PartiPrompt prompt dataset}.
    Qualitatively, our model inherits the large-scale details from the base model, inheriting its diversity, but generated images follow stylistic aspects of the {\draft} model.
    }
    \label{fig:qualall-2}
\end{figure}

\subsubsection{Reward Models}
The HPSv2~\cite{hps} model is trained on the Human Preference Dataset v2.
HPDv2 is a large-scale dataset with 798k binary preference choices for 434k images. Each pair contains two images generated by different models using the same prompt, and is annotated with a binary choice made by one annotator.
The prompts are collected from DrawBench and DiffusionDB containing user-written prompts, the latter of which is `sanitized' using ChatGPT to remove biases arising due to style words and leads to a reduced NSFW score.
The PickScore~\cite{pickscore} model is trained on the Pick-a-Pic dataset.
The Pick-a-Pic dataset was created using a web application where users can write a prompt and are presented with two generated images, and they are asked to select their preferred option or indicate a tie if they have no strong opinion about either image. 
Although moderation is done to remove users who generate NSFW images or make judgements at a rapid pace (indicating low quality or random preference), the Pick-a-Pic model contains a lot of NSFW prompts.
Consequently, we observe more NSFW generated images when finetuned with the Pickscore model compared to the HPSv2 model, even when prompts are not NSFW.
However, the Pickscore model also generates more aesthetically pleasing images.

\subsubsection{Choice of $\gamma$}
Unless the KL parameter $\lambda$ or LoRA scaling parameter $\alpha'$ that are scalar quantities, {\methodabbv} requires a function $\gamma(t)$.
In this paper, we consider the family of functions $$ \gamma_{p,T}(t) = 1 - \left(\frac{T-t}{T}\right)^p $$
For $p=1$, the weighing is simply linear, i.e. $\gamma_{1,T}(t) = \frac{t}{T}$. For $p>1$, the power term quickly vanishes and the sampling dynamics are governed by the base model for more timesteps.
For $p<1$, the power term remains close to 1, therefore diminishing the effect of the base model in the earlier timesteps.
We consider $p = 1, 1.25, 1.5, 2, 3, 4, 5$ for the ablations in the paper, and $p = 2$ for the qualitative studies.
However, we note that more sophisticated $\gamma$ scheduling is possible, i.e. $\gamma_{p,T}(t) = H\left(\frac{t}{T} - p\right)$, or $\gamma_{\kappa,T}(t) = \sigma\left(\kappa(t - T/2)\right) $ where $H$ is the Heaviside step function, and $\sigma$ is the sigmoid function.
We leave exploration of these sophisticated scheduling functions to future work. 

\begin{figure}[ht!]
    \centering
    \begin{minipage}{0.32\linewidth}
        \subcaption{\draft}
        \centering
        \includegraphics[width=\textwidth]{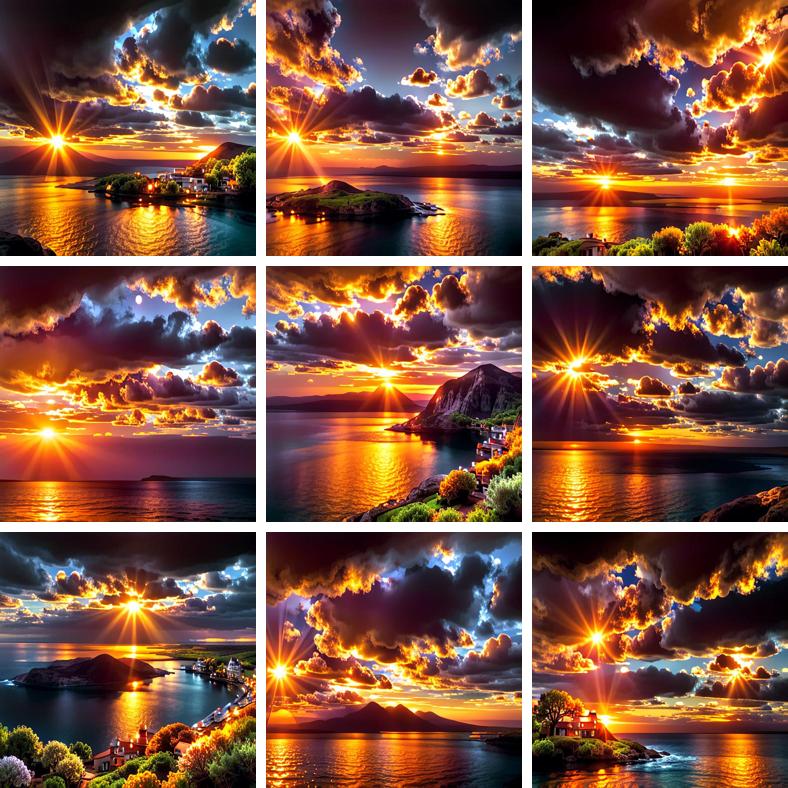} 
    \end{minipage} \hfill
    \begin{minipage}{0.32\linewidth}
        \subcaption{Base model}
        \centering
        \includegraphics[width=\textwidth]{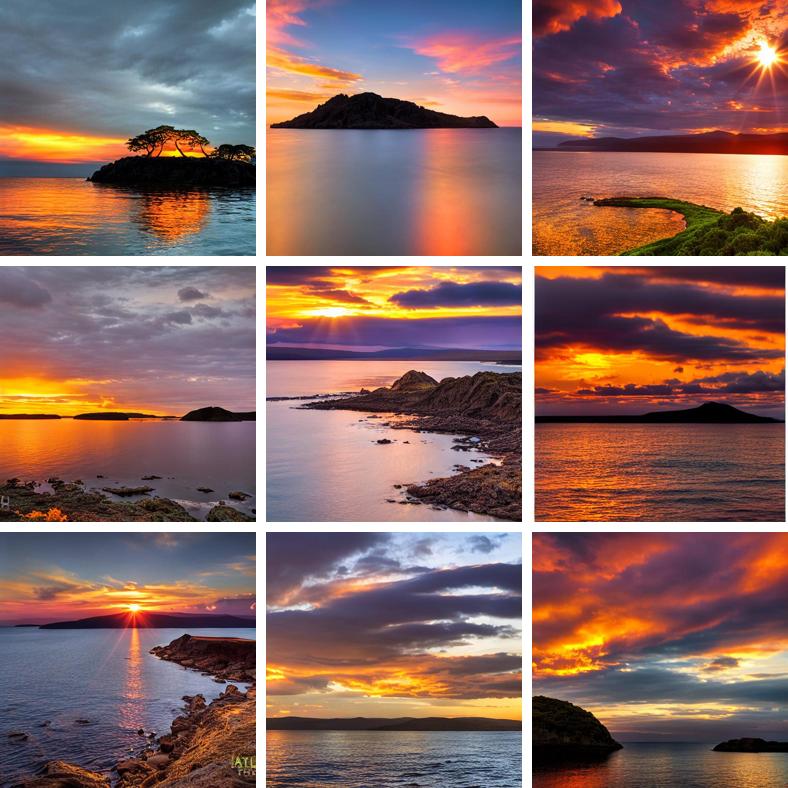} 
    \end{minipage} \hfill
    \begin{minipage}{0.32\linewidth}
        \subcaption{Ours}
        \centering
        \includegraphics[width=\textwidth]{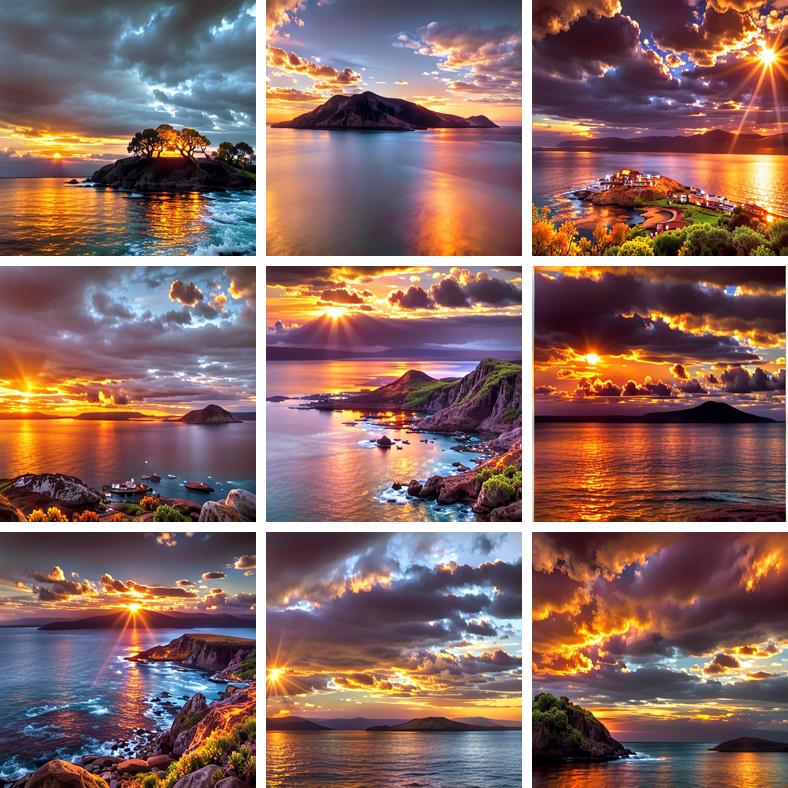}
    \end{minipage}
    \input{images/qualitative/sd_hps/prompt-1.txt} \\
    \textcolor{white}{filler} \\
    \begin{minipage}{0.32\linewidth}
        \centering
        \includegraphics[width=\textwidth]{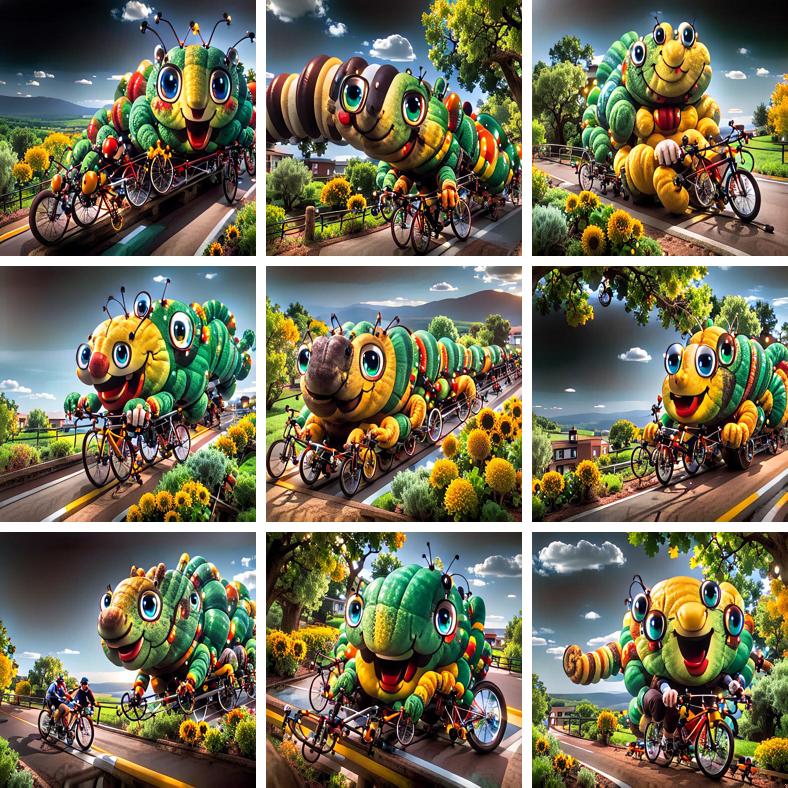} 
    \end{minipage} \hfill
    \begin{minipage}{0.32\linewidth}
        \centering
        \includegraphics[width=\textwidth]{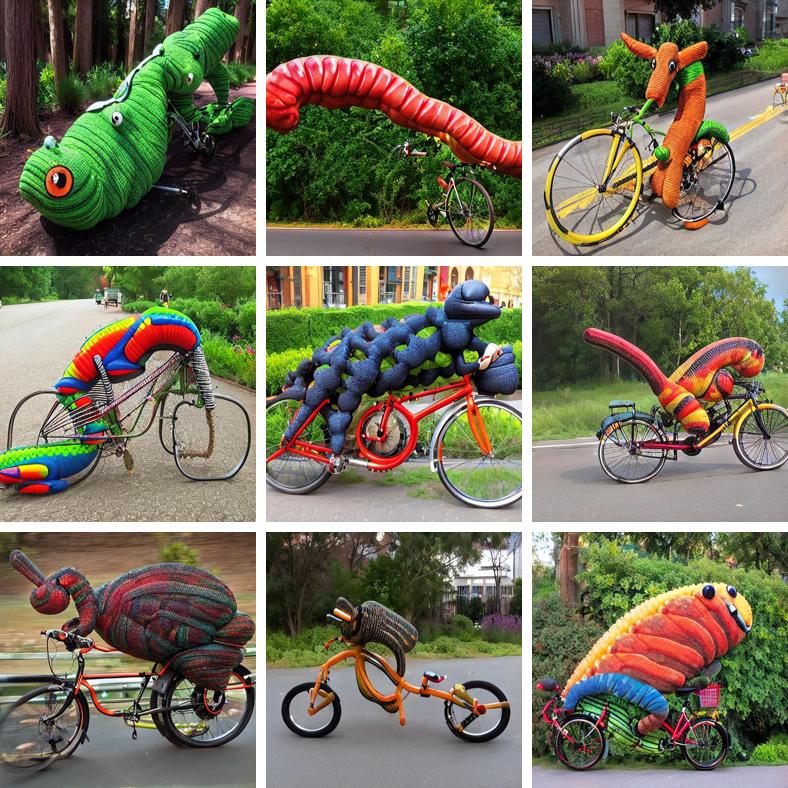} 
    \end{minipage} \hfill
    \begin{minipage}{0.32\linewidth}
        \centering
        \includegraphics[width=\textwidth]{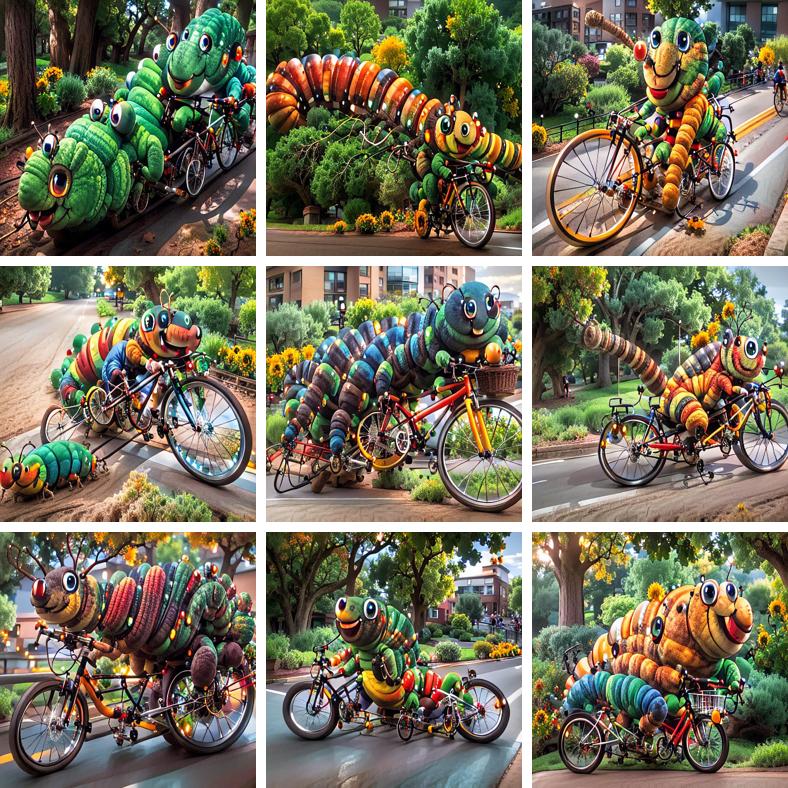}
    \end{minipage}
    \input{images/qualitative/sd_hps/prompt-2.txt} \\
    \textcolor{white}{filler} \\
    \begin{minipage}{0.32\linewidth}
        \centering
        \includegraphics[width=\textwidth]{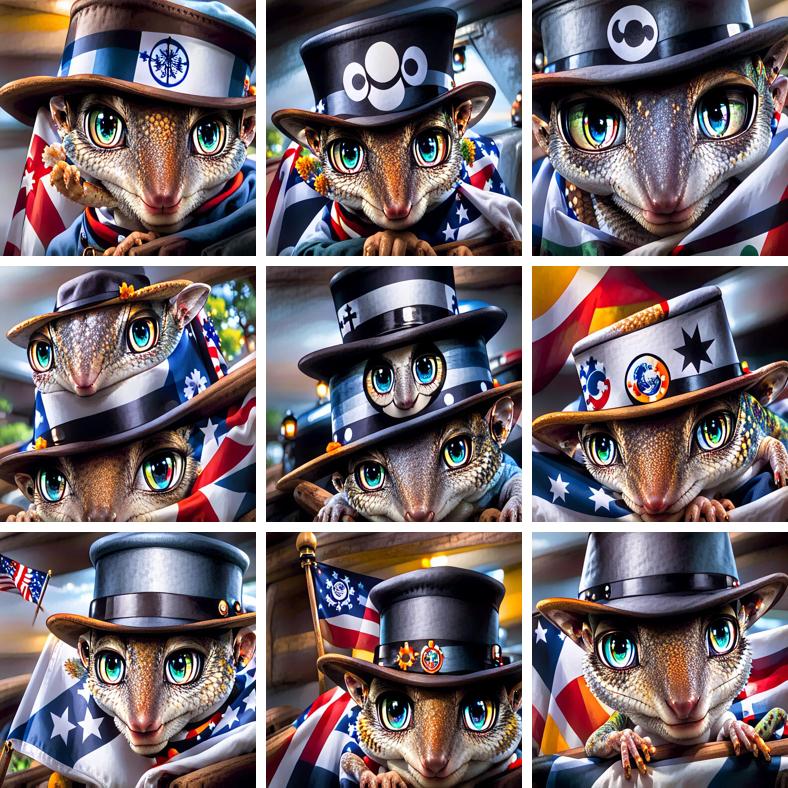} 
    \end{minipage} \hfill
    \begin{minipage}{0.32\linewidth}
        \centering
        \includegraphics[width=\textwidth]{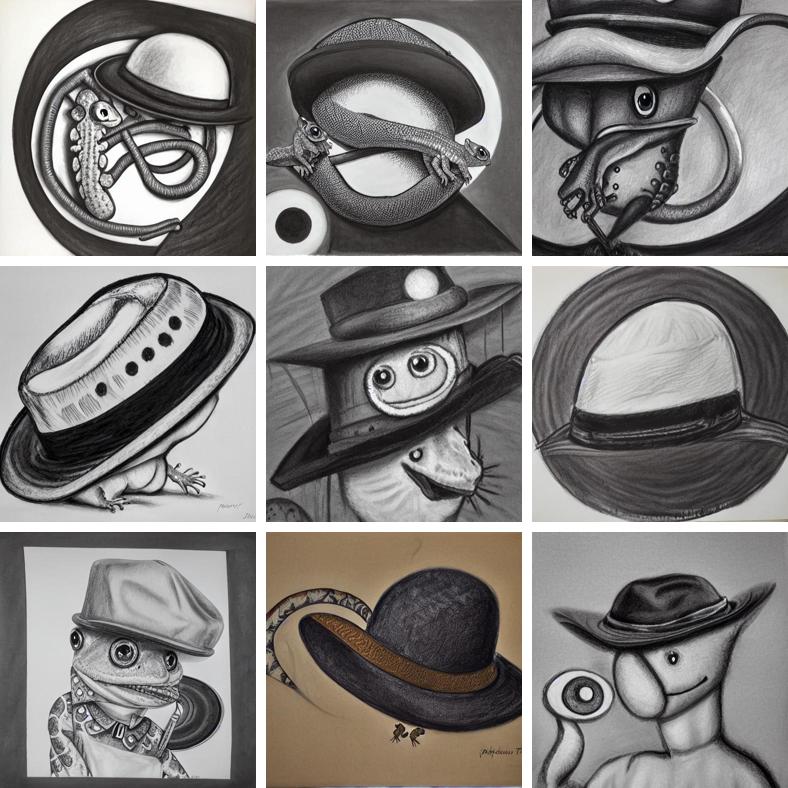} 
    \end{minipage} \hfill
    \begin{minipage}{0.32\linewidth}
        \centering
        \includegraphics[width=\textwidth]{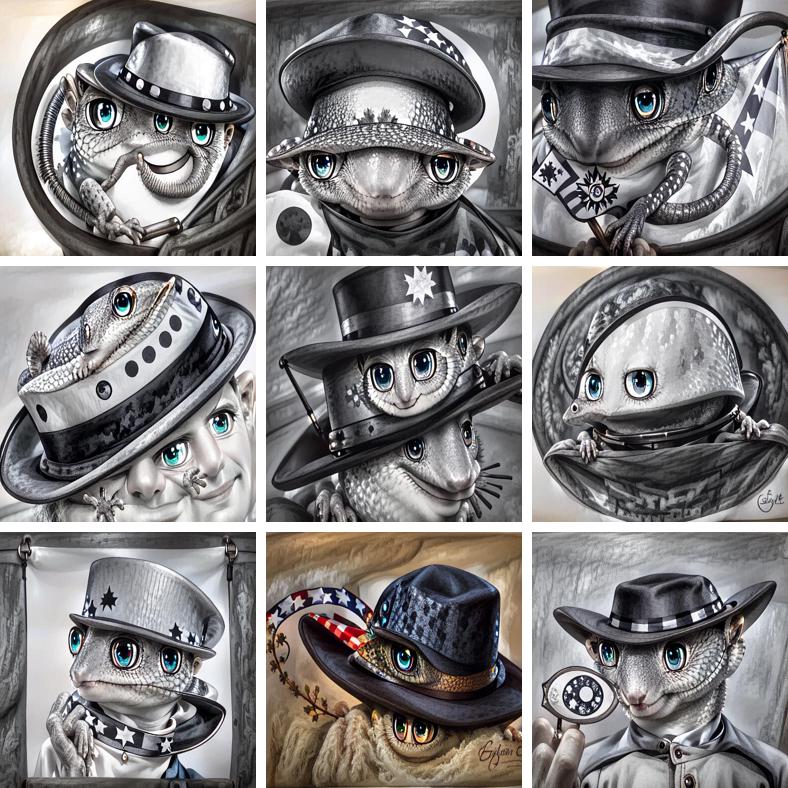}
    \end{minipage}
    \input{images/qualitative/sd_hps/prompt-3.txt} \\
    \textcolor{white}{filler} \\
    \caption{\textbf{Qualitative comparison of {\draft} and {\methodabbv}}. Three columns of rows show set of nine images generated from the same seeds by the (a){\draft}, (b)Base model, and (c)Our model. 
    Our method preserves the diversity of details of different images, while adding aesthetic quality leading to both high rewards and high user preference.
    }
    \label{fig:qualdiv1}
\end{figure}

\begin{figure}[ht!]
    \centering
    \begin{minipage}{0.32\linewidth}
        \huge{\subcaption{\draft}}
        \centering
        \includegraphics[width=\textwidth]{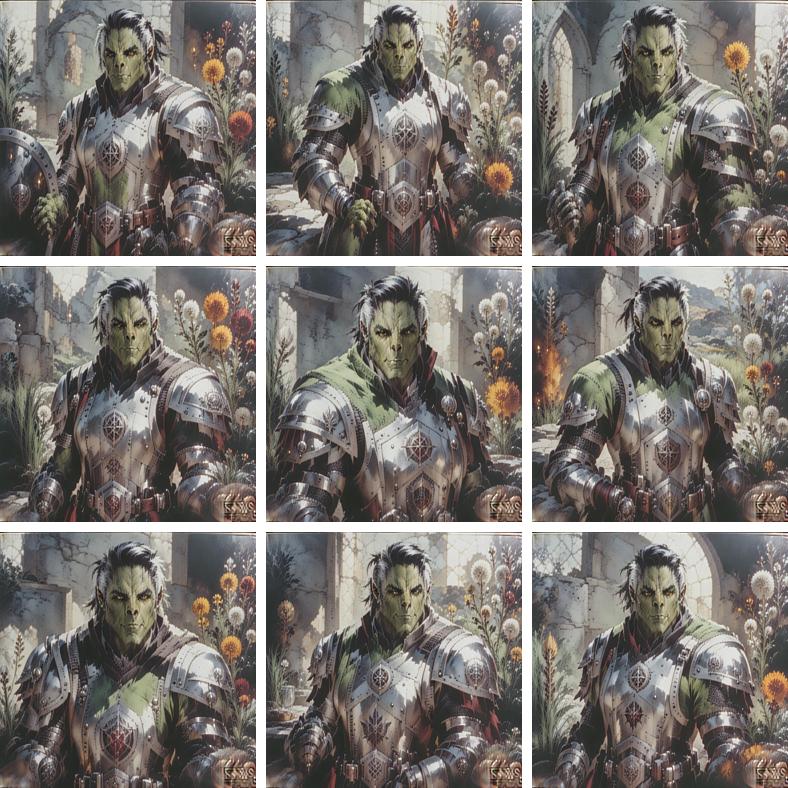} 
    \end{minipage} \hfill
    \begin{minipage}{0.32\linewidth}
        \subcaption{Base model}
        \centering
        \includegraphics[width=\textwidth]{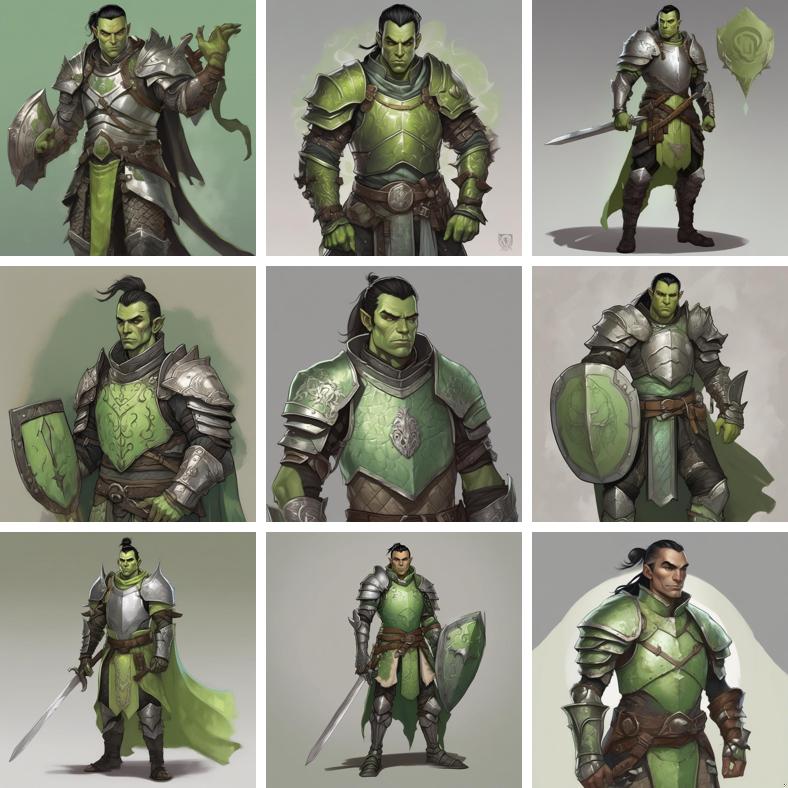} 
    \end{minipage} \hfill
    \begin{minipage}{0.32\linewidth}
        \subcaption{Ours}
        \centering
        \includegraphics[width=\textwidth]{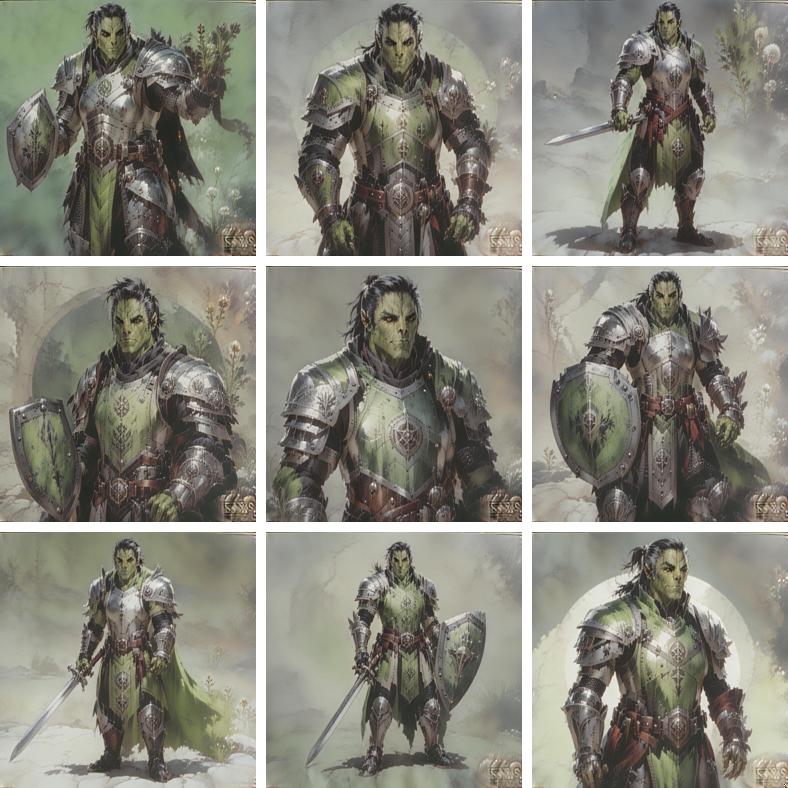}
    \end{minipage}
    \input{images/qualitative/sdxl_pickscore/prompt-1.txt} \\
    \textcolor{white}{filler} \\
    \begin{minipage}{0.32\linewidth}
        \centering
        \includegraphics[width=\textwidth]{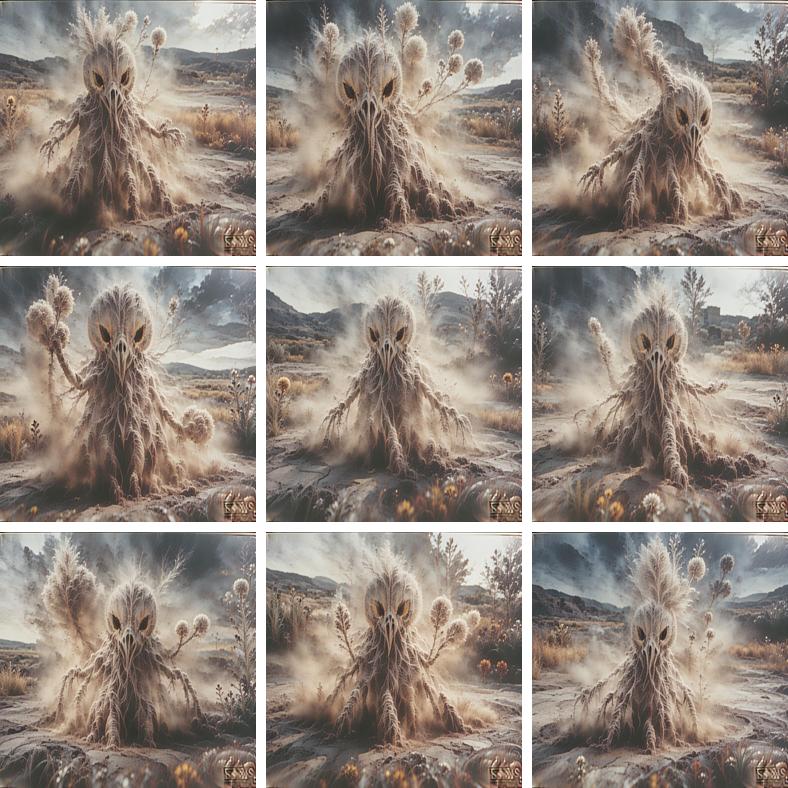} 
    \end{minipage} \hfill
    \begin{minipage}{0.32\linewidth}
        \centering
        \includegraphics[width=\textwidth]{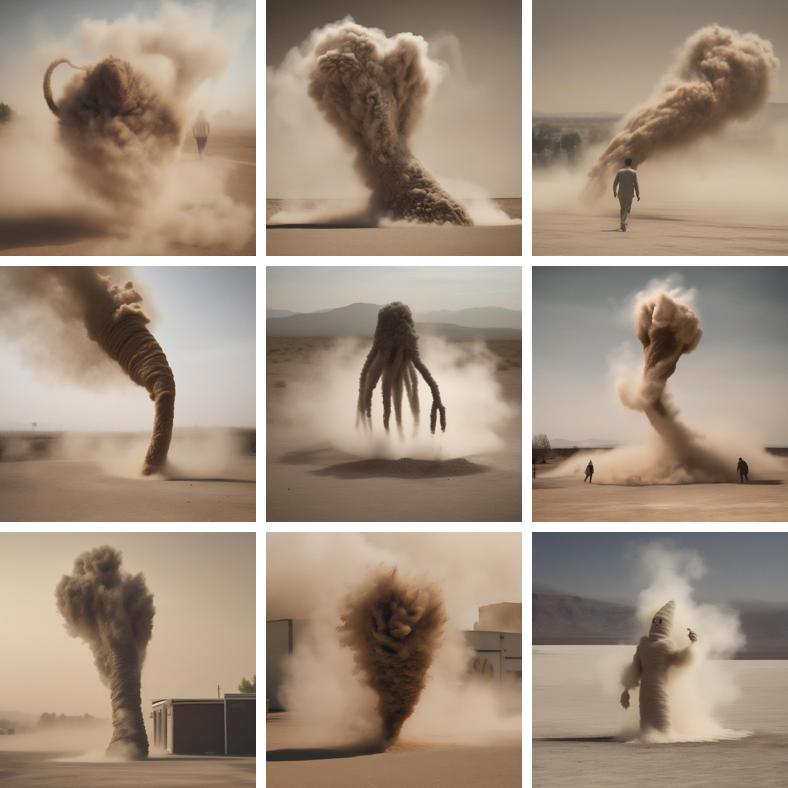} 
    \end{minipage} \hfill
    \begin{minipage}{0.32\linewidth}
        \centering
        \includegraphics[width=\textwidth]{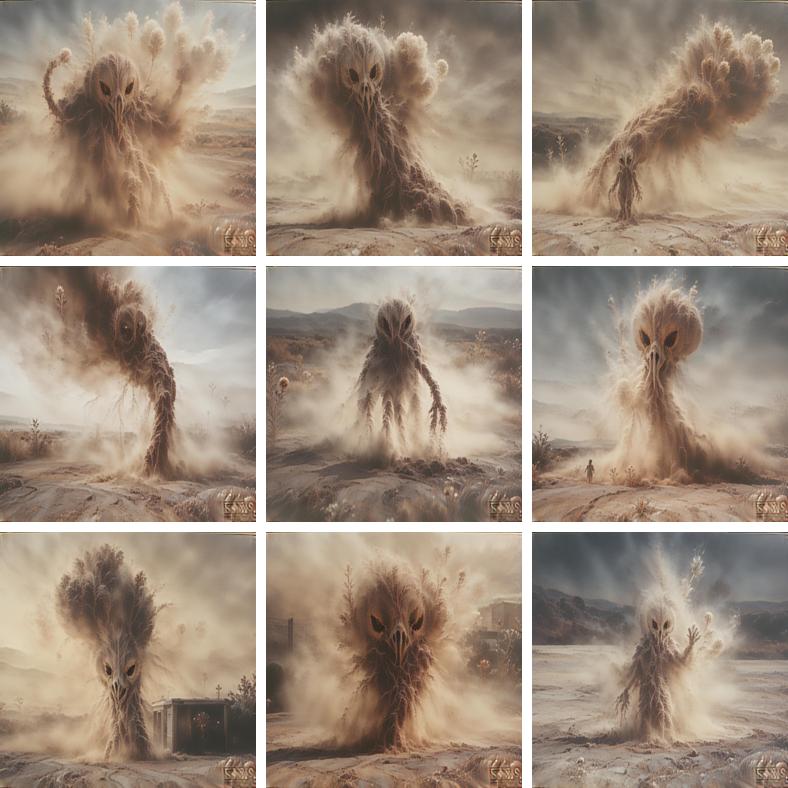}
    \end{minipage}
    \input{images/qualitative/sdxl_pickscore/prompt-2.txt} \\
    \textcolor{white}{filler} \\
    \caption{\textbf{Qualitative comparison of {\draft} and {\methodabbv}}. Three columns of rows show set of nine images generated from the same seeds by the (a){\draft}, (b)Base model, and (c)Our model. 
    Our method preserves the diversity of details of different images, while adding aesthetic quality leading to both high rewards and high user preference.
    }
    \label{fig:qualdiv4}
\end{figure}

\begin{figure}[ht!]
    \centering
    \begin{minipage}{0.32\linewidth}
        \subcaption{\draft}
        \centering
        \includegraphics[width=\textwidth]{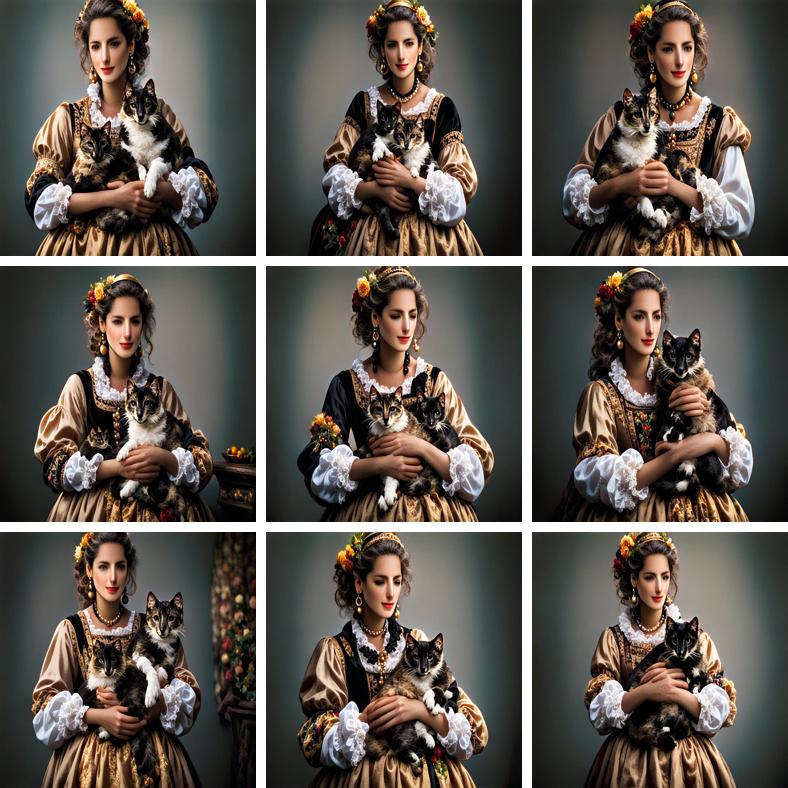} 
    \end{minipage} \hfill
    \begin{minipage}{0.32\linewidth}
        \subcaption{Base model}
        \centering
        \includegraphics[width=\textwidth]{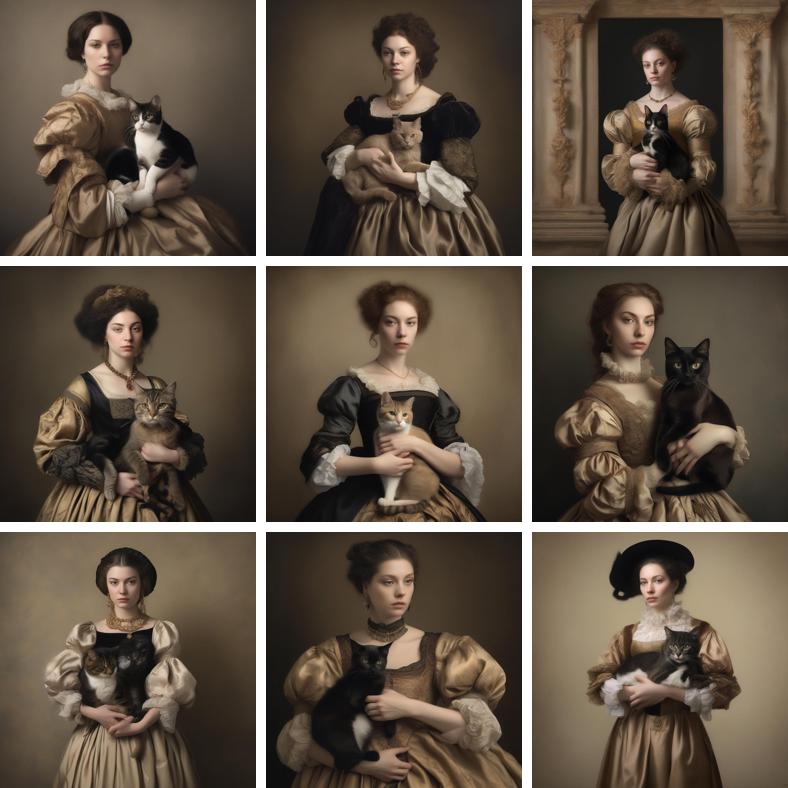} 
    \end{minipage} \hfill
    \begin{minipage}{0.32\linewidth}
        \subcaption{Ours}
        \centering
        \includegraphics[width=\textwidth]{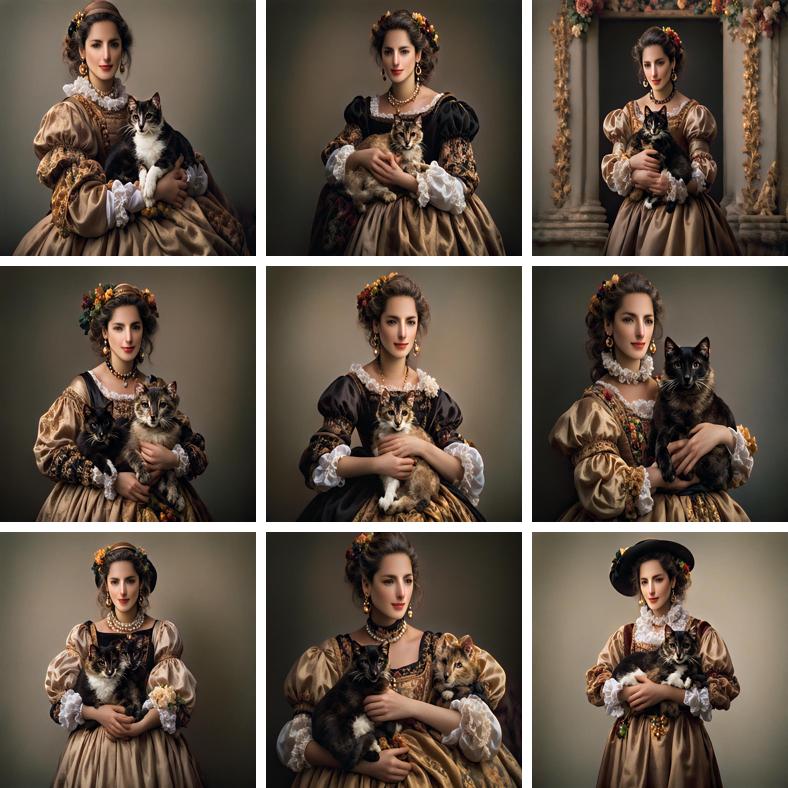}
    \end{minipage} \\
    \input{images/qualitative/sdxl_hps/prompt-1.txt} \\
    \textcolor{white}{filler} \\
    \begin{minipage}{0.32\linewidth}
        \centering
        \includegraphics[width=\textwidth]{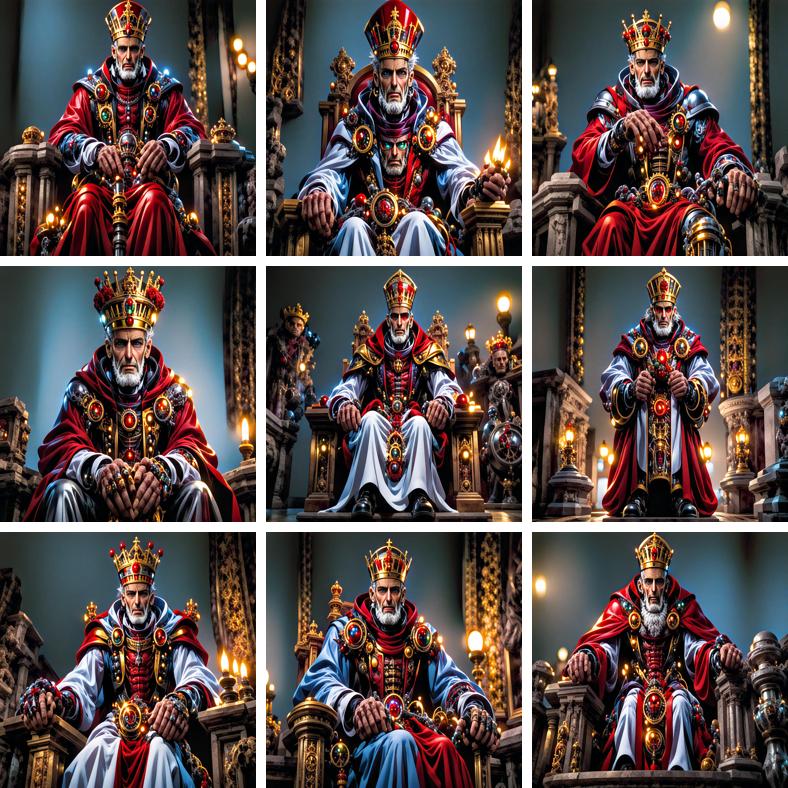} 
    \end{minipage} \hfill
    \begin{minipage}{0.32\linewidth}
        \centering
        \includegraphics[width=\textwidth]{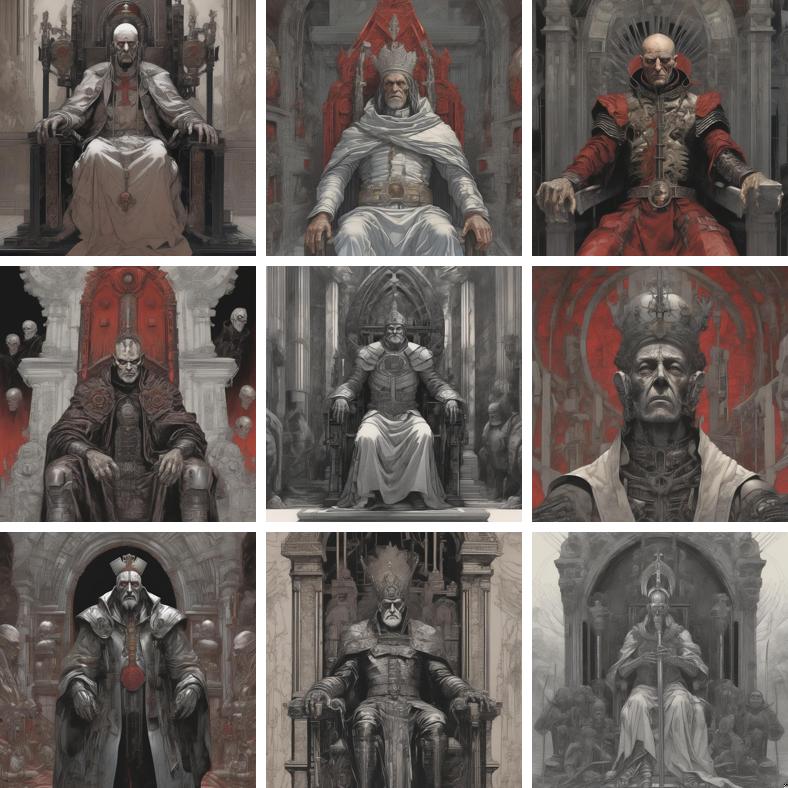} 
    \end{minipage} \hfill
    \begin{minipage}{0.32\linewidth}
        \centering
        \includegraphics[width=\textwidth]{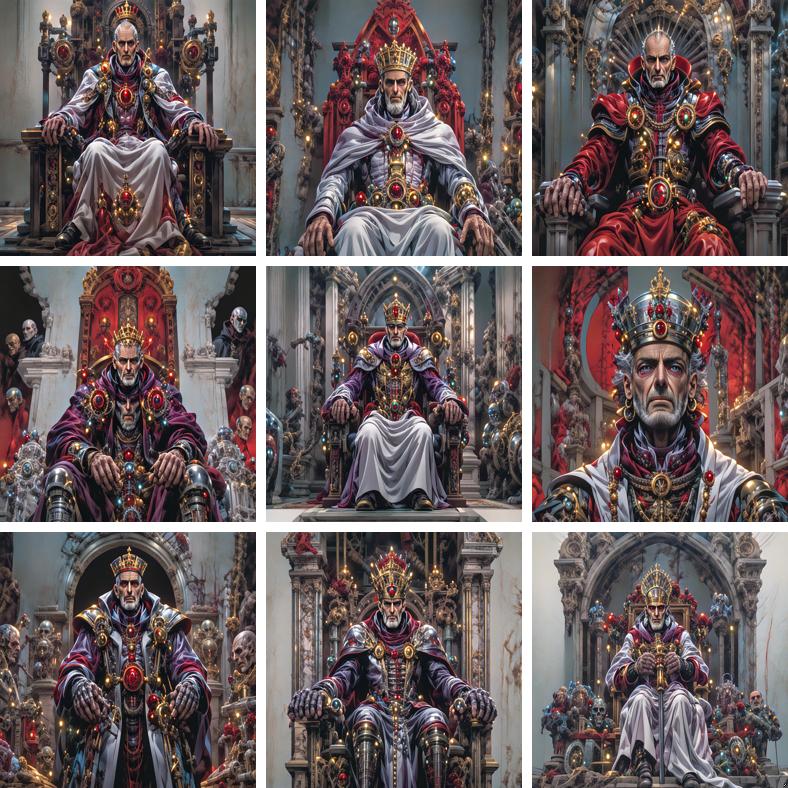}
    \end{minipage} \\
    \input{images/qualitative/sdxl_hps/prompt-2.txt} \\
    \textcolor{white}{filler} \\
    \begin{minipage}{0.32\linewidth}
        \centering
        \includegraphics[width=\textwidth]{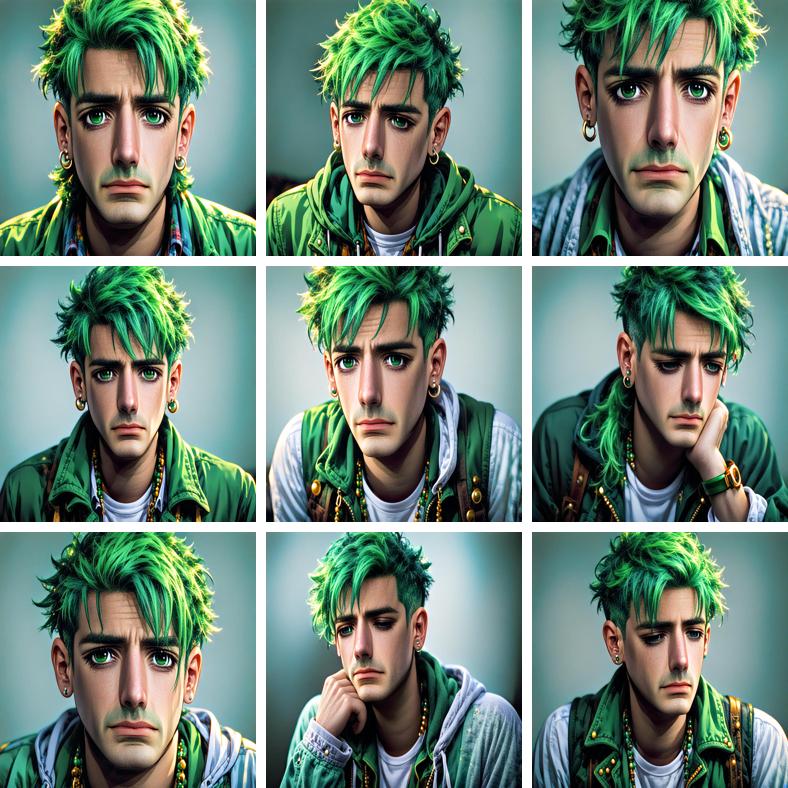} 
    \end{minipage} \hfill
    \begin{minipage}{0.32\linewidth}
        \centering
        \includegraphics[width=\textwidth]{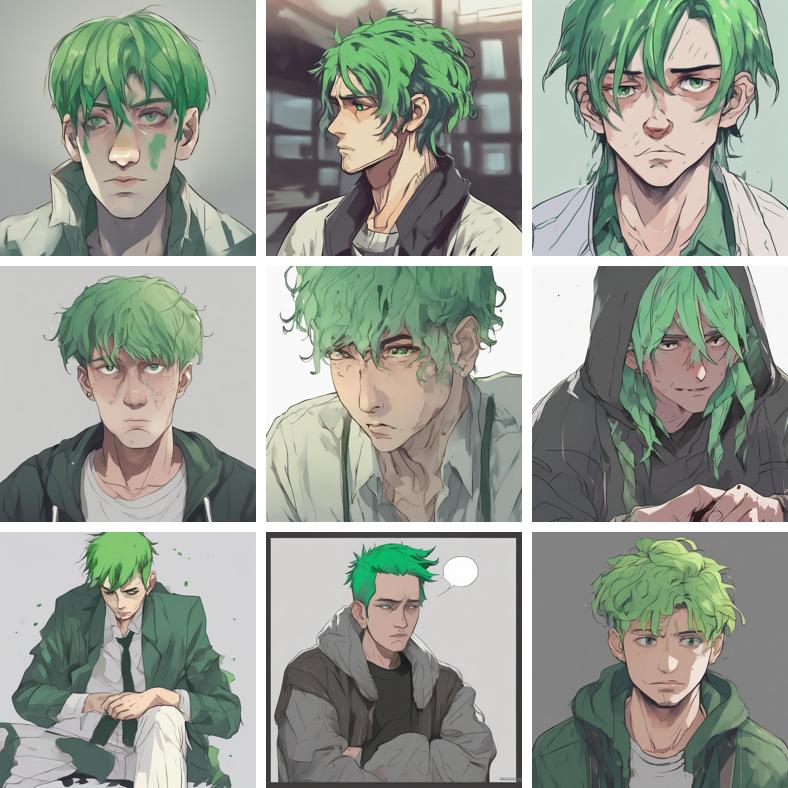} 
    \end{minipage} \hfill
    \begin{minipage}{0.32\linewidth}
        \centering
        \includegraphics[width=\textwidth]{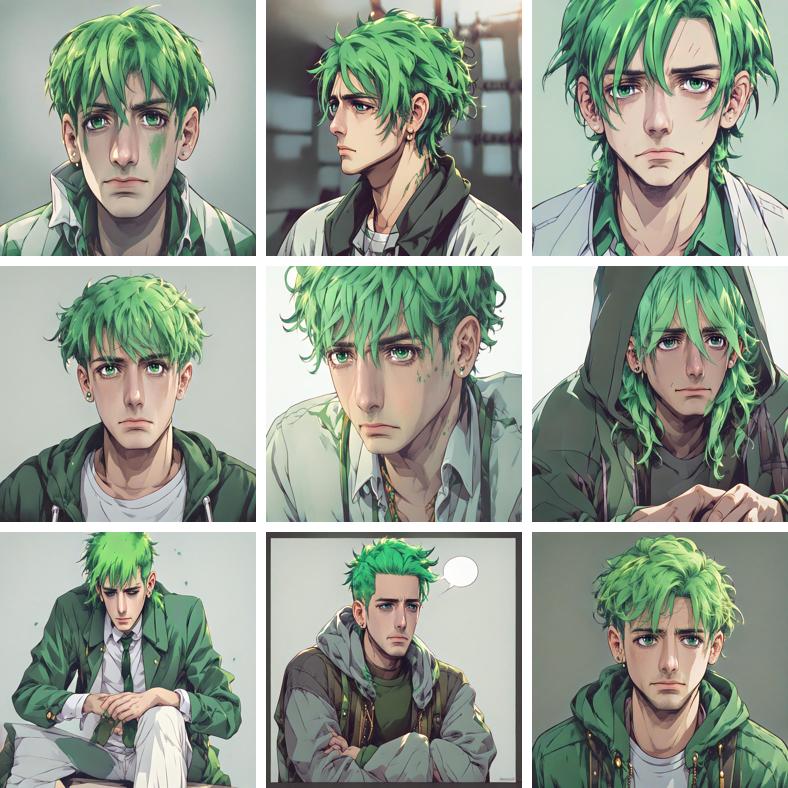}
    \end{minipage} \\
    \input{images/qualitative/sdxl_hps/prompt-3.txt} \\
    \textcolor{white}{filler} \\
    \caption{\textbf{Qualitative comparison of {\draft} and {\methodabbv}}. Three columns of rows show set of nine images generated from the same seeds by the (a){\draft}, (b)Base model, and (c)Our model. 
    Our method preserves the diversity of details of different images, while adding aesthetic quality leading to both high rewards and high user preference.
    }
    \label{fig:qualdiv3}
\end{figure}

\begin{figure}[ht!]
    \centering
    \begin{minipage}{0.32\linewidth}
        \subcaption{\draft}
        \centering
        \includegraphics[width=\textwidth]{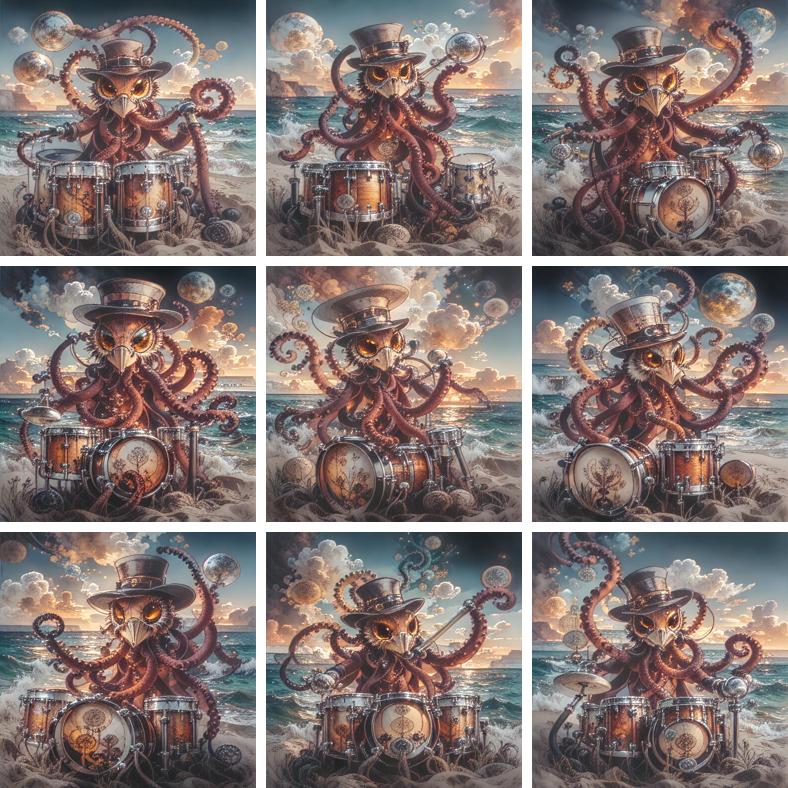} 
    \end{minipage} \hfill
    \begin{minipage}{0.32\linewidth}
        \subcaption{Base model}
        \centering
        \includegraphics[width=\textwidth]{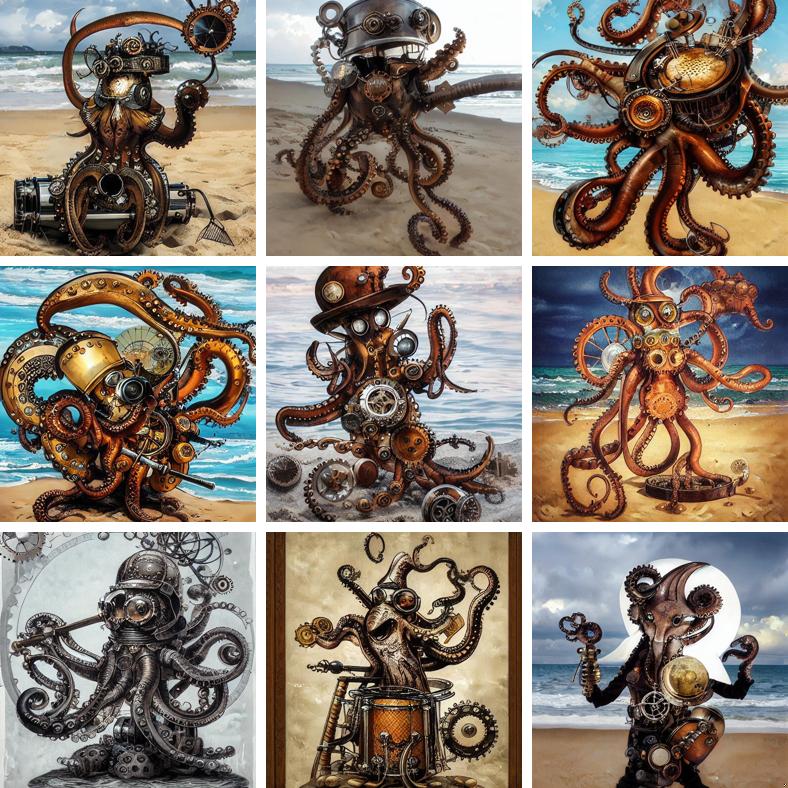} 
    \end{minipage} \hfill
    \begin{minipage}{0.32\linewidth}
        \subcaption{Ours}
        \centering
        \includegraphics[width=\textwidth]{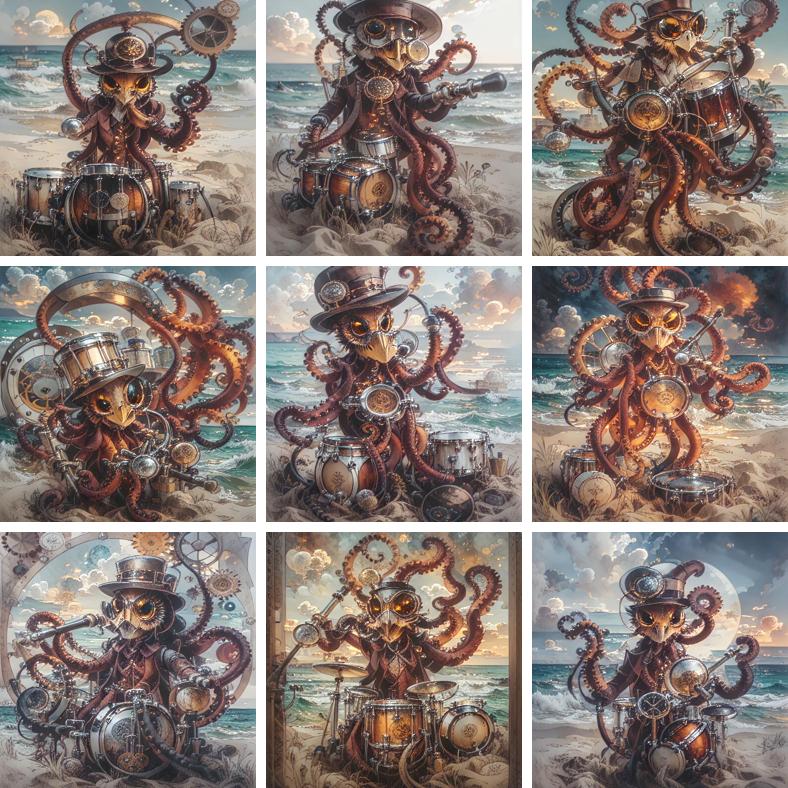}
    \end{minipage}  \\
    \input{images/qualitative/sd_pickscore/prompt-1.txt} \\
    \textcolor{white}{filler} \\
    \begin{minipage}{0.32\linewidth}
        \centering
        \includegraphics[width=\textwidth]{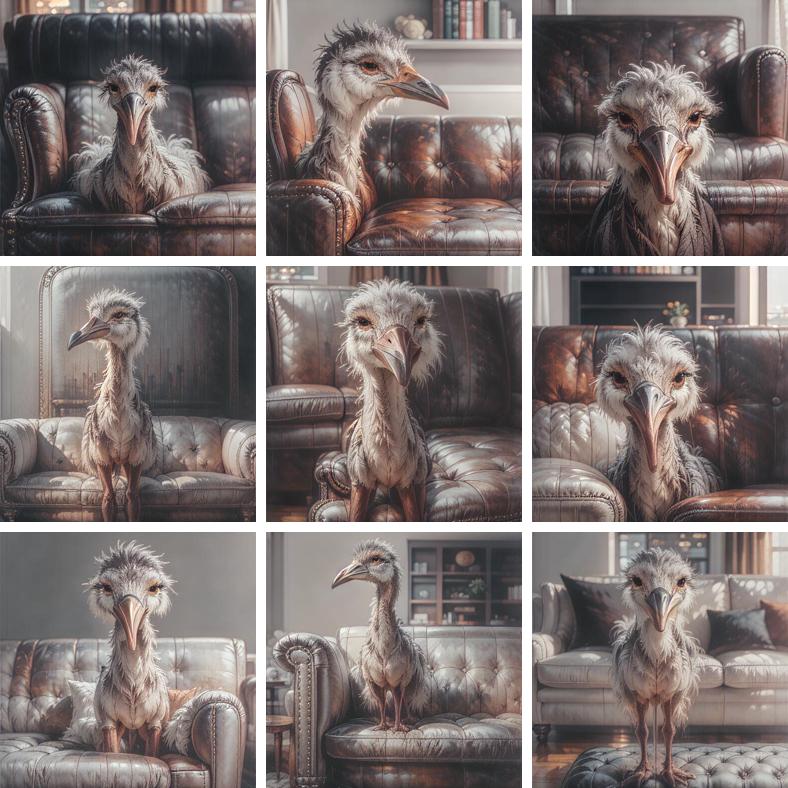} 
    \end{minipage} \hfill
    \begin{minipage}{0.32\linewidth}
        \centering
        \includegraphics[width=\textwidth]{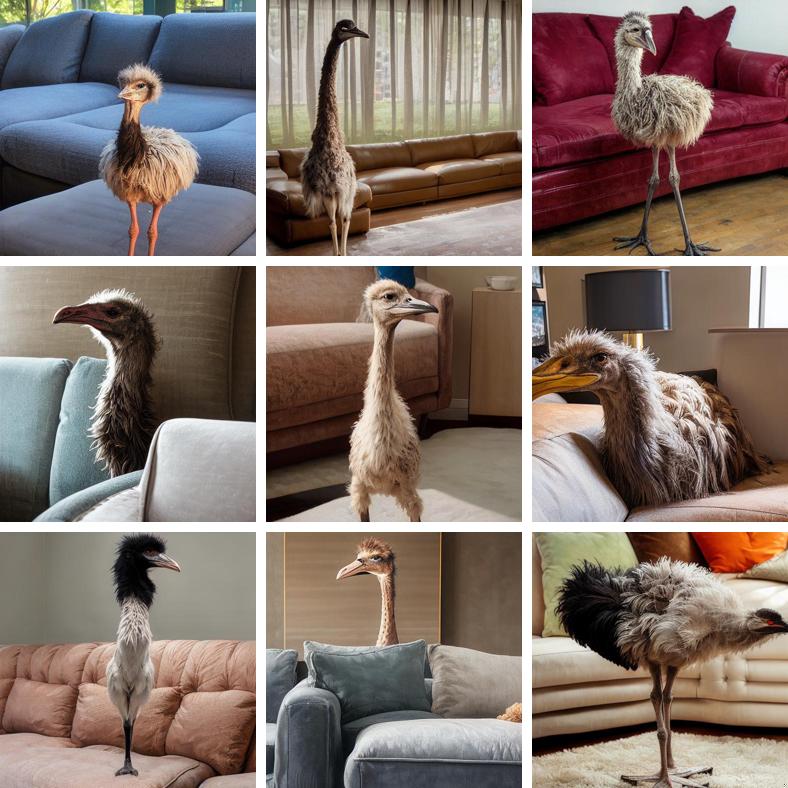} 
    \end{minipage} \hfill
    \begin{minipage}{0.32\linewidth}
        \centering
        \includegraphics[width=\textwidth]{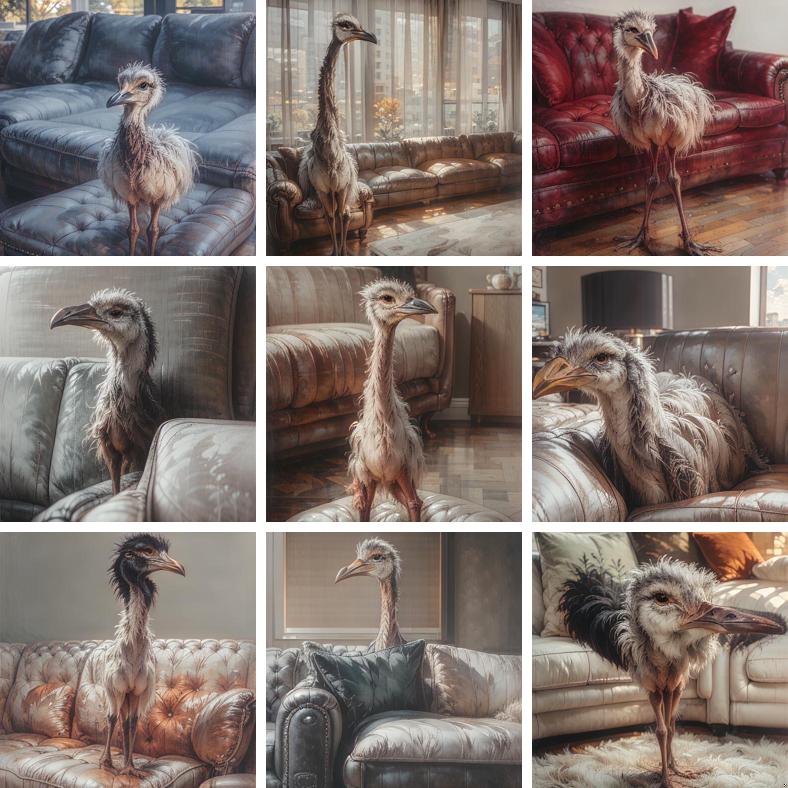}
    \end{minipage} \\
    \input{images/qualitative/sd_pickscore/prompt-2.txt} \\
    \caption{\textbf{Qualitative comparison of {\draft} and {\methodabbv}}. Three columns of rows show set of nine images generated from the same seeds by the (a){\draft}, (b)Base model, and (c)Our model. 
    Our method preserves the diversity of details of different images, while adding aesthetic quality leading to both high rewards and high user preference.
    }
    \label{fig:qualdiv2}
\end{figure}

\begin{figure}[ht!]
    \centering
    \begin{minipage}{0.85\linewidth}
        \includegraphics[width=\linewidth]{images/ablations/sdxl_pickscore_eval_PartiPrompt_pickscore.png}
        \includegraphics[width=\linewidth]{images/ablations/sdxl_pickscore_eval_hpsv2__anime__pickscore.png}
        \includegraphics[width=\linewidth]{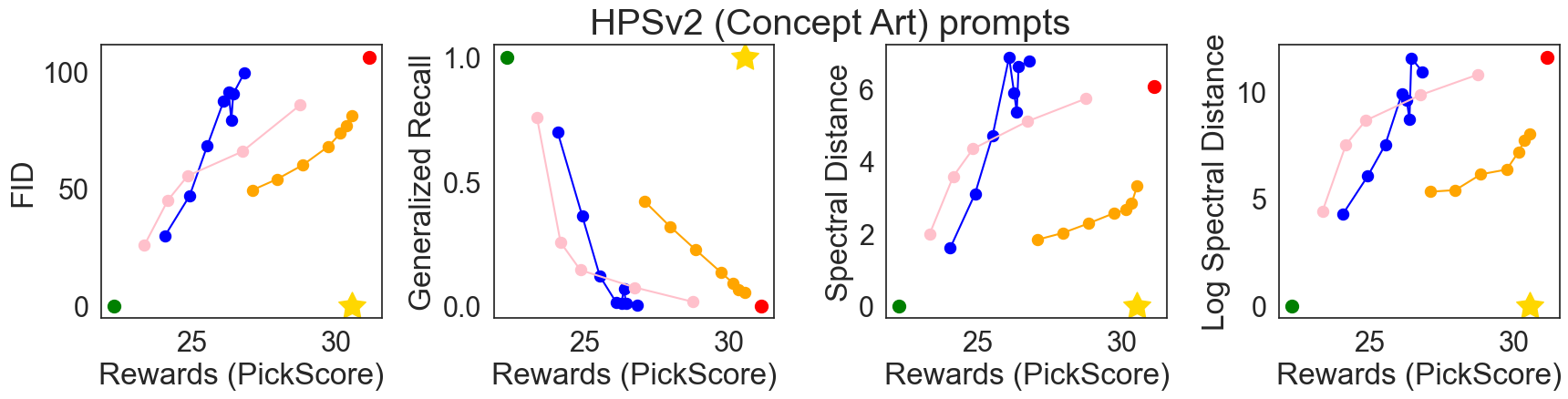}
        \includegraphics[width=\linewidth]{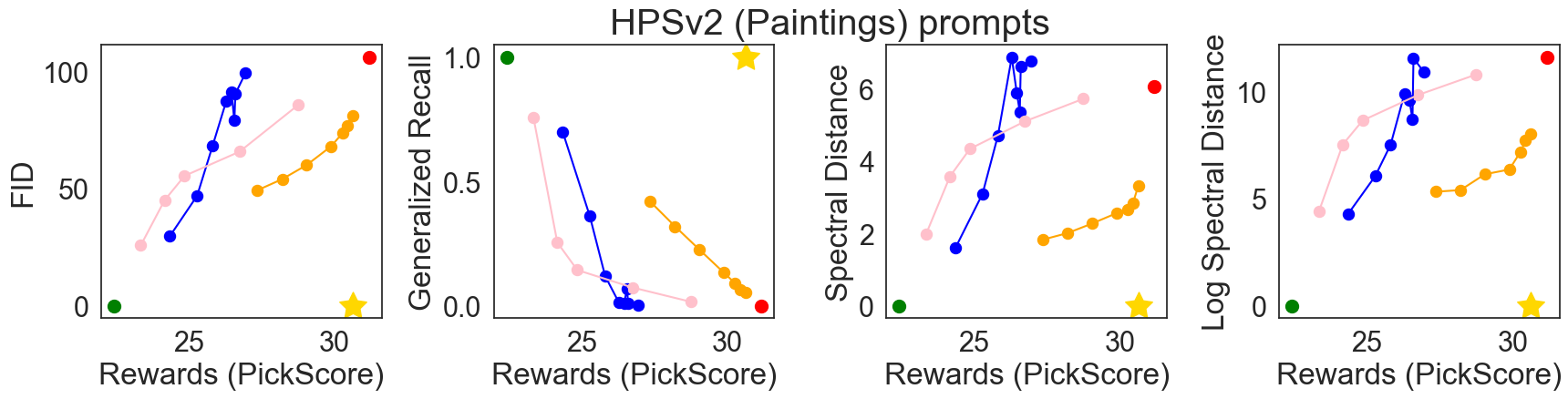}
        \includegraphics[width=\linewidth]{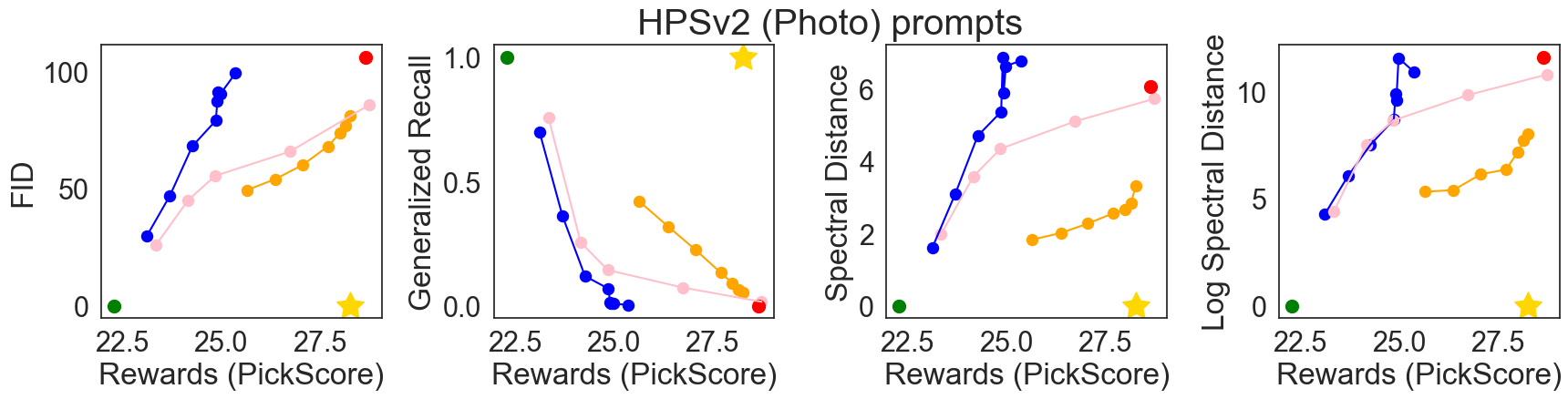}
    \end{minipage}
    \begin{minipage}{0.12\textwidth}
        \includegraphics[width=\linewidth]{images/graph-legend.png}
    \end{minipage}
    \caption{\textbf{Reward-diversity tradeoff for SDXL trained on PickScore}: \textcolor{green}{\textbf{Green}} represents the base model, \textcolor{red}{\textbf{Red}} represents \draft with no regularization, \textcolor{yellow}{\textbf{Gold}} star represents the ideal score. \textcolor{blue}{\textbf{Blue}} represents different models with different KL regularization coefficients $\lambda$, \textcolor{pink}{\textbf{Pink}} represents different amounts of LoRA scaling, and \textcolor{orange}{\textbf{Orange}} represents different $\gamma(t)$ for \methodabbv.
    An ideal baseline would achieve the highest reward (represented by \draft) as well as a complete match with the base distribution.
    For all measures for both PartiPrompt and HPSv2 subset prompts, {\methodabbv} achieves Pareto-optimality.
    }
    \label{fig:sdxl_pickscore_all}
\end{figure}

\begin{figure*}[ht!]
    \centering
    \begin{minipage}{0.85\linewidth}
        \includegraphics[width=\linewidth]{images/ablations/sd_pickscore_eval_PartiPrompt_pickscore.png}
        \includegraphics[width=\linewidth]{images/ablations/sd_pickscore_eval_hpsv2__anime__pickscore.png}
        \includegraphics[width=\linewidth]{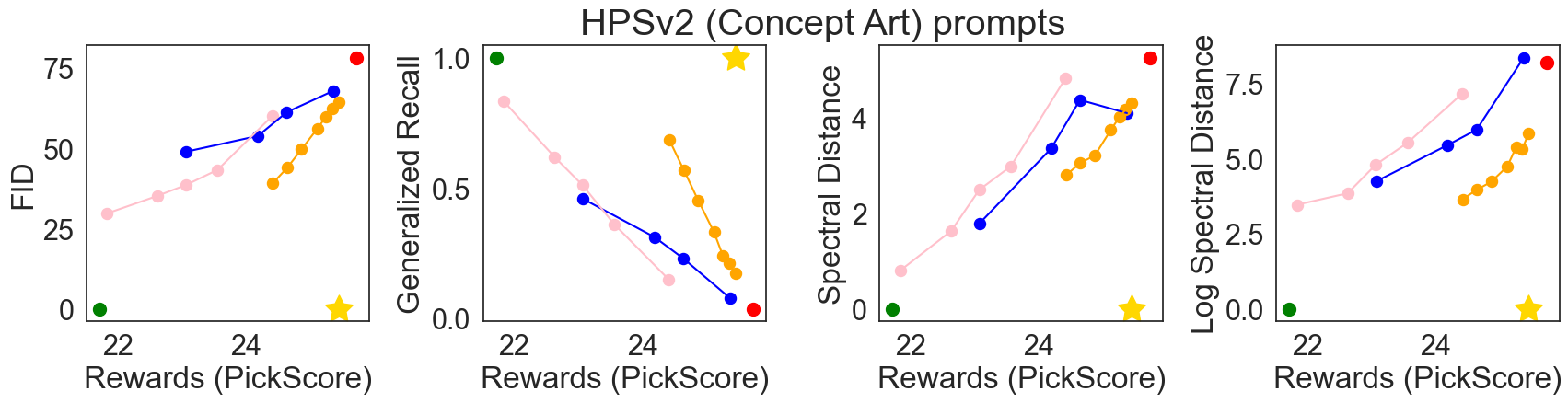}
        \includegraphics[width=\linewidth]{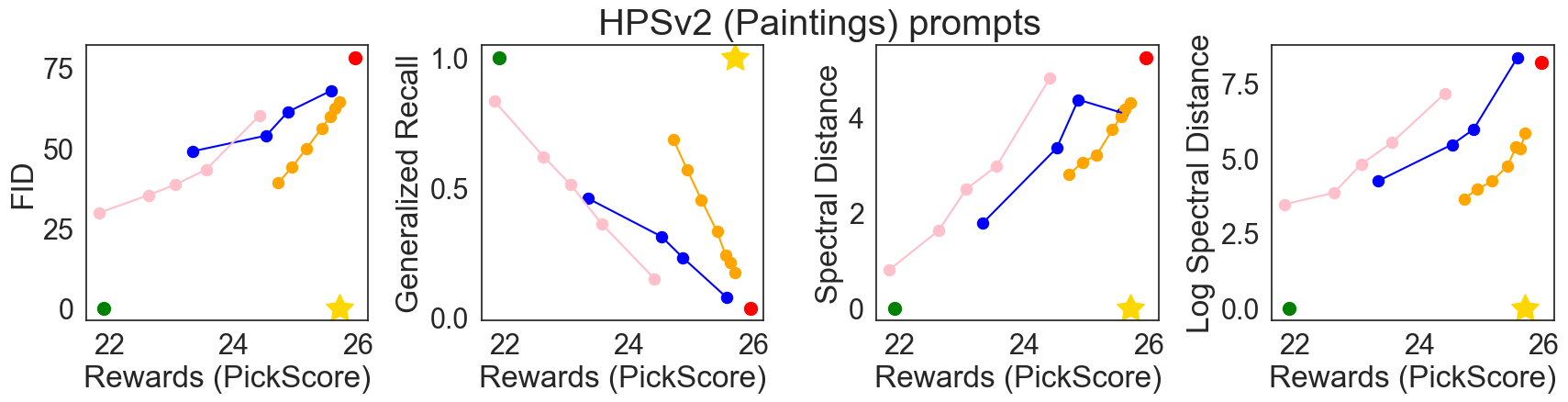}
        \includegraphics[width=\linewidth]{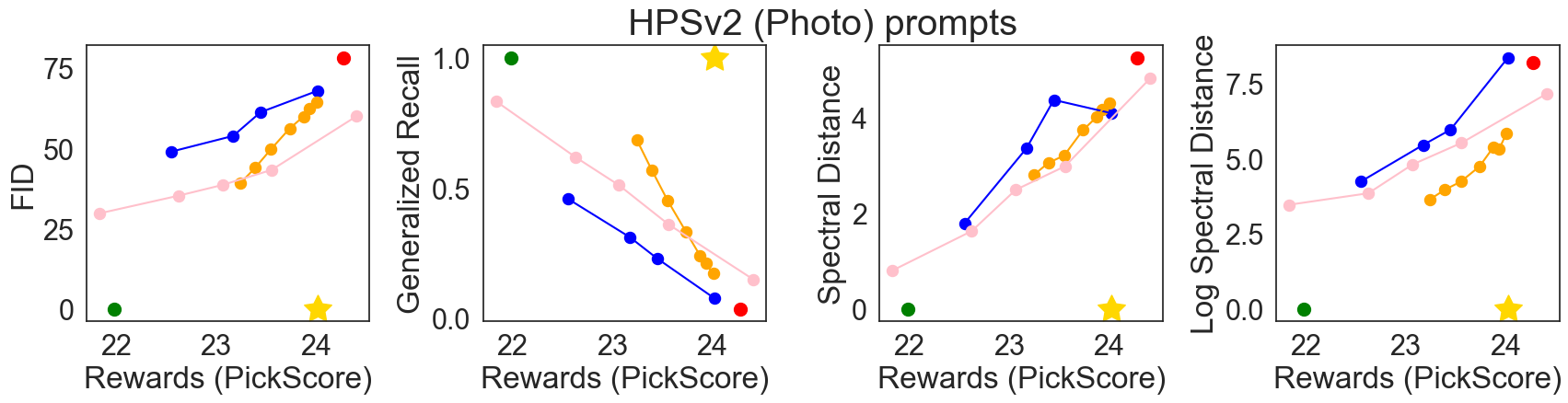}
    \end{minipage}
    \begin{minipage}{0.12\textwidth}
        \includegraphics[width=\linewidth]{images/graph-legend.png}
    \end{minipage}
    \caption{\textbf{Reward-diversity tradeoff for SDv1.4 trained on PickScore}: \textcolor{green}{\textbf{Green}} represents the base model, \textcolor{red}{\textbf{Red}} represents \draft with no regularization, \textcolor{yellow}{\textbf{Gold}} star represents the ideal score. \textcolor{blue}{\textbf{Blue}} represents different models with different KL regularization coefficients $\lambda$, \textcolor{pink}{\textbf{Pink}} represents different amounts of LoRA scaling, and \textcolor{orange}{\textbf{Orange}} represents different $\gamma(t)$ for \methodabbv.
    An ideal baseline would achieve the highest reward (represented by \draft) as well as a complete match with the base distribution.
    For all measures for both PartiPrompt and HPSv2 subset prompts, {\methodabbv} achieves Pareto-optimality.
    }
    \label{fig:sd_pickscore_all}
\end{figure*}

\begin{figure*}[ht!]
    \centering
    \begin{minipage}{0.85\linewidth}
    \includegraphics[width=\linewidth]{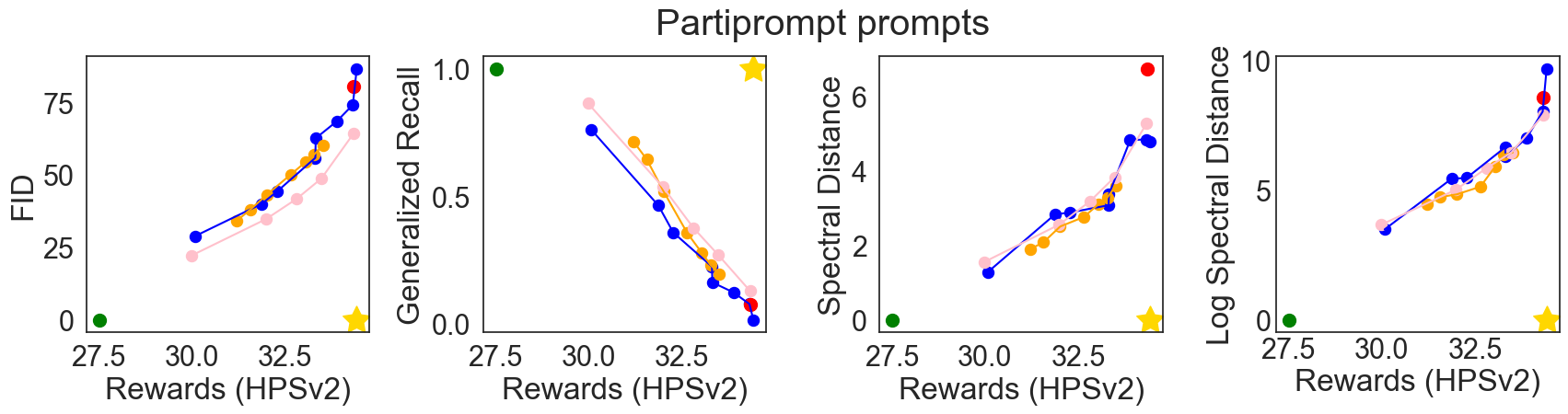}
    \includegraphics[width=\linewidth]{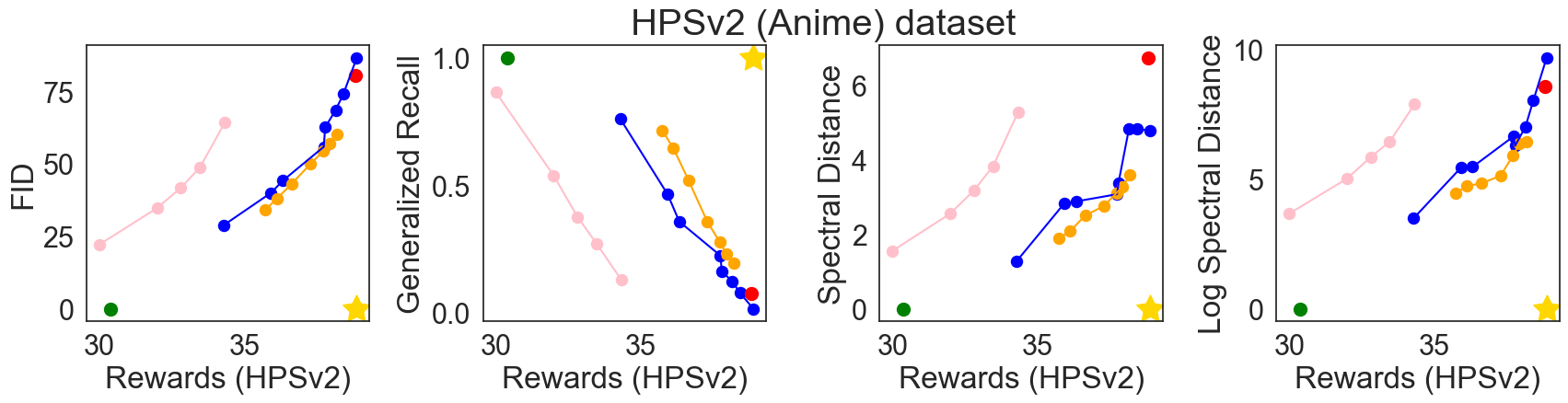}
    \includegraphics[width=\linewidth]{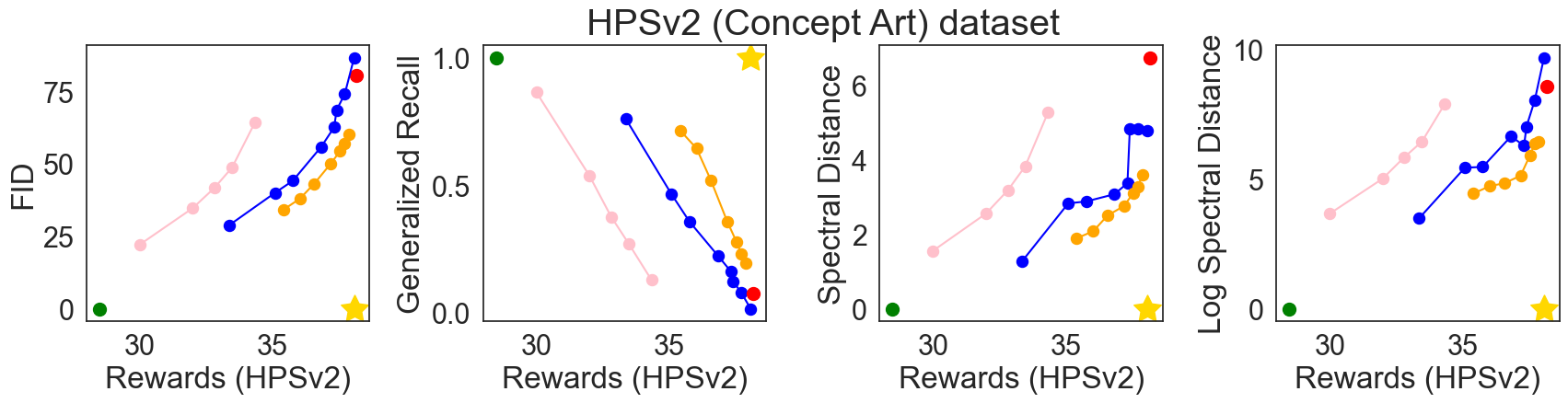}
    \includegraphics[width=\linewidth]{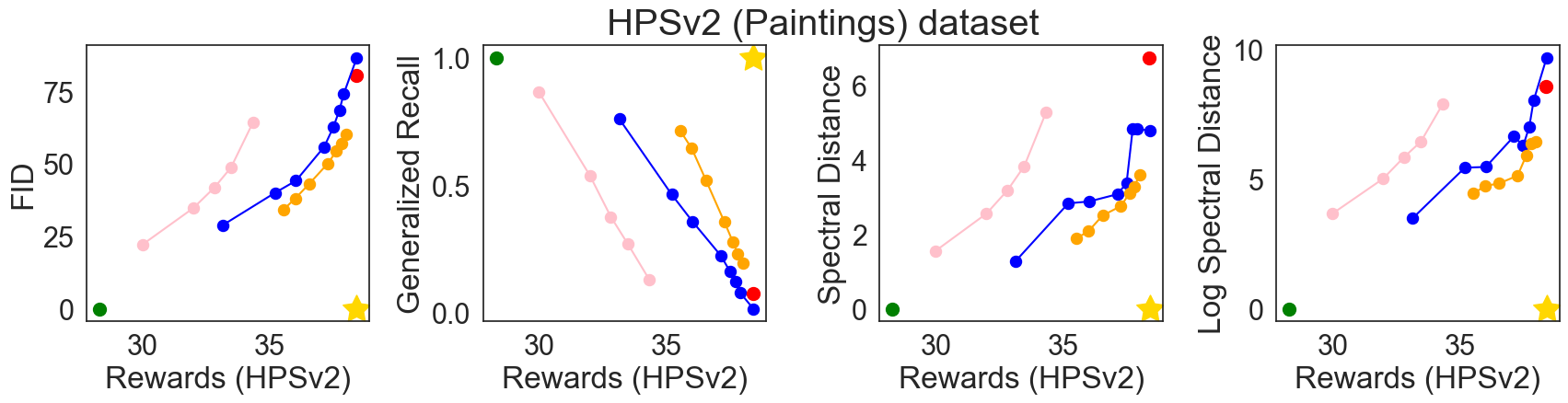}
    \includegraphics[width=\linewidth]{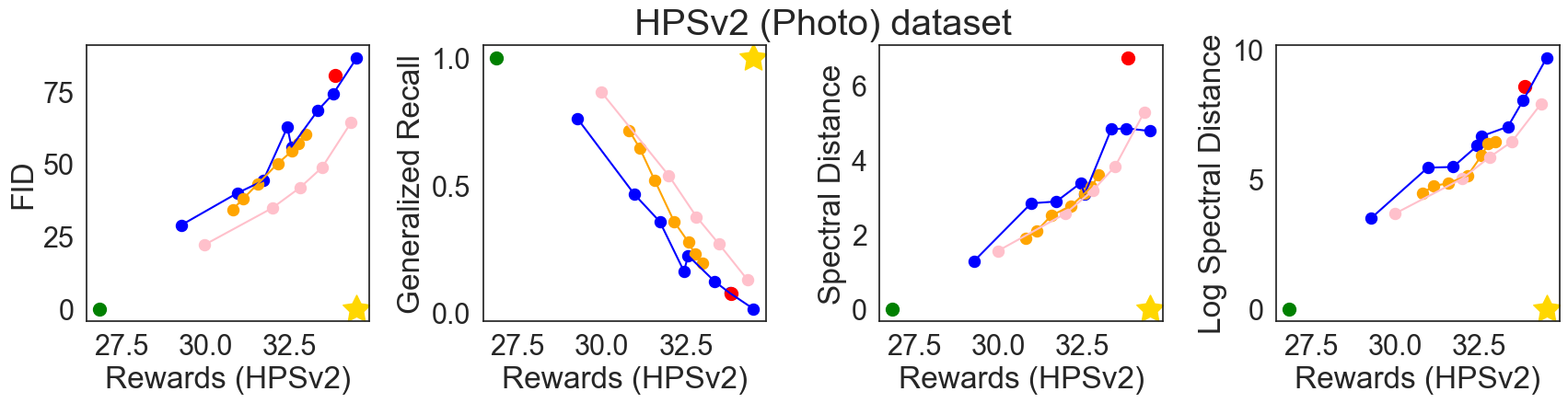}
    \end{minipage}
    \begin{minipage}{0.12\linewidth}
        \includegraphics[width=\linewidth]{images/graph-legend.png}
    \end{minipage}
    \caption{\textbf{Reward-diversity tradeoff for SDXL trained on HPSv2}: \textcolor{green}{\textbf{Green}} represents the base model, \textcolor{red}{\textbf{Red}} represents \draft with no regularization, \textcolor{yellow}{\textbf{Gold}} star represents the ideal score. \textcolor{blue}{\textbf{Blue}} represents different models with different KL regularization coefficients $\lambda$, \textcolor{pink}{\textbf{Pink}} represents different amounts of LoRA scaling, and \textcolor{orange}{\textbf{Orange}} represents different $\gamma(t)$ for \methodabbv.
    An ideal baseline would achieve the highest reward (represented by \draft) as well as a complete match with the base distribution.
    For all measures for both PartiPrompt and HPSv2 subset prompts, {\methodabbv} achieves Pareto-optimality.
    }
    \label{fig:sdxl_hps_all}
\end{figure*}

\begin{figure*}[ht!]
    \centering
    \begin{minipage}{0.85\linewidth}
    \includegraphics[width=\linewidth]{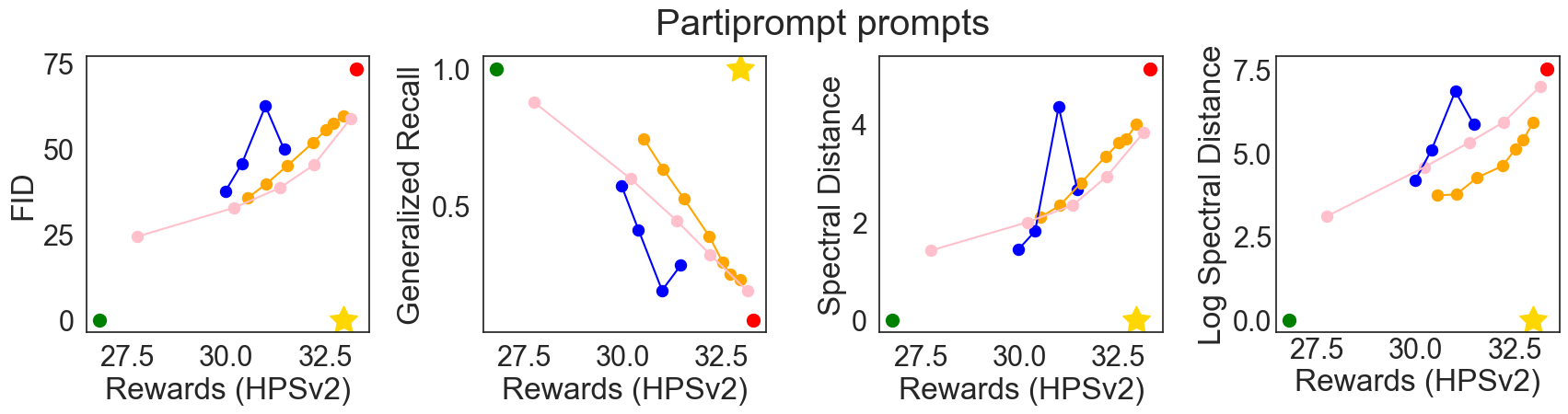}
    \includegraphics[width=\linewidth]{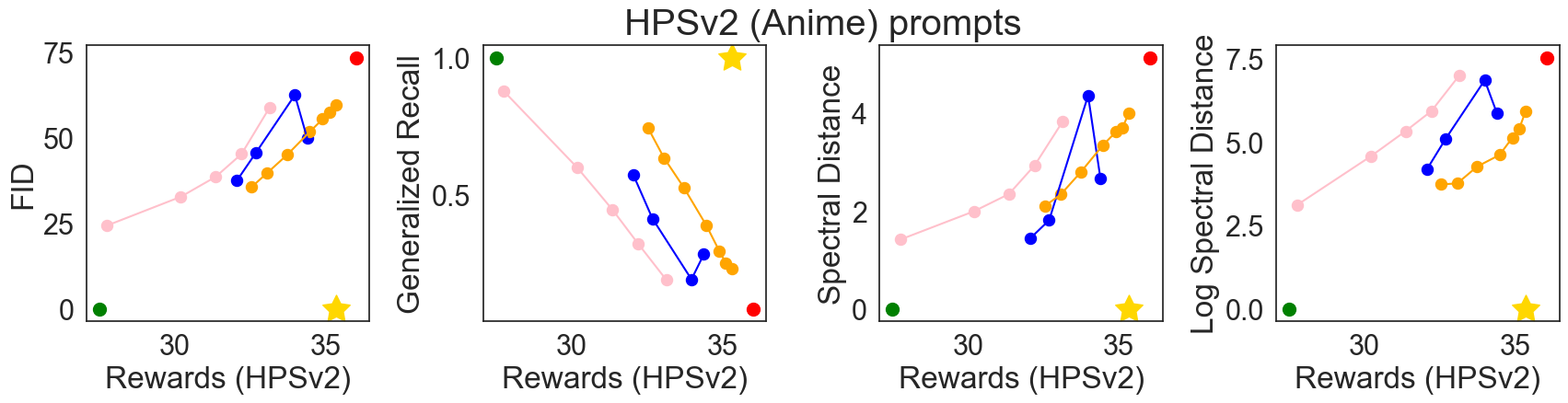}
    \includegraphics[width=\linewidth]{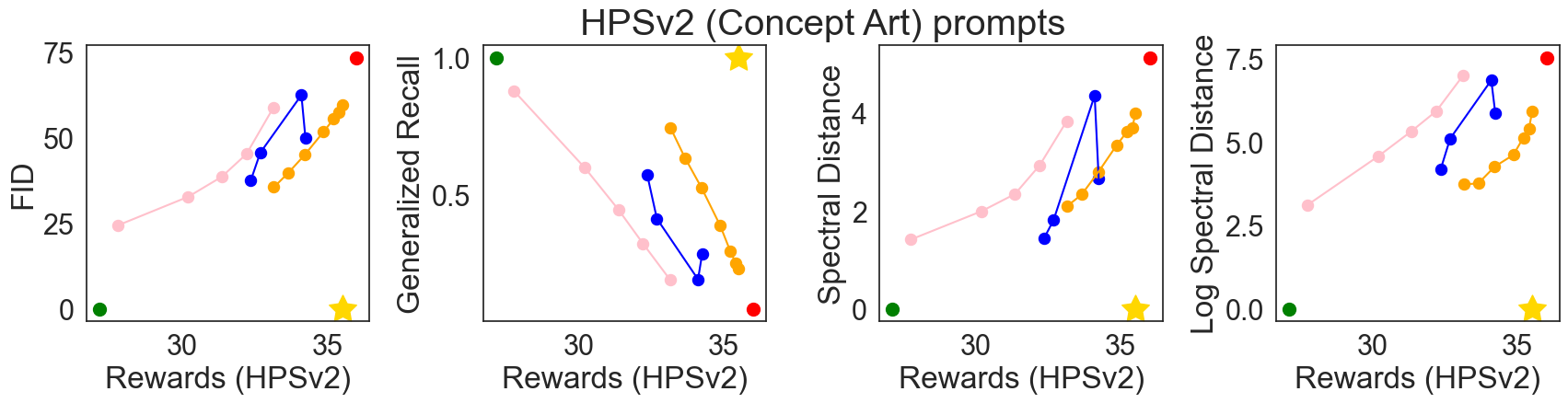}
    \includegraphics[width=\linewidth]{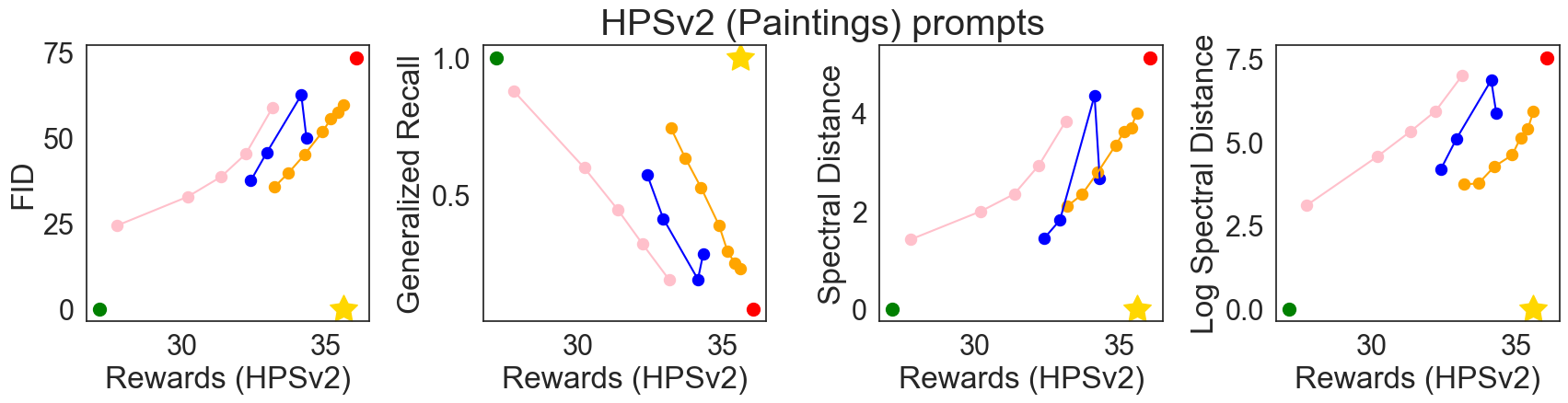}
    \includegraphics[width=\linewidth]{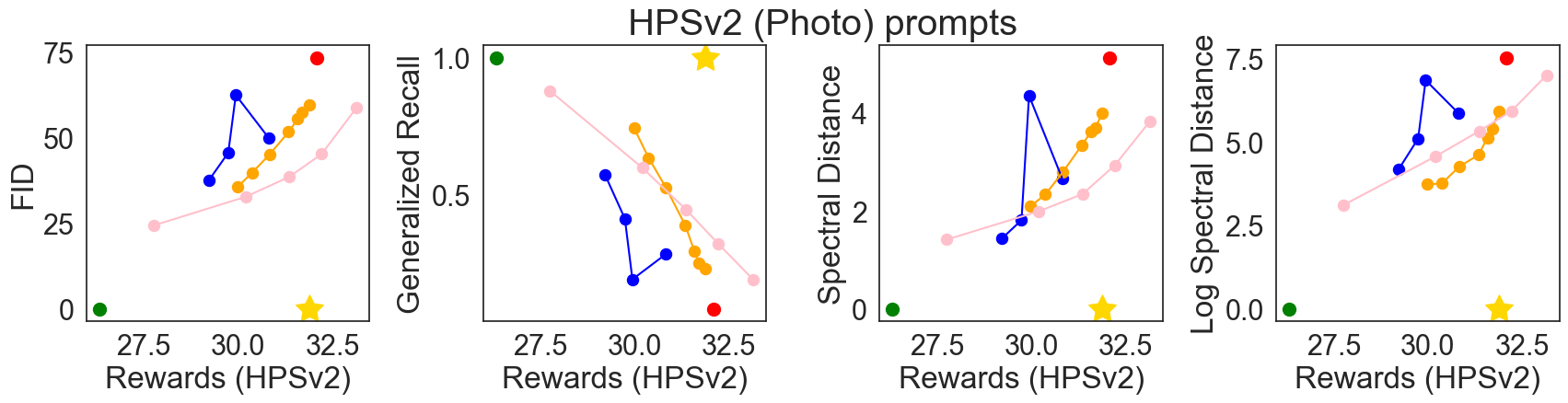}
    \end{minipage}
    \begin{minipage}{0.12\linewidth}
        \includegraphics[width=\linewidth]{images/graph-legend.png}
    \end{minipage}
    \caption{\textbf{Reward-diversity tradeoff for SDv1.4 trained on HPSv2}: \textcolor{green}{\textbf{Green}} represents the base model, \textcolor{red}{\textbf{Red}} represents \draft with no regularization, \textcolor{yellow}{\textbf{Gold}} star represents the ideal score. \textcolor{blue}{\textbf{Blue}} represents different models with different KL regularization coefficients $\lambda$, \textcolor{pink}{\textbf{Pink}} represents different amounts of LoRA scaling, and \textcolor{orange}{\textbf{Orange}} represents different $\gamma(t)$ for \methodabbv.
    An ideal baseline would achieve the highest reward (represented by \draft) as well as a complete match with the base distribution.
    For all measures for both PartiPrompt and HPSv2 subset prompts, {\methodabbv} achieves Pareto-optimality.
    }
    \label{fig:sd_hps_all}
\end{figure*}

\begin{figure}[ht!]
    \centering
    \includegraphics[width=0.9\linewidth]{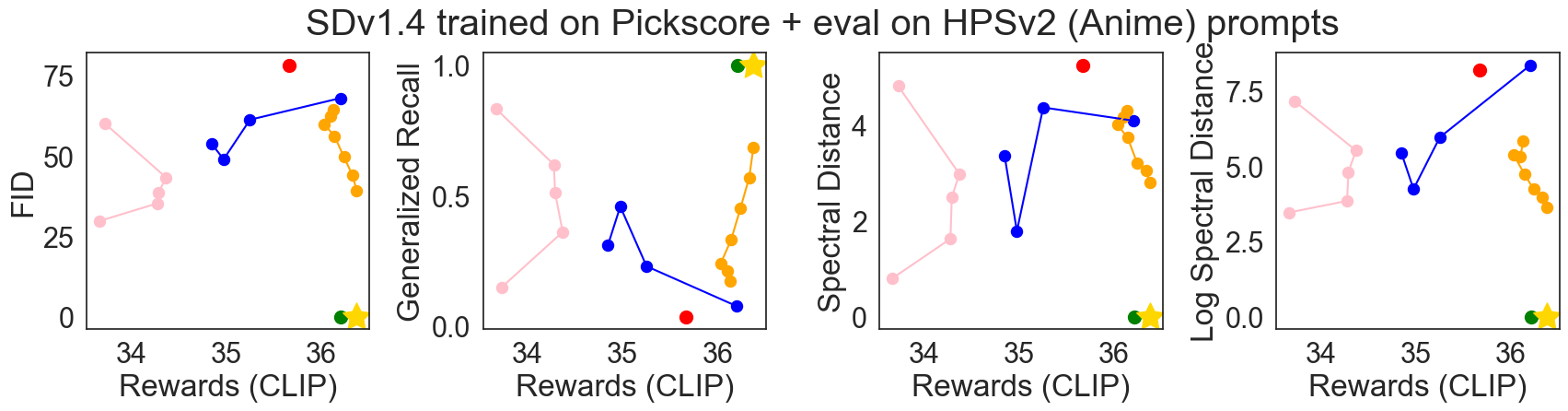}
    \includegraphics[width=0.9\linewidth]{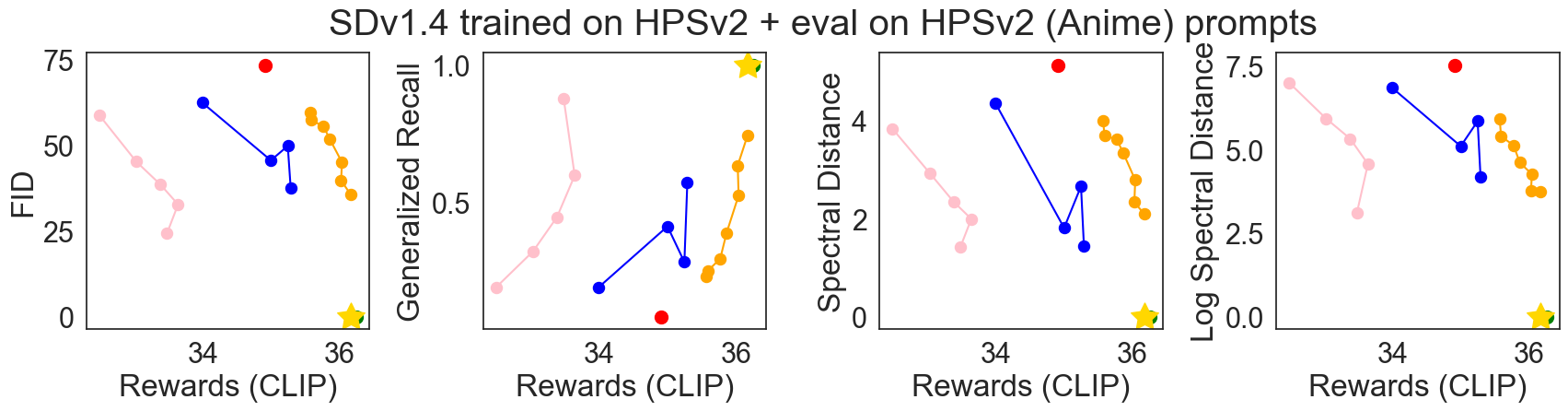}
    \includegraphics[width=0.9\linewidth]{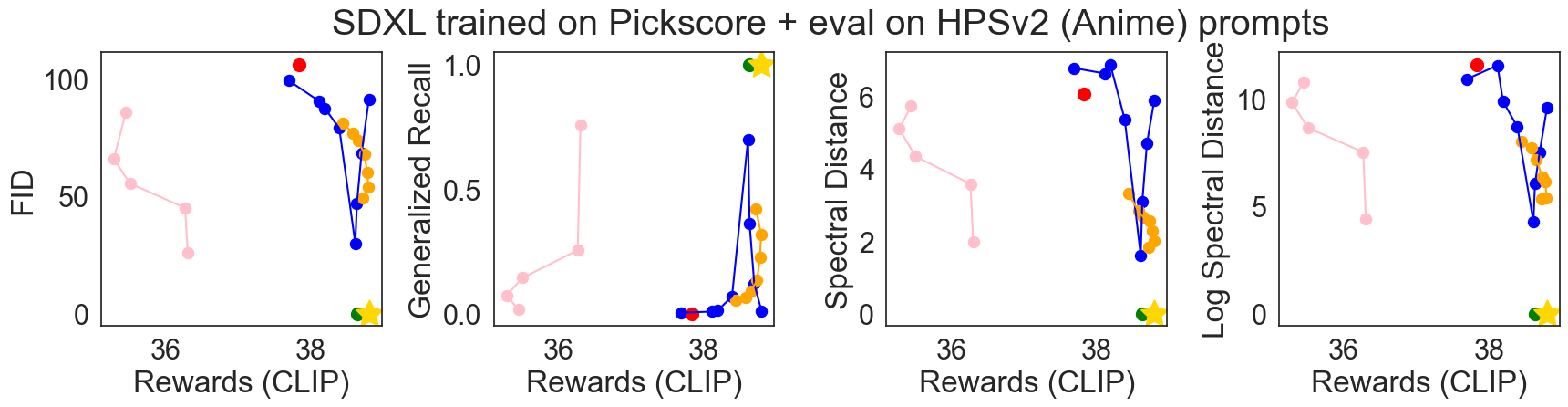}
    \includegraphics[width=0.9\linewidth]{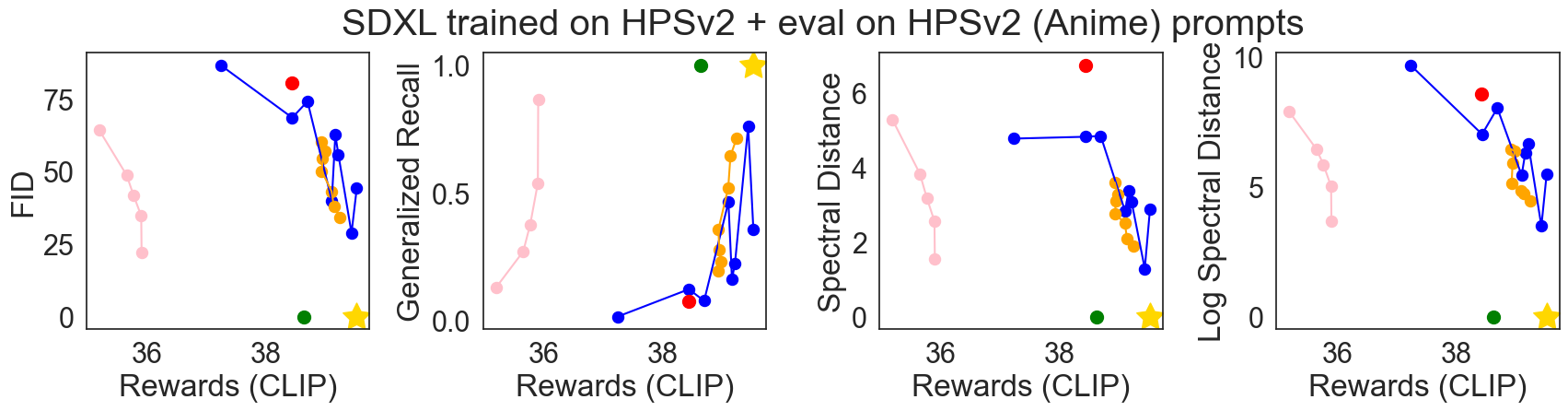}
    \caption{\textbf{CLIP-Diversity tradeoff for configurations on HPSv2 anime prompts}: \textcolor{green}{\textbf{Green}} represents the base model, \textcolor{red}{\textbf{Red}} represents {\draft} with no regularization, \textcolor{yellow}{\textbf{Gold}} star represents the ideal score. \textcolor{blue}{\textbf{Blue}} represents different models with different KL regularization coefficients $\lambda$, \textcolor{pink}{\textbf{Pink}} represents different amounts of LoRA scaling, and \textcolor{orange}{\textbf{Orange}} represents different $\gamma(t)$ for \methodabbv.
    {\methodabbv} consistently outperforms LoRA scaling in CLIP alignment.
    }
    \label{fig:hps_anime_clip}
\end{figure}

\begin{figure}[ht!]
    \centering
    \includegraphics[width=0.9\linewidth]{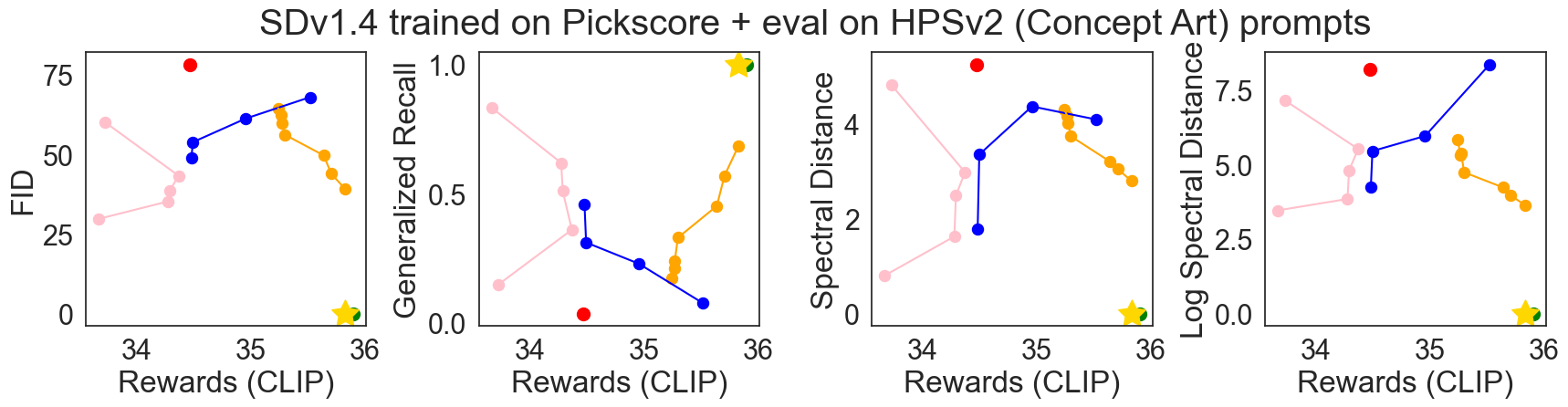}
    \includegraphics[width=0.9\linewidth]{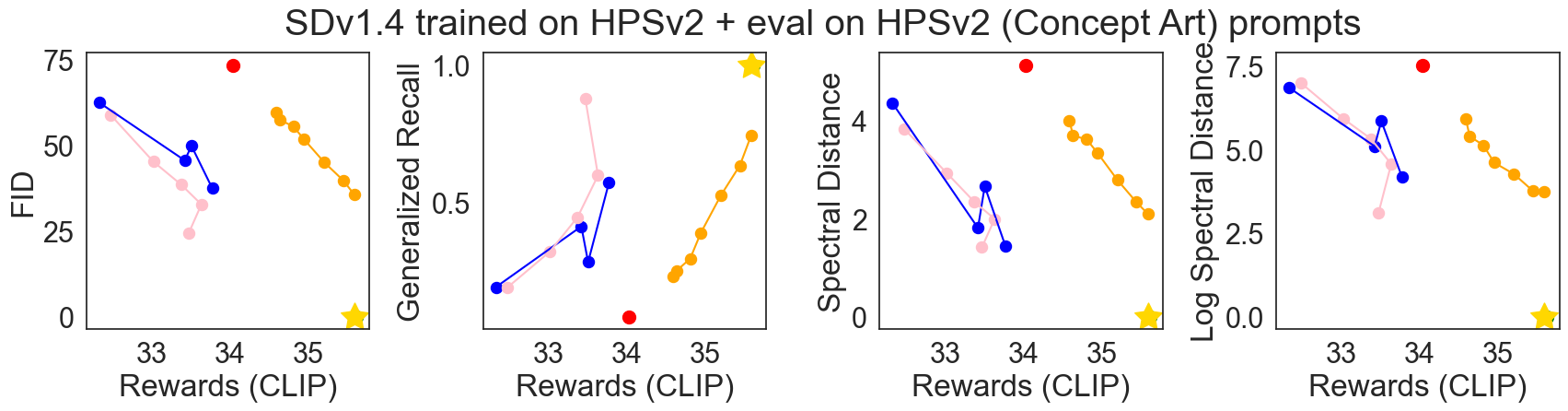}
    \includegraphics[width=0.9\linewidth]{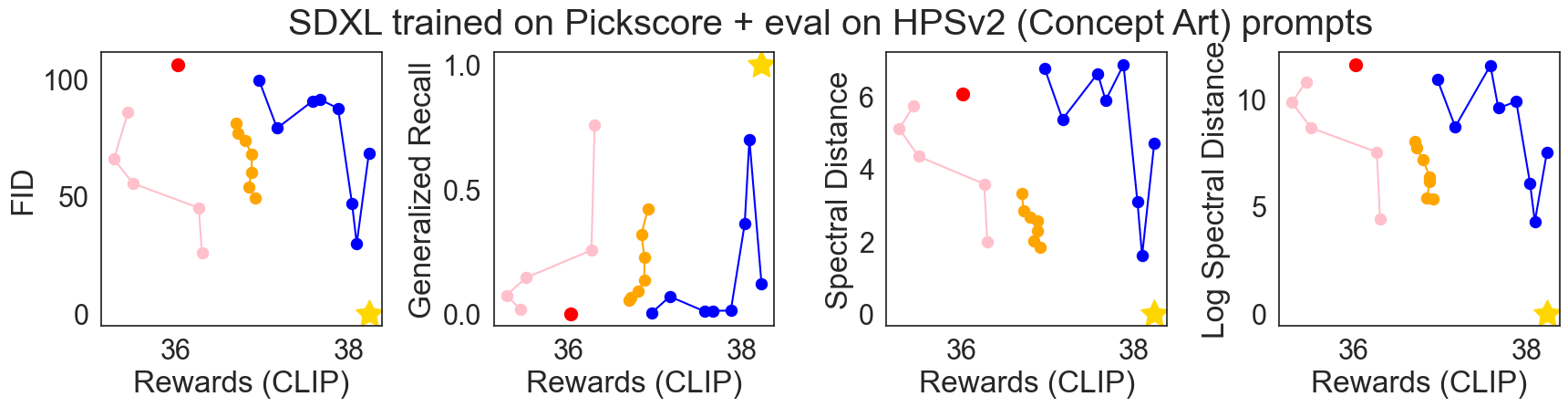}
    \includegraphics[width=0.9\linewidth]{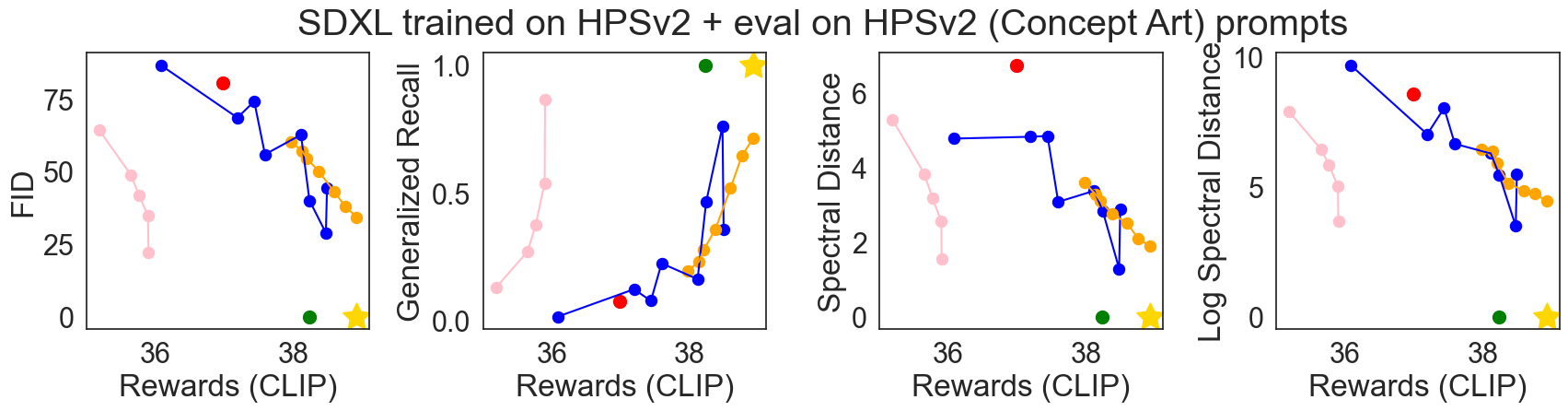}
    \caption{\textbf{CLIP-Diversity tradeoff for configurations on HPSv2 concept art prompts}: \textcolor{green}{\textbf{Green}} represents the base model, \textcolor{red}{\textbf{Red}} represents {\draft} with no regularization, \textcolor{yellow}{\textbf{Gold}} star represents the ideal score. \textcolor{blue}{\textbf{Blue}} represents different models with different KL regularization coefficients $\lambda$, \textcolor{pink}{\textbf{Pink}} represents different amounts of LoRA scaling, and \textcolor{orange}{\textbf{Orange}} represents different $\gamma(t)$ for \methodabbv.
    {\methodabbv} consistently outperforms LoRA scaling in CLIP alignment.
    }
    \label{fig:hps_conceptart_clip}
\end{figure}

\begin{figure}[ht!]
    \centering
    \includegraphics[width=0.9\linewidth]{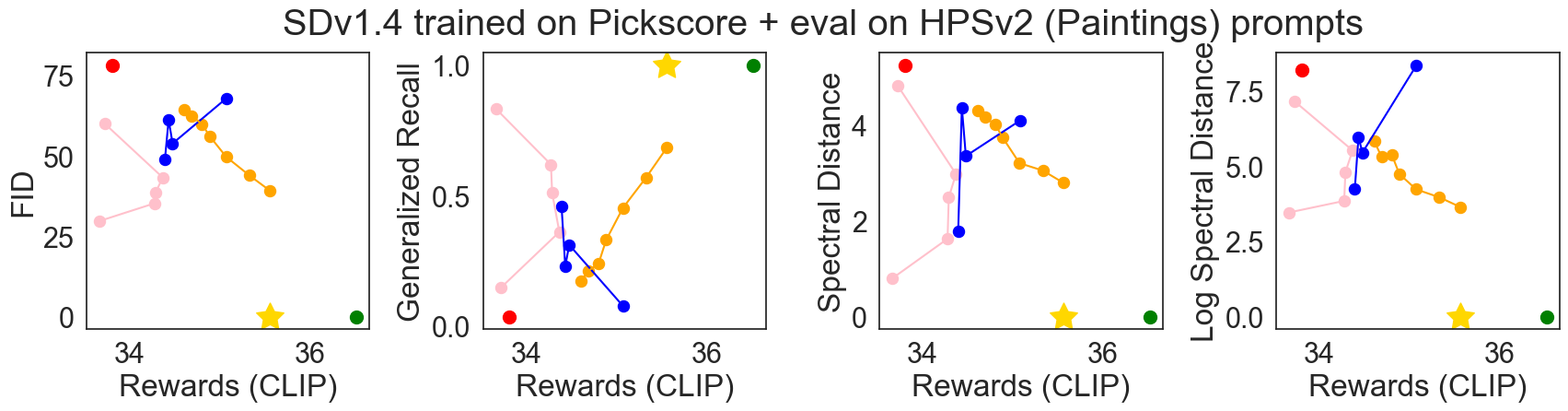}
    \includegraphics[width=0.9\linewidth]{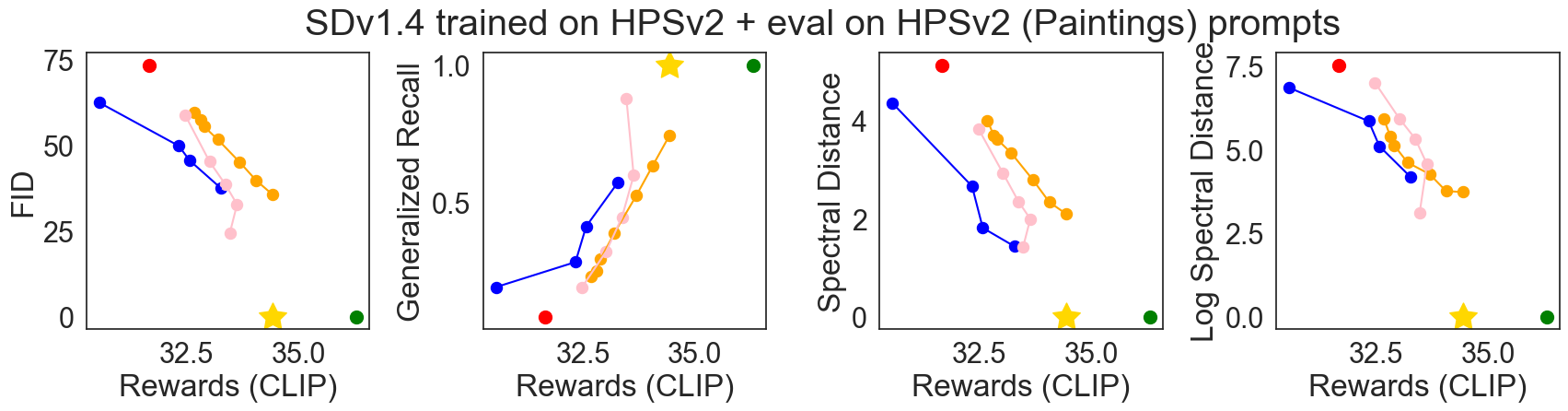}
    \includegraphics[width=0.9\linewidth]{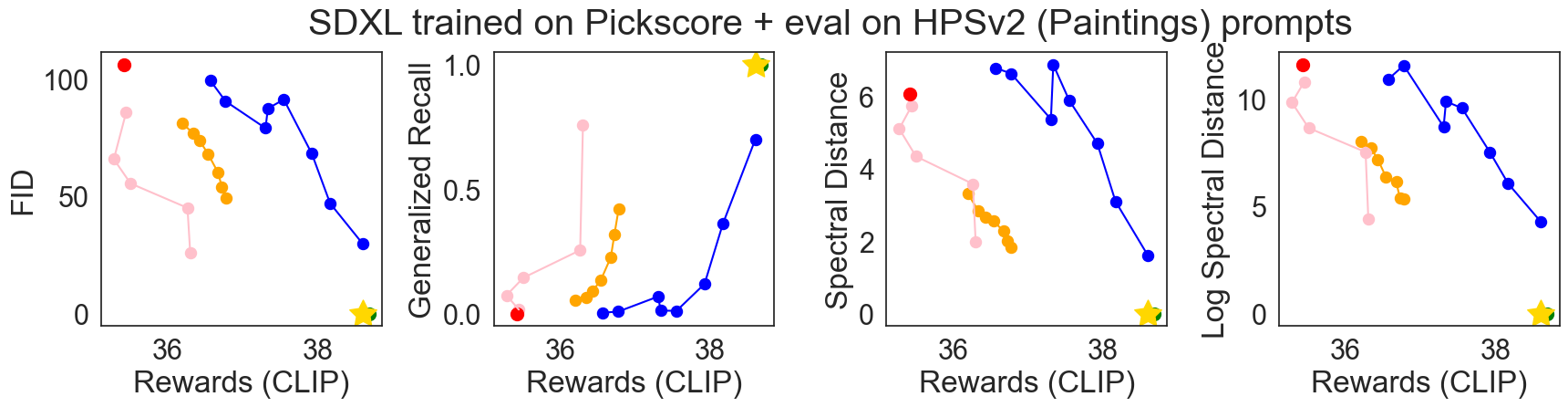}
    \includegraphics[width=0.9\linewidth]{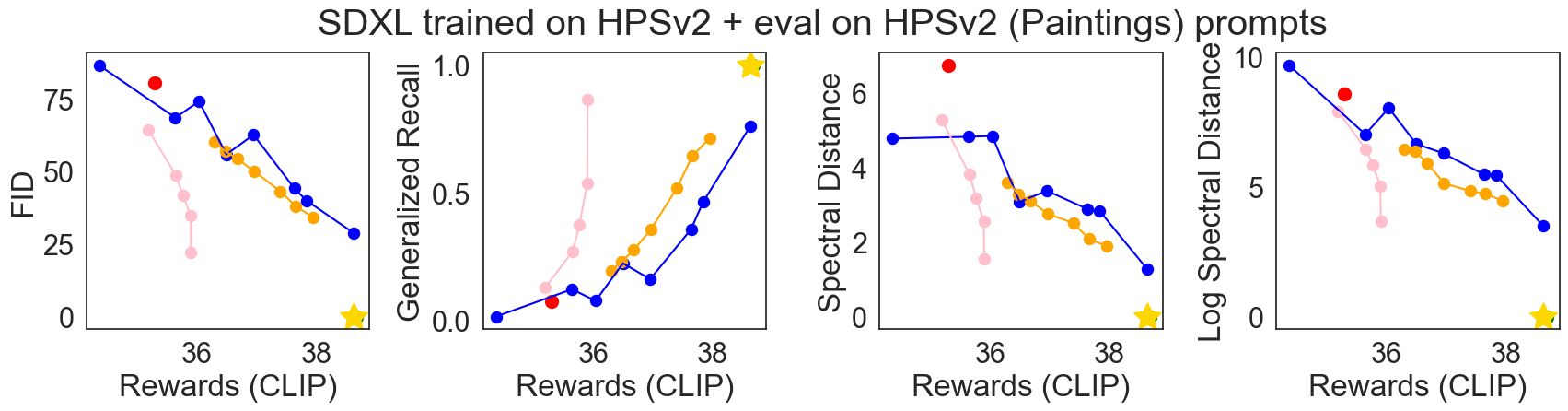}
    \caption{\textbf{CLIP-Diversity tradeoff for configurations on HPSv2 painting prompts}: \textcolor{green}{\textbf{Green}} represents the base model, \textcolor{red}{\textbf{Red}} represents {\draft} with no regularization, \textcolor{yellow}{\textbf{Gold}} star represents the ideal score. \textcolor{blue}{\textbf{Blue}} represents different models with different KL regularization coefficients $\lambda$, \textcolor{pink}{\textbf{Pink}} represents different amounts of LoRA scaling, and \textcolor{orange}{\textbf{Orange}} represents different $\gamma(t)$ for \methodabbv.
    {\methodabbv} consistently outperforms LoRA scaling in CLIP alignment.
    }
    \label{fig:hps_painting_clip}
\end{figure}

\begin{figure}[ht!]
    \centering
    \includegraphics[width=0.9\linewidth]{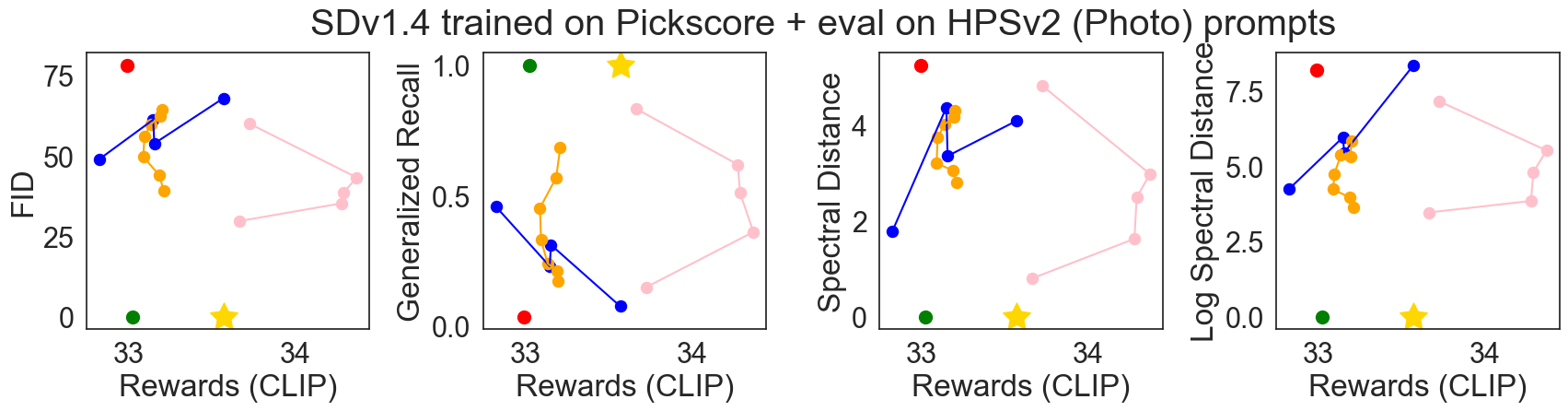}
    \includegraphics[width=0.9\linewidth]{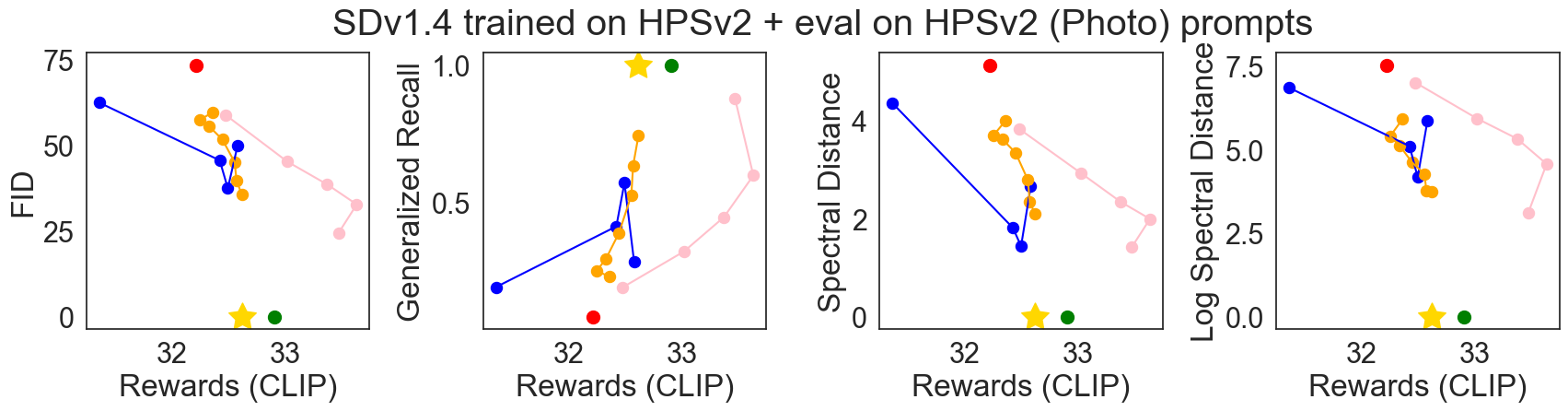}
    \includegraphics[width=0.9\linewidth]{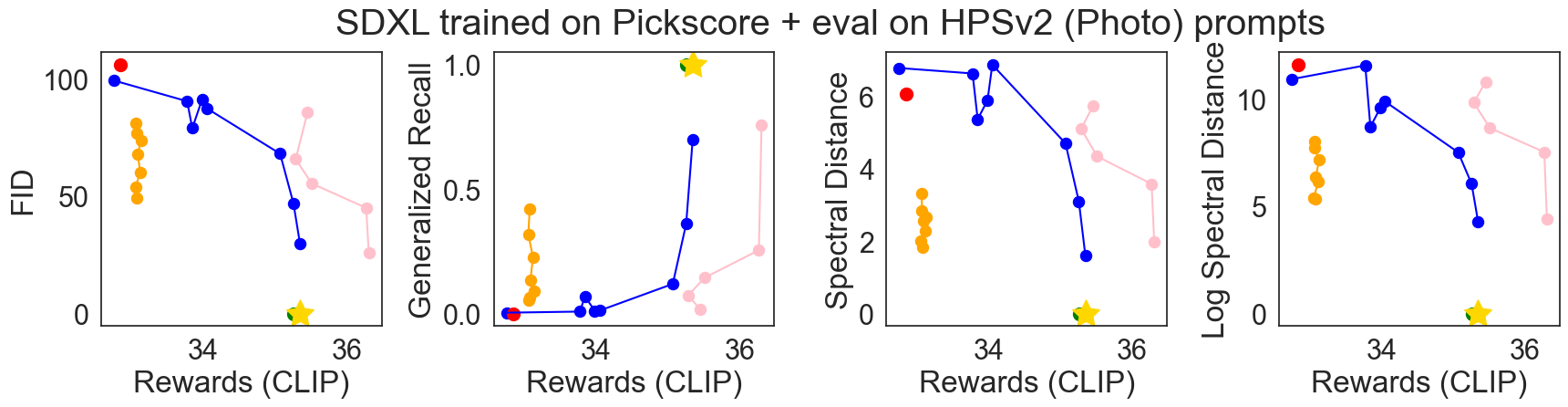}
    \includegraphics[width=0.9\linewidth]{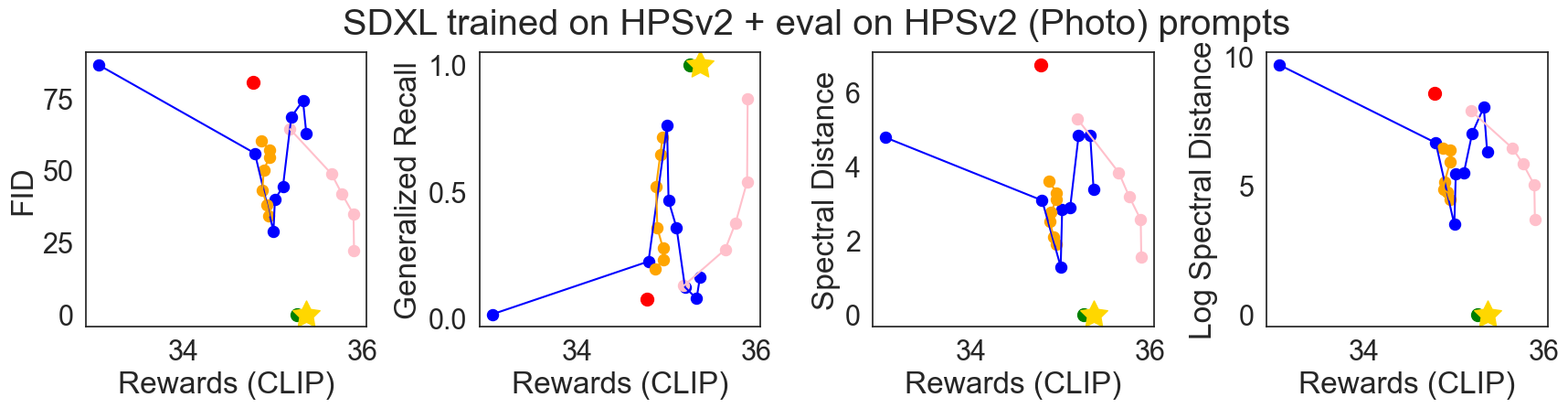}
    \caption{\textbf{CLIP-Diversity tradeoff for configurations on HPSv2 photo prompts}: \textcolor{green}{\textbf{Green}} represents the base model, \textcolor{red}{\textbf{Red}} represents {\draft} with no regularization, \textcolor{yellow}{\textbf{Gold}} star represents the ideal score. \textcolor{blue}{\textbf{Blue}} represents different models with different KL regularization coefficients $\lambda$, \textcolor{pink}{\textbf{Pink}} represents different amounts of LoRA scaling, and \textcolor{orange}{\textbf{Orange}} represents different $\gamma(t)$ for \methodabbv.
    {\methodabbv} underperforms LoRA scaling in CLIP alignment.
    }
    \label{fig:hps_photo_clip}
\end{figure}

\end{document}